\documentclass[twoside,11pt]{article}
\usepackage{thm-restate}
\usepackage{jmlr2e}

\usepackage{amsmath}
\usepackage{booktabs}
\usepackage{cleveref}
\usepackage{crossreftools}
\usepackage{dsfont}
\usepackage{enumitem}
\usepackage{lastpage}
\usepackage{siunitx}
\usepackage{threeparttable}
\usepackage{xspace}

\hypersetup{breaklinks=true}

\newcommand{\acal}{\ensuremath{\mathcal{A}}}
\newcommand{\fcal}{\ensuremath{\mathcal{F}}}
\newcommand{\pcal}{\ensuremath{\mathcal{P}}}
\newcommand{\qcal}{\ensuremath{\mathcal{Q}}}
\newcommand{\wcal}{\ensuremath{\mathcal{W}}}
\newcommand{\xcal}{\ensuremath{\mathcal{X}}}
\newcommand{\zcal}{\ensuremath{\mathcal{Z}}}

\newcommand{\Rds}{\ensuremath{\mathds{R}}}

\newcommand{\LB}{\left[}
\newcommand{\RB}{\right]}
\newcommand{\LC}{\left\{}
\newcommand{\RC}{\right\}}
\newcommand{\LP}{\left(}
\newcommand{\RP}{\right)}

\newcommand{\ie}{{\it i.e.}\xspace}
\newcommand{\eg}{{\it e.g.}\xspace}
\newcommand{\iid}{{\it i.i.d.}\xspace}
\newcommand{\cf}{{\it cf.}\xspace}

\DeclareMathOperator*{\esssup}{\text{\rm esssup}}

\newcommand{\R}{\Rds}
\newcommand{\F}{\mathds{F}}
\newcommand{\Rd}{{\R^d}}
\newcommand{\q}{\mathds{Q}}

\newcommand{\datadist}{\mu_z^{\otimes n}}
\newcommand{\datainfty}{\mu_z^{\otimes \infty}}
\newcommand{\klb}[2]{\text{\normalfont\textbf{KL}}(#1\|#2)}

\newcommand{\risk}{\mathcal{R}}
\newcommand{\er}{\widehat{\mathcal{R}}_S}
\newcommand{\err}{\text{\normalfont err}}
\newcommand{\gabs}{G_S^{\text{\normalfont abs}}(\wcal)}

\newcommand{\upperbox}{\overline{\dim}_B}
\newcommand{\rad}{\text{\normalfont \textbf{Rad}}}
\newcommand{\tv}{\text{\normalfont TV}}
\newcommand{\mutualinfty}{\text{I}_\infty}

\newcommand{\Eof}[2][]{\mathds{E}_{#1} \left[ #2 \right]}
\newcommand{\Pof}[2][]{\mathds{P}_{#1} \left( #2 \right)}
\newcommand{\E}{\mathds{E}}
\newcommand{\condexp}[3][]{\mathds{E}_{#1} \left[ #2 \big| #3 \right]}
\newcommand{\prob}{\mathds{P}}

\newcommand{\closed}{{\normalfont\textbf{CL}}(\mathds{R}^d)}
\newcommand{\normof}[1]{\left\Vert #1 \right\Vert}
\newcommand{\set}[1]{\left\{ #1 \right\}}
\newcommand{\scal}[2]{\left\langle #1 , #2 \right\rangle}
\newcommand{\ghat}{\hat{g}}
\newcommand{\landau}[1]{\mathcal{O} \left( #1 \right)}

\jmlrheading{25}{2024}{1-\pageref{LastPage}}{4/24; Revised 9/24}{12/24}{24-0605}{Benjamin Dupuis, Paul Viallard, George Deligiannidis and Umut \c{S}im\c{s}ekli}
\ShortHeadings{PAC-Bayesian Theory on Random Sets}{Dupuis, Viallard, Deligiannidis, \c{S}im\c{s}ekli}
\firstpageno{1}

\newtheorem{assumption}{Assumption}
\crefname{equation}{Eq.}{Eqs.}
\Crefname{equation}{Equation}{Equations}
\crefname{theorem}{Th.}{Ths.}
\Crefname{theorem}{Theorem}{Theorems}
\crefname{corollary}{Cor.}{Cors.}
\Crefname{corollary}{Corollary}{Corollaries}
\crefname{algorithm}{Algo.}{Algos.}
\Crefname{algorithm}{Algorithm}{Algorithms}
\crefname{assumption}{assumption}{assumptions}
\Crefname{assumption}{Assumption}{Assumptions}

\begin{document}

\title{Uniform Generalization Bounds on Data-Dependent Hypothesis Sets via PAC-Bayesian Theory on Random Sets}

\author{\name Benjamin Dupuis* \email benjamin.dupuis@inria.fr \\
       \addr INRIA - Département d’Informatique de l’Ecole Normale Supérieure\\
       PSL Research University\\
       Paris, France
       \AND
       \name Paul Viallard \email paul.viallard@inria.fr \\
       \addr Univ Rennes, Inria, CNRS IRISA - UMR 6074\\ 
       Rennes, France
       \AND
       \name George Deligiannidis \email george.deligiannidis@stats.ox.ac.uk \\
       \addr Department of Statistics \\
       University of Oxford, Oxford, UK
       \AND
       \name Umut \c{S}im\c{s}ekli \email umut.simsekli@inria.fr \\
        \addr INRIA - Département d’Informatique de l’Ecole Normale Supérieure\\
       PSL Research University - CNRS\\
       Paris, France
       \\
       \\
       {\normalfont \textbf{*}} Corresponding author.}

\editor{Gergely Neu}

\maketitle

\begin{abstract}

We propose data-dependent uniform generalization bounds by approaching the problem from a PAC-Bayesian perspective. We first apply the PAC-Bayesian framework on ``random sets'' in a rigorous way, where the training algorithm is assumed to output a data-dependent hypothesis set after observing the training data. This approach allows us to prove data-dependent bounds, which can be applicable in numerous contexts. To highlight the power of our approach, we consider two main applications.
First, we propose a PAC-Bayesian formulation of the recently developed fractal-dimension-based generalization bounds. The derived results are shown to be tighter and they unify the existing results around one simple proof technique.
Second, we prove uniform bounds over the trajectories of continuous Langevin dynamics and stochastic gradient Langevin dynamics. These results provide novel information about the generalization properties of noisy algorithms.
\end{abstract}


\section{Introduction}

Over the past decades, providing generalization guarantees for modern machine learning algorithms has been a major research topic.
In most cases, these algorithms can be framed as the following optimization problem:
\begin{align}
    \label{eq:stoch_opt}
    \min \set{\risk(w) := \int_\zcal \ell(w,z) d \mu_z (z),~w \in \R^d},
\end{align}
where $(\zcal, \mathcal{F})$ is a measurable space, $\mu_z$ a data distribution on $\zcal$ and $\ell: \R^d \times \zcal \longrightarrow \R_+$ is the composite loss function.
The function $\risk: \Rd \longrightarrow \mathds{R}_+$ is called the population risk. The vectors $w \in \Rd$ are the \emph{parameters} (or \emph{weights} in the neural network context) of the model.
For instance, in a regression setting, $\zcal$ would be the product $\mathcal{X} \times \mathcal{Y}$ of an input space $\mathcal{X}$ and a target space $\mathcal{Y}$.
In that case, $\ell$ would be the composition of a parametrized predictor $F_w : \mathcal{X} \longrightarrow \mathcal{Y}$ and a loss function $\mathcal{L}: \mathcal{Y} \times \mathcal{Y} \longrightarrow \R_+$, \ie, $\ell(w, (x,y)) = \mathcal{L}(F_w(x), y)$. 
In practical cases, we know neither the data distribution $\mu_z$ nor the population risk $\risk$, but we have access to independent and identically distributed (\iid) samples $S = (z_1,\dots,z_n) \sim \datadist$, drawn from the data distribution, where $\datadist$ is the product measure $\mu_z \otimes \dots \otimes \mu_z$.
Problem \eqref{eq:stoch_opt} is then replaced by the minimization of the empirical risk, defined by 
\begin{align}
\label{eq:er_intro}
\er(w) := \frac{1}{n} \sum_{i=1}^n \ell(w,z_i).
\end{align} 
Since Problem \eqref{eq:stoch_opt} is replaced by the function in \Cref{eq:er_intro}, it is necessary to evaluate the quantity $\risk(w) - \er(w)$ in order to estimate the performance of the method; we call this quantity the generalization error.
One particular class of generalization bounds (\ie, upper bounds on the generalization error) that has been widely studied is the class of uniform generalization bounds \citep[see \eg,][]{shalevschwartz2014understanding}.
It consists of considering a subset of the possible functions $\set{\ell(w,\cdot),~w\in\wcal}$, from which we aim to choose the learned model.
In many practical cases, we consider a set of parameters $\wcal \subseteq \R^d$, and we are interested in the worst generalization error over $\wcal$.
For example, $\wcal$ could be the set of vectors having a certain norm.
For a given fixed set $\wcal \subseteq \Rd$, which we will call a \emph{hypothesis set}, the quantity of interest becomes:
\begin{align}
    \label{eq:worst_gen_error_intro}
    G_S(\wcal) := \sup_{w\in\wcal} \left( \risk(w) - \er(w) \right).
\end{align}
Several studies have derived bounds on this quantity, known as \emph{uniform generalization error} or \emph{worst-case generalization error}.
While some of them define a notion of complexity of the hypothesis set, like Rademacher complexity \citep{bartlett2002rademacher}, other authors introduced intrinsic dimensions of $\wcal$, like Vapnik-Chervonenkis (VC) dimension \citep{vapnik1968uniform,vapnik1971uniform} or fractal dimensions \citep{simsekli2020hausdorff}.
When these quantities are significantly smaller than the ambient dimension $d$, they may be able to explain the good generalization performances of modern deep learning models, for which $d$ typically takes very high values.
Other notions of hypothesis set complexity have been used, even for non-uniform generalization bounds \citep{grunwald2019tight}.

However, as argued by \citet{nagarajan2019uniform}, even the tightest uniform generalization bounds, over a fixed set $\wcal$ (that may depend on the data distribution, $\mu_z$, but not on the data sample $S \sim \datadist$), may be vacuous. Moreover, it is known that neural networks can fit random labels \citep{zhang2017understanding,zhang2021understanding}, making quantities like Rademacher complexity over vast hypothesis sets excessively large to explain the generalization performance that is observed in practice.  

These examples show that one of the main drawbacks of such approaches is the lack of data dependence, \ie, the fact that the hypothesis set has no dependence on $S$. In addition to reducing the set of hypotheses, hence tightening the bounds, considering data-dependent sets provides generalization bounds that are specific to the dataset used during training.

Such data-dependent hypothesis sets, which we will denote $\wcal_S$, also depend on the learning algorithm and naturally emerge in stochastic optimization. For instance, depending on the context, $\wcal_S$ can be the set of minimizers of the empirical risk in \Cref{eq:er_intro}, or the trajectory of an iterative algorithm minimizing \Cref{eq:er_intro}. This last setting is considered in the fractal-based generalization literature \citep{simsekli2020hausdorff,dupuis2023generalization,hodgkinson2022generalization}. Classical machine learning models, such as Langevin diffusions \citep{raginsky2017non,mou2018generalization}, also generate data-dependent trajectories. Additional concrete examples will be provided in \Cref{sec:pac_bayes_for_random_sets}.

\subsection{Motivation} 
The data-dependence of $\wcal_S$ makes the task of bounding $G_S(\wcal_S)$ much more challenging compared to the case of a fixed hypothesis set, as most classical techniques will not be valid anymore.
Toward this goal, some studies imposed new assumptions on $\wcal_S$, like hypothesis set stability \citep{foster2019hypothesis}. 
Other works, like \citep{simsekli2020hausdorff, hodgkinson2022generalization, camuto2021fractal, dupuis2023generalization} introduced Mutual Information (MI) terms to control the statistical dependence between $S$ and $\wcal_S$. However, all these methods require different proof techniques and are based on complicated geometric and algorithmic assumptions. We emphasize the interest in proving new generic uniform generalization bounds for data-dependent hypothesis sets in two particular cases.

\subsubsection{Fractal-based generalization bounds}
Several recent works \citep{simsekli2020hausdorff, camuto2021fractal, hodgkinson2022generalization, birdal2021intrinsic, dupuis2023generalization} relate the worst-case generalization error to a notion of \emph{fractal dimension} \citep{falconer2014fractal} of the hypothesis set. These results can be informally summarized by stating that, with probability at least $1 - \zeta$ and for $n$ big enough, one has\footnote{We use the symbol $\lesssim$ to specify that absolute constants or logarithmic terms are omitted from the statement.}:
\begin{equation}
    \label{eq:fractal_informal}
    \sup_{w \in \wcal_S} \big( \mathcal{R}(w) - \er(w) \big) \lesssim \sqrt{\frac{(\text{Fractal Dim. of }\wcal_S) + \mathrm{I} + \log(1/\zeta)}{n}},
\end{equation}
which involves several notions of fractal dimensions that will be defined in the sequel. In \Cref{eq:fractal_informal}, $\mathrm{I}$ is a MI term, \ie, a notion of statistical dependence between the data and the hypothesis set, which may differ among the different results. 

The interest of such bounds resides in the fact that the fractal dimension of $\wcal_S$ might be much smaller than the dimension of the ambient space $\Rd$, making it pertinent for the study of over-parameterized models.
Moreover, these fractal-based bounds have opened the door to new links between the generalization error and topological data analysis (TDA) \citep{boissonnat2018geometric}.
Indeed, \citet{birdal2021intrinsic,dupuis2023generalization,andreeva2023metric} have shown that the aforementioned dimensions can be estimated using tools from TDA and observed empirical correlation with the generalization error.
While providing new insight into the generalization abilities of certain learning algorithms, in particular, heavy-tailed dynamics \citep{gurbuzbalaban2021heavy,simsekli2019tail,simsekli2020hausdorff}, these bounds suffer from several issues.
They depend on MI terms that differ from paper to paper and are generally beyond the reach of computability.
Moreover, these terms are often convoluted and hard to interpret, especially in \citep{dupuis2023generalization} and \citep{simsekli2020hausdorff}.
In \citep{dupuis2023generalization}, an additional ``geometric stability'' assumption is needed to simplify the MI term and make it similar to other works \citep{hodgkinson2022generalization}; this is an intricate assumption that seems hard to verify in practice and the resulting generalization bound has a worse rate in terms of the number of data points $n$.
Hence, our first aim is to provide a unified theoretical framework that can encompass all the existing fractal-dimension-based bounds with the correct rate of convergence without making additional non-trivial assumptions.

\subsubsection{Langevin dynamics}
Our second main motivation in data-dependent hypothesis sets stems from the generalization properties of Continuous Langevin Dynamics (CLD) and Stochastic Gradient Langevin Dynamics (SGLD). The CLD algorithm corresponds to a continuous-time gradient flow, perturbed with white Gaussian noise. It is described by the following Stochastic Differential Equation (SDE):
\begin{align}
    \label{eq:langevin_equation}
    dW_t = -\nabla \er(W_t) dt + \sqrt{2 \beta^{-1}} dB_t,
\end{align}
where $(B_t)_{t\geq 0}$ is a standard Brownian motion in $\mathds{R}^d$.
SGLD is a discrete version of the above dynamics, \ie, it is a stochastic gradient descent algorithm where standard Gaussian noise is added to the unbiased estimation of the gradient at each iteration: 
\begin{align}
    \label{eq:sgld}
    \forall k \in \mathds{N}, ~ W_{k+1} = W_k - \eta_{k+1} \hat{g}_{k+1} + \sqrt{2 \beta^{-1} \eta_{k+1}} \epsilon_{k+1}, \quad \epsilon_{k+1} \sim \mathcal{N}(0, I_d),
\end{align}
where $\eta_k$ is the step size or the learning rate at iteration $k$, $\beta$ is an inverse temperature parameter, $\hat{g}_{k+1}$ is an unbiased estimate of $\nabla \er(W_k)$, and $\epsilon_{k}$ are independent realizations of $\mathcal{N}(0, I_d)$  drawn at each iteration. As noted by \citet{raginsky2017non}, in the case where $\hat{g}_k = \nabla \er(W_k)$, \Cref{eq:sgld} is exactly the Euler-Maruyama discretization of \Cref{eq:langevin_equation}.

Several works have successfully proven generalization bounds for both CLD and SGLD by focusing on the generalization error on the last iterate, \eg, \citep{mou2018generalization,farghly2021time,neu2021information}, see \Cref{tab:cld,tab:sgld} for more details. However, we argue that proving data-dependent uniform generalization bounds over their trajectories might provide additional insight.

Indeed, when considering an iterative stochastic optimization algorithm, the choice of a stopping criterion is often arbitrary and without a theoretical basis. Uniform generalization bounds over the trajectory address this issue by providing performance guarantees, regardless of the stopping time. Usual generalization bounds over the last iterate do not hold for data-dependent stopping times.

Moreover, when gradient dynamics converge toward a local minimum of the empirical risk, the behavior of the algorithm around this point may be seen as a characteristic of this local minimum. By proving bounds that are uniform over the data-dependent optimization trajectory around a local minimum, the theory can capture geometric and statistical properties that are particular to this local minimum. Therefore, if the trajectory of the algorithm is considered near a local minimum, uniform bounds on the trajectory provide information that bounds over the last iterate fail to capture. These remarks, in fact, extend beyond the study of Langevin dynamics to any other gradient-based algorithm.

\begin{table}[!t]
\small
\centering
\begin{threeparttable}
\begin{tabular}{c c || c c} 
\toprule
{Abbreviation}  & {Meaning} & {Abbreviation} & {Meaning}\\
\midrule
    B. & {Loss bounded by $B>0$} & L. & {Loss is $L$-Lipschitz $\mathcal{C}^0$} \\
     SG. & {$\ell(w,z)$ subgaussian w.r.t. $z$} & R. & {$\ell_2$-regularization} \\
     {BS-$b$} & {Batch size is $b \in \mathds{N}^\star$} & $T \to \infty$ & {Long-time limit} \\
     {D.} & {Dissipativity} & {$\E_\mathcal{A}$} & {Expectation on the noise} \\
     {$\E_S$} & {Expected bound} & {HP$^{(\zeta)}$} & {With high probability ($1 - \zeta$)} \\
\bottomrule
\end{tabular}
\caption{Summary of the abbreviations used in the Tables \ref{tab:fractal_bounds}, \ref{tab:cld} and \ref{tab:sgld}. They concern the different assumptions that can be made on the loss $\ell(w,z)$, as well as the different types of bounds that can be proven.}
\label{tab:notations}
\vspace{-5mm}
\end{threeparttable}
\end{table} 
\begin{table}[!t]
\tabcolsep=0.11cm
\small
\centering
\begin{threeparttable}
\begin{tabular}{@{} l S[table-format=6.0] l S[table-format=5.0] cc @{}} 
\toprule
{}  & {Asmp. on $\ell(w,z)$} & {Fractal dim.} & {IT term} & {Smallest IT}  \\
\midrule
  \citet{simsekli2020hausdorff} & {SG., L.} & Euclidean  & {$\simeq I_\infty(S, N_\delta)$} & \\
  \citet{hodgkinson2022generalization} & {B., L.} &  Euclidean & {$I_\infty(S, \wcal)$} &  \\
  \citet{dupuis2023generalization} & {B.} & Data-dependent & {$\max_j  I_\infty(S, N_{\delta,j})$} & \\
  \textbf{Ours} & {B., L.} &  Euclidean & {$\log \frac{d\rho_S}{d \pi}(\wcal)$} & {\checkmark} \\
  \textbf{Ours} & {B.} & Data-dependent  & {$\log \frac{d\rho_S}{d \pi}(\wcal)$} & {\checkmark} \\
\bottomrule
\end{tabular}
\caption{Overview of the comparison between our work and existing fractal generalization bounds. The abbreviations used in this table can be found in \Cref{tab:notations}. We refer to the respective papers for the exact definition of the introduced IT terms. We can see in this table that our method yields smaller IT terms.}
\label{tab:fractal_bounds}
\vspace{-5mm}
\end{threeparttable}
\end{table} 

\subsection{Contributions and Overview of Main Results}

In this paper, we prove data-dependent uniform generalization bounds that address the issues mentioned above. More precisely, we propose a theoretical framework in which such uniform bounds can be derived from a single-proof technique without the need for technical assumptions specific to each problem.

Our approach will be to extend the so-called PAC-Bayesian analysis techniques \citep{shawetaylor1997pac, mcallester1998some,catoni2007pac} to random hypothesis sets, which will be formally defined in \Cref{sec:pac-bayes-bounds}. This framework has proven to be able to provide practical generalization bounds, even for neural networks. However, to the best of our knowledge, no equivalent for data-dependent uniform bounds has been studied from the PAC-Bayesian perspective. 

As opposed to the classical setting where PAC-Bayesian theory relies on probability distributions defined over one predefined hypothesis set, we will consider stochastic algorithms that ``generate random hypothesis sets'' $\wcal_S$, which follow a data-dependent probability distribution, denoted $\rho_S$, over the space of possible hypothesis sets, \ie, we have $\wcal_S \sim \rho_S$.
We will refer to our proposed methods as being a \emph{PAC-Bayesian framework for random sets.} In other words, we will construct a setting where a stochastic learning algorithm generates a random set of vectors rather than a single random vector.

Our contributions are detailed as follows.

\defcitealias{futami2023time}{Futami et al. (2023)}
\defcitealias{farghly2021time}{Farghly et al. (2021)}

\begin{enumerate}[noitemsep,topsep=0pt,leftmargin=.2in]
    \item We rigorously formalize the PAC-Bayesian theory for random sets and state uniform generalization upper bounds, which hold in this general framework. Informally, we can summarize our bounds with the following statement. With high probability, we have:
\begin{align}
    \label{eq:informal_summary_of_results}
    \sup_{w \in \wcal_S} \left( \risk(w) - \er(w) \right) \lesssim \landau{\sqrt{\frac{ \text{IT}(\rho_S, \pi) + \mathcal{C}(\wcal_S)}{n}}},
\end{align}
\begin{table}[!t]
\tabcolsep=0.11cm
\small
\centering
\begin{threeparttable}
\begin{tabular}{@{} l S[table-format=6.0] l S[table-format=5.0] cc @{}} 
\toprule
{Paper}  & {Bounded quantity} & {Type} & {Asmp. on $\ell(w,z)$} & {Bound ($\mathcal{O}$)}\\
\midrule
  \citet{mou2018generalization} & ${\E_\acal \err_S(W_T)}$ & {$\mathds{E}_S$} & {L. B.}  & {$\frac{LB \sqrt{\beta T}}{n}$} \\
  \citet{mou2018generalization} & ${\E_\acal \err_S(W_T)}$ & {HP} & {SG. R.}  & {$\left\{\frac{\beta }{n} 
  \int_0^T e^{-R_{t}^T} \mathds{E} \normof{g_S(t)}^2 dt\right\}^{\frac{1}{2}}$} \\
  \citet{li2020generalization} & ${\E_\acal \err_S(W_T)}$ & {$\mathds{E}_S$} & {L. B. $T \to \infty$}  & {$\frac{e^{4\beta B} BL \sqrt{\beta}}{n \sqrt{\lambda}}$} \\
 \citetalias{futami2023time} & ${\E_\acal \err_S(W_T)}$ & {$\mathds{E}_S$} & {SG. D.}  & {$\sqrt{\frac{\min(1,\eta T)}{n}}$} \\
   \textbf{Ours} & ${\E_\acal  \underset{0\leq t \leq T}{\sup} \err_S(W_t)}$ & {HP} & {L. B.}  & {$\frac{ BL \sqrt{\beta T}}{\sqrt{n}}$} \\
   \textbf{Ours} & ${\E_\acal \underset{0\leq t \leq T}{\sup} \err_S(W_t)}$ & {HP} & {B.}  & {$B \left\{ \frac{\beta}{n} \int_0^T \E \normof{g_S(t)}^2 dt \right\}^{\frac{1}{2}}$} \\
   \textbf{Ours} & ${\E_\acal \underset{0\leq t \leq T}{\sup} \err_S(W_t)}$ & {HP} & {L. B.}  & {$\frac{B}{\sqrt{n}} + B \frac{\beta T L^2}{n}$} \\
\bottomrule
\end{tabular}
\caption{Comparison of our CLD bounds with some classical bounds from the literature. The notation $\err_S(w)$ denotes $\risk(w) - \er(w)$. The abbreviations used in this table can be found in \Cref{tab:notations}. $\E_\acal$ denotes the expectation over the randomness of the algorithm. $g_S(t)$ is a shortcut for $\nabla \er(w_t)$. We refer to the corresponding papers for their exact statements and definitions. }
\label{tab:cld}
\vspace{-5mm}
\end{threeparttable}
\end{table} 

where $\pi$ is a data-independent distribution over the space of hypothesis sets, $\wcal_S \sim \rho_S$ is a \emph{random hypothesis set} following the distribution $\rho_S$, $\mathcal{C}(\wcal_S)$ represents a notion of complexity of $\wcal_S$, and $\text{IT}(\rho_S, \pi)$ denotes an Information Theoretic (IT) term that measures how ``far away'' $\rho_S$ and $\pi$ are from each other.
A typical example of an IT term is the Kullback-Leibler (KL) divergence.
The quantity $\mathcal{C}(\wcal_S)$ may differ between different applications.
In particular, we prove in \Cref{theorem:data_dep_rademacher_set} that $\mathcal{C}(\wcal_S)$ can take the form of a \emph{data-dependent Rademacher complexity}.
As an additional technical contribution, we show that our setup naturally paves the way for new data-dependent uniform lower bounds.

\item Thanks to our PAC-Bayesian framework for random sets, we derive generalization bounds using the two commonly used notions of fractal dimensions in the literature. All our bounds share the same information-theoretic terms and the same proof technique, which is a direct consequence of the aforementioned bound with data-dependent Rademacher complexity. Therefore, our theory simplifies the previous approaches to fractal-based generalization.
In accordance with \Cref{eq:informal_summary_of_results}, our bounds have the following form:
\begin{align*}
    \sup_{w \in \wcal_S} \left( \risk(w) - \er(w) \right) \lesssim \landau{\sqrt{\frac{(\text{Fractal Dim. of }\wcal_S) + \text{IT}(\rho_S, \pi)}{n}}}, 
\end{align*}
with $S\sim\datadist$ and $\wcal_S \sim \rho_S$. Moreover, we show that our IT terms are smaller and, therefore, yield tighter bounds, in addition to being more intuitive than the existing terms in the literature.
Our main contributions to fractal generalization bounds are summarized in \Cref{tab:fractal_bounds}. As one can see in \Cref{tab:fractal_bounds}, two types of fractal dimensions are mainly used in the literature. The first one is induced by the Euclidean distance and we abbreviate it by ``Euclidean'' in the table. The second one is induced by a data-dependent pseudometric, and we call it ``data-dependent'' \citep{dupuis2023generalization}, see \Cref{sec:fractal_bounds} for technical details. In particular, we can recover bounds with the data-dependent fractal dimension and the same mutual information as in \citep[Theorem $3.8$]{dupuis2023generalization}, but without the need for any stability assumption. Moreover, our proof technique makes it possible to improve the convergence rate in $n$, compared to \citep[Theorem $3.8$]{dupuis2023generalization}: we obtain a $n^{-1/2}$ rate, whereas the prior work had a $n^{-1/3}$ rate.

\begin{table}[!t]
\tabcolsep=0.11cm
\small
\centering
\begin{threeparttable}
\begin{tabular}{@{} l S[table-format=6.0] l S[table-format=5.0] cc @{}} 
\toprule
{Paper}  & {Bounded quantity} & {Type} & {Asmp. on $\ell(w,z)$} & {Bound ($\mathcal{O}$)}\\
\midrule
\citet{raginsky2017non} & ${\E_\acal \err_S(W_T)}$ & {$\mathds{E}_S$} & {SG. D.}  & {$\eta T + \frac{1}{n} + e^{-\eta T /c}$} \\[2mm]
  \citet{mou2018generalization} & ${\E_\acal \err_S(W_T)}$ & {$\mathds{E}_S$} & {L. B.}  & {$\frac{LB}{n} \sqrt{\beta \sum_k \eta_k}$} \\
  \citet{mou2018generalization} & ${\E_\acal \err_S(W_T)}$ & {HP} & {L. SG. BS-$1$ R.}  & {$\left\{\frac{\beta}{n} \sum_{k=1}^T e^{-R^T_k} \eta_k \E \normof{g_k}^2 \right\}^{\frac{1}{2}}$} \\
  \citet{pensia2018generalization} & ${\err_S(W_T)}$ & {HP$^\zeta$} & {L. SG. BS-$1$}  & {$L \sqrt{\frac{\beta}{n \zeta} \sum_{k=1}^T \eta_k}$} \\
  \citet{negrea2019information} & ${\E_\acal \err_S(W_T)}$ & {$\E_S$} & {SG. BS-$b$}  & {$\left\{\frac{1}{bn} \sum_{k=1}^T \beta_k \eta_k \text{Var} (g_k) \right\}^{\frac{1}{2}} $} \\
  \citet{haghifam2020sharpened} & ${\E_\acal \err_S(W_T)}$ & {$\E_S$} & {B. BS-$n$}  & {$ \frac{B}{n}  \left\{\sum_{k=1}^T \eta_k \beta_k \normof{\zeta_k}^2 \right\}^{\frac{1}{2}} $} \\
  \citet{neu2021information} & ${\E_\acal \err_S(W_T)}$ & {$\E_S$} & {SG. BS-$1$}  & {$\left\{\frac{\beta}{n} \sum_{k=1}^T \eta_k \text{Var} (g_k|w_k) \right\}^{\frac{1}{2}} $} \\
  \citetalias{farghly2021time} & ${\E_\acal \err_S(W_T)}$ & {$\E_S$} & {SG. D.}  & {$ \min \set{1, \eta T} \left( \sqrt{\eta} + \frac{1}{n \sqrt{\eta}} \right) $} \\
   \citetalias{futami2023time} & ${\E_\acal \err_S(W_T)}$ & {$\mathds{E}_S$} & {SG. D.}  & {$\sqrt{\frac{\min(1,\eta T)}{n}}$} \\
  \textbf{Ours} & {$\E_\acal  \underset{1 \leq k \leq T}{\max} \err_S(W_k)$} & {HP} & {B. L. BS-$b$}  & {$\frac{BL}{\sqrt{n}}  \sqrt{\beta\sum_{k=1}^T \eta_k}$} \\
  \textbf{Ours} & {$\E_\acal \underset{1 \leq k \leq T}{\max} \err_S(W_k)$} & {HP} & {B. BS-$b$}  & {$B \left\{ \frac{\beta}{n} \sum_{k=1}^T \eta_k \E \normof{g_k}^2 \right\}^{\frac{1}{2}}$} \\
  \textbf{Ours} & {$\E_\acal \underset{1 \leq k \leq T}{\max} \err_S(W_k)$} & {HP} & {B. L. BS-$n$}  & {$B\frac{L^2 \beta}{n} \sum_{k=1}^T \eta_k$} \\
\bottomrule
\end{tabular}
\caption{Comparison of our SGLD bounds with some classical bounds existing in the literature. The notation $\err_S(w)$ denotes $\risk(w) - \er(w)$ as before. The abbreviations used in this table can be found in \Cref{tab:notations}. $\E_\acal$ denotes the expectation over the randomness of the algorithm. We do not formally define here the variance terms appearing in \citep{negrea2019information} and \citep{neu2021information}, as well as the variation on the gradient $\zeta_t$ used in \citep{haghifam2020sharpened}, see the corresponding papers.}
\label{tab:sgld}
\vspace{-5mm}
\end{threeparttable}
\end{table} 

\item We use our new techniques to derive uniform generalization bounds over the trajectory of CLD. For a fixed time horizon $T$, our main bound states that, with high probability
\begin{align}
    \label{eq:cld-informal-main-bound}
   \Eof[\acal]{\sup_{0\leq t \leq T} \left(\risk(W_t) - \er(W_t) \right)}  \lesssim \landau{\sqrt{ \frac{1}{n\sigma^2} \int_0^T  \mathds{E}_{\acal} \big[\Vert \nabla \er(W_t) \Vert^2 \big] dt } },
\end{align}
where $\E_\acal$ denotes the expectation over the randomness of the algorithm, \ie, the expectation over the data-dependent distribution $\rho_S$, describing the law of the random hypothesis set (the trajectory) generated by the Langevin equation.
Our results, summarized in \Cref{tab:cld}, yield, to our knowledge, the first uniform bounds over the trajectories of CLD. Our bounds have an order of magnitude that is coherent with the existing literature on Langevin dynamics. In particular, they relate the uniform generalization error to the average gradient norm of the empirical risk along the random trajectory. 
We also apply our methods to the SGLD recursion and prove high-probability uniform generalization bounds of the following form:
\begin{align}
    \label{eq:sgld-informal-bound}
    \Eof[\acal]{\max_{1 \leq k \leq T} \left( \risk(W_k) - \er(W_k) \right)} \lesssim \landau{\sqrt{\frac{1}{n\sigma^2} \sum_{k=1}^T \eta_k \mathds{E}_{\acal} \normof{\hat{g}_k}^2}}.
\end{align}
\end{enumerate}
These uniform generalization bounds for SGLD have a similar form as their continuous-time counterpart, presented in \Cref{eq:cld-informal-main-bound}. On the left-hand side of \Cref{eq:sgld-informal-bound}, the expectation $\mathds{E}_{\mathcal{A}}$ is taken outside of the maximum. Therefore, this result could not be trivially deduced from the existing results presented in \Cref{tab:sgld}.

\subsection{Organization of the paper}

We start, in \Cref{sec:preliminaries}, by introducing notation and describing two varieties of classical generalization bounds on which we build part of our theory, namely Rademacher complexity bounds (\Cref{sec:uniform-convergence-bounds}) and PAC-Bayesian bounds (\Cref{sec:pac-bayes-bounds}). We introduce our PAC-Bayesian framework for random sets in full generality in \Cref{sec:pac_bayes_for_random_sets} and prove our general data-dependent uniform generalization bounds in \Cref{sec:main-result-data-dependent}, within the introduced framework. The remainder of the paper is devoted to applications of these bounds to fractal-based generalization bounds (\Cref{sec:fractal_bounds}), CLD (\Cref{sec:langevin_girsanov}), and SGLD (\Cref{sec:bounds_for_sgld}). All the proofs of the main results are provided in \Cref{sec:posponed_proofs}.

\section{Preliminaries}
\label{sec:preliminaries}

In this section, we recall some notations and a few existing uniform generalization bounds and PAC-Bayesian bounds in \Cref{sec:uniform-convergence-bounds,sec:pac-bayes-bounds}, respectively. These results will be used throughout the paper.

\subsection{Notations}
\label{sec:notations}

We use the notations $\risk$, $\er$, $G_S$, $\wcal$, $\ell$, $\mu_z$ and $(\zcal, \mathcal{F})$ as they have been defined in the introduction. With a slight abuse of notation, we will write $G_S(w)$, or $\err_S(w)$, for $G_S(\set{w})$. When $\wcal$ depends on the data, it is said to be data-dependent. A fixed (or data-independent) set $\wcal$ will be called a fixed hypothesis set.

Given two probability measures $\mu$ and $\nu$, the absolute continuity of $\mu$ with respect to $\nu$ will be denoted $\mu \ll \nu$; the equivalence between $\mu$ and $\nu$ (\ie, $\mu\ll\nu$ and $\nu\ll\mu$) will be denoted $\mu \sim \nu$. Given $\mu$ a probability measure and $X$ a random variable on the same probability space, we define the image measure (or pushforward) $X_\#\mu$, by $X_\#\mu(B) := \mu(X^{-1}(B))$.
To differentiate between different settings, probability measures will be denoted $(\pcal, \qcal_S)$ when they are distributions over parameter vectors and $(\pi,\rho_S)$ when they are distributions over sets. 
Definitions of the MI terms (in particular $I_\infty$, which will appear in our bounds) are provided in \Cref{sec:appendix_technical_background}, along with additional technical background.

Throughout the text, we will use the notion of \emph{Markov kernel}. A family $(\qcal_S)_{S \in \zcal^n}$ of probability distributions, on a measurable space $(\Omega, \mathcal{T})$  is a Markov kernel on $\Omega \times \zcal^n$ if, for all $A \in \mathcal{T}$, the map $S \longmapsto \qcal_S(A)$ is $\mathcal{F}^{\otimes n}$-measurable. 
We denote by $\mathfrak{P}(\mathds{R}^d)$ the set of all subsets of the parameter space $\Rd$. Given a finite set $A$, its cardinality will be denoted $|A|$.

\subsection{Uniform generalization bounds with data-independent hypothesis sets}
\label{sec:uniform-convergence-bounds}

In this section, we present existing results on uniform generalization bounds for fixed hypothesis sets, \ie, bounding the quantity $G_S(\wcal)$ defined in \Cref{eq:worst_gen_error_intro}. While several approaches exist in the literature \citep{vapnik1968uniform,vapnik1971uniform}, we focus on uniform generalization bounds based on the so-called Rademacher complexity~\citep{koltchinskii2004rademacher,koltchinskii2001rademacher,koltchinskii2002empirical,bartlett2002rademacher,bartlett2002model}, which we will use for the proof of some of our main results. Throughout this section, $\wcal \subset \Rd$ denotes a \emph{fixed} hypothesis set.

The next theorem provides the definition of Rademacher complexity, along with a known high probability upper bound of $G_S(\wcal)$, see, \eg,\ \citet[Theorem $3.3$]{mohri2018foundations}.

\begin{restatable}[Uniform generalization bounds with the Rademacher complexity]{theorem}{theoremrademacherbound}
\label{theorem:rademacher-bound}
For any bounded loss function $\ell: \mathds{R}^d\times\zcal\to[0,B]$, where $B>0$ is a constant, we have
\begin{align}
    \mathds{P}_{S\sim \datadist}\!\left(\sup_{w\in\wcal}\left( \risk(w) - \er(w) \right) \le 2\rad_S(\wcal) + 3B\sqrt{\frac{\log(1/\zeta)}{2n}}\right)\ge 1-\zeta,\label{eq:rademacher-bound-2}
\end{align}
where $\rad_S(\wcal)$ is the empirical Rademacher complexity, defined as
\begin{align}
    \label{def:empirical_rad_definition}
    \rad_S(\wcal) := \frac{1}{n}  \mathds{E}_{\epsilon} \left[\sup_{w \in \wcal}  \sum_{i=1}^n \epsilon_i \ell(w,z_i)  \right].
\end{align}
In this equation $\epsilon := (\epsilon_1,\dots,\epsilon_n)$ is a vector of \iid Rademacher random variables, characterized by $\prob(\epsilon_i=1) = \prob(\epsilon_i=-1)=1/2$.
\end{restatable}

We also define the (expected) Rademacher complexity as $\rad(\wcal) := \mathds{E}_{S\sim\datadist} \left[ \rad_S(\wcal)\right]$.

The Rademacher complexity $\rad(\wcal)$, and its empirical counterpart $\rad_S(\wcal)$, can be interpreted as a complexity measure of the hypothesis set $\wcal$ \citep{shalevschwartz2014understanding}. 
As examples of applications of these Rademacher complexity-based bounds, by instantiating the hypothesis set $\wcal$, one can upper-bound the Rademacher complexity for linear classifiers~\citep{bartlett2002rademacher,kakade2008complexity,awasthi2020rademacher} or neural networks~\citep{neyshabur2015norm,bartlett2017spectrally}.

The key ingredient in proving \Cref{theorem:rademacher-bound} is the so-called symmetrization lemma, \ie, \Cref{lemma:symmetrization}.
It also plays a great role in our analysis, see \Cref{sec:data_dep_rademacher_bound}.
The fact that this lemma does not hold in the case of a data-dependent hypothesis set has already been noted by several studies \citep{foster2019hypothesis, dupuis2023generalization}.
This is one of the main bottlenecks for having tight uniform generalization bounds.
In \Cref{sec:data_dep_rademacher_bound}, we will show how we can leverage the PAC-Bayesian theory (see \Cref{sec:pac-bayes-bounds}) to remove this limitation and obtain bounds with Rademacher complexities of data-dependent hypothesis sets. 

\subsection{Background on PAC-Bayesian bounds}
\label{sec:pac-bayes-bounds}

In the above section, we presented an example of uniform generalization bound over a fixed hypothesis set $\wcal$. 
Contrary to the uniform generalization bounds, the PAC-Bayesian framework takes a radically different viewpoint by considering \emph{randomized} hypotheses; each hypothesis in $\wcal$ is associated with a weight characterizing its importance~\cite[see \eg, ][]{alquier2024user}. These weights are represented by probability measures.
We distinguish two kinds of probability measures on $\wcal$:
{\it (i)} $\pcal$, called the \emph{prior distribution}\footnote{We assume that the prior and posterior distributions on $\wcal$ are defined on an arbitrary $\sigma$-algebra $\Sigma_{\wcal}$.}, which does not depend on $S$.
{\it (ii)} $\qcal_S$, called the \emph{posterior distribution}, that is learned based on the data $S$.
In this setting, we can study the generalization gap of {\it randomized} hypotheses sampled from $\qcal_S$ thanks to the \emph{expected} generalization gap (with an expectation over $\qcal_S$):
\begin{align}
\label{eq:expected-gap}
\mathds{E}_{w\sim\qcal_S}\left[ \risk(w) - \er(w) \right], \quad \text{with } S \sim \datadist.
\end{align}

The PAC-Bayesian framework provides us techniques to bound the expected generalization gap: we present the following well-known bound from \citet{mcallester2003pac,maurer2004note}.

\begin{theorem}[\citet{mcallester2003pac,maurer2004note}]
\label{ex:example_particular_pb}
We assume that $\ell$ is bounded in $[0,1]$.
Let $\pcal$ be any prior distribution on $\wcal$.
For any Markov kernel\footnote{This notion has been defined in \Cref{sec:preliminaries}.} $\qcal_S$, if for all $S \in \zcal^n$, we have $\qcal_S \ll \pcal$, then with probability at least $1 - \zeta$ over $S\sim\datadist$ we have:
\begin{align*}
    \mathds{E}_{w\sim\qcal_S}\left[\risk(w) - \er(w)\right] \leq \sqrt{\frac{\klb{\qcal_S}{\pcal} + \log(\frac{2\sqrt{n}}{\zeta})}{2n}}
\end{align*}
where $\displaystyle\klb{\qcal_S}{\pcal}$ is the Kullback Leibler (KL) divergence between $\qcal_S$ and $\pcal$, whose definition is recalled in \Cref{sec:appendix_technical_background}.
\end{theorem}

\Cref{ex:example_particular_pb} is a special case of a more general result. 
Considering any suitable function $\phi: \wcal\times\zcal^n \to \R$, one can prove the following PAC-Bayesian bound \citep{germain2009pac}.

\begin{restatable}[PAC-Bayesian bound of \citet{germain2009pac}]{theorem}{theoremexistingpacbayesgermain}
\label{theorem:existing-pac-bayes-germain}
Let $\zeta \in (0,1)$ and consider any measurable function $\phi: \wcal\times\zcal^n\to\R$. 
Let $\pcal$ be any prior distribution on $\wcal$, such that $e^\phi$ is in $L^1(\pcal \otimes \datadist)$.
We have, for any Markov kernel $\qcal_S$ such that, for all $S \in \zcal^n$, we have $\qcal_S \ll \pcal$ and $\phi(\cdot,S) \in L^1(\qcal_S)$:
\begin{align*}
\mathds{P}_S \left( \mathds{E}_{w\sim\qcal_S}[\phi(w, S)] \leq \klb{\qcal_S}{\pcal} + \log\frac{1}{\zeta} + \log \mathds{E}_S\mathds{E}_{w \sim \pcal}e^{\phi(w, S)} \right) \geq 1 - \zeta,
\end{align*}
\end{restatable}

Put into words, \Cref{theorem:existing-pac-bayes-germain} is an upper-bound of $\mathds{E}_{w\sim\qcal_S}\phi(w, S)$, holding with probability at least $1-\zeta$ w.r.t.\ the sample $S$, and depending essentially on two terms.
{\it (i)} The KL divergence between the probability measures $\qcal_S$ and $\pcal$, which can be interpreted as a ``distance'' (that is not symmetric).
{\it (ii)} The term $\log \mathds{E}_S\mathds{E}_{w \sim \pcal}e^{\phi(w, S)}$ that can be further upper-bounded when the function $\phi$ is instantiated.
For instance, when $\phi(w, S)=2n(\risk(w){-}\er(w))^2$, \Cref{theorem:existing-pac-bayes-germain} boils down to \Cref{ex:example_particular_pb}.

\noindent{}PAC-Bayesian results such as \Cref{theorem:existing-pac-bayes-germain} have been applied to various machine learning models such as linear classifiers~\citep[][]{herbrich2000pac,langford2002pac,ambroladze2006tighter,langford2005tutorial,germain2009pac,parradohernandez2012pac}, majority votes~\citep{lacasse2006pac,germain2015risk,zantedeschi2021learning}, or stochastic neural networks~\citep{langford2001not,dziugaite2017computing}.
However, in practice, we may be interested in a single hypothesis $w$, and, therefore, may want to control the generalization gap
\begin{align}
\left\{ \risk(w) - \er(w) \right\} \quad\text{where: }S \sim \datadist,~ w\sim\qcal_S,\label{eq:disintegrated-gap}
\end{align}
instead of integrating over the posterior distribution $\qcal_S$.
This is made possible, to an extent, using so-called
\emph{disintegrated} PAC-Bayesian bounds~\citep{blanchard2007occam,catoni2007pac,rivasplata2020pac,viallard2024general}.
We present the bound proven by \citet[Theorem~1{\it (i)}]{rivasplata2020pac}.

\begin{restatable}[Disintegrated PAC-Bayesian bound of \citet{rivasplata2020pac}]{theorem}{theoremexistingpacbayesrivasplata}
\label{theorem:existing-pac-bayes-rivasplata}
Let $\zeta \in (0,1)$ and consider any measurable function $\phi: \wcal\times\zcal^n\to\R$.
Let $\pcal$ be any prior distribution on $\wcal$, such that $e^\phi$ is in $L^1(\pcal \otimes \datadist)$.
We have, for any Markov kernel $\qcal_S$ such that, for all $S \in \zcal^n$, we have $\qcal_S \ll \pcal$ and $\phi(\cdot,S) \in L^1(\qcal_S)$:
\begin{align*}
\mathds{P}_{S, w\sim\qcal_S}\left( \phi(w,S)\leq  \log \left( \frac{d \qcal_S}{d\pcal}(w) \right) + \log\frac{1}{\zeta} + \log \mathds{E}_S\mathds{E}_{w \sim \pcal}e^{\phi(w, S)}  \right) \ge 1-\zeta.
\end{align*}
\end{restatable}

\noindent{}We can see two main differences with \Cref{theorem:existing-pac-bayes-germain}: \Cref{theorem:existing-pac-bayes-rivasplata} is an upper-bound on $\phi(w,S)$ instead of the expectation $\E_{w\sim\qcal}\phi(w,S)$ and the KL divergence is replaced by the logarithm of the Radon-Nikodym derivative $\log\frac{d \rho_S}{d\pi}(w)$, evaluated at $w\sim\qcal_S$. 
Since the other terms remain unchanged, the mechanism is similar when instantiating $\phi$, \ie, the right-most term must be upper-bounded, and different choices of $\phi$ may lead to different generalization bounds. The proof of the above theorem follows from the derivations of \citet{rivasplata2020pac}; we note that it rigorously holds under the absolute continuity assumption.

\section{PAC-Bayesian Theory on Random Sets}
\label{sec:framework_introduction}

In this section, we introduce our framework for PAC-Bayesian bounds for random sets.
As mentioned in the introduction, our goal is to reformulate the PAC-Bayesian bounds for random hypothesis sets.
This section may be seen as a direct consequence of classical PAC-Bayesian theory.
Indeed, it will be noted that \Cref{theorem:existing-pac-bayes-germain,theorem:existing-pac-bayes-rivasplata} are valid in a more general setting, in which $\wcal$ is replaced by an arbitrary probability space. 
We first present, in \Cref{sec:pac_bayes_for_random_sets}, a generic approach to generalize the PAC-Bayesian bounds of \Cref{theorem:existing-pac-bayes-germain,theorem:existing-pac-bayes-rivasplata} for random sets.
In \Cref{sec:random_closed_sets_construction}, we propose a more detailed construction based on the notion of random closed sets \citep[Chapter $1$]{molchanov2017theory}, therefore providing a sound theoretical foundation for the introduced methods.

\subsection{Random set formalization}
\label{sec:pac_bayes_for_random_sets}

 We consider a set $E \subseteq \mathfrak{P}(\mathds{R}^d)$, together with a $\sigma$-algebra $\mathfrak{E}$, making $(E,\mathfrak{E})$ a measurable space. We will now rewrite known PAC-Bayesian bounds by replacing random hypothesis $w\in\mathds{R}^d$ by random hypothesis sets\footnote{The notation $\wcal$ will always refer to sets rather than points.} $\wcal \in E$. $E$ should be interpreted as the collection of all possible hypothesis sets.
According to our PAC-Bayesian approach, we consider a \emph{learning algorithm} as a mapping generating a data-dependent probability distribution on $(E,\mathfrak{E})$ from a dataset $S \in \mathcal{Z}^n$. 
More formally, this leads to the following definition.

\begin{definition}[Priors and posteriors]
    \label{def:priors_posteriors}
    A \emph{prior}, $\pi$, is a data-independent probability distribution on $(E,\mathfrak{E})$. A family of posteriors $(\rho_S)_{S \in \mathcal{Z}^n}$, is defined as a Markov kernel on $E \times \mathcal{Z}^n$. 
    We further require that the posteriors are absolutely continuous with respect to the prior, \ie $\rho_S \ll \pi$, $\datadist$-almost surely. 
\end{definition}

Whenever we consider priors and posteriors in the remainder of the paper, we assume that the properties of \Cref{def:priors_posteriors} are satisfied. This framework encompasses several classical settings, such as the following examples.

\begin{example}[Singleton distributions]
    \label{example:singletons}
    Assume that, for $\wcal \in E$, there exists $ w \in \mathds{R}^d$, such that $\wcal = \{ w \}$, $\pi$-almost surely. Then the distributions $\rho_S$ and $\pi$ naturally extend to distributions over $\mathds{R}^d$. In that case, the presented framework reduces to the classical PAC-Bayesian setting, as given in \Cref{sec:pac-bayes-bounds}, where the distributions are defined over $\Rd$.
\end{example}

\begin{example}[Stochastic Gradient Descent]
\label{example:SGD}
    Consider the Stochastic Gradient Descent algorithm (SGD), applied over $T$ iterations to the minimization of the empirical risk $\er$. For a dataset $S\in\zcal^n$ and external randomness $\mathcal{A}$, coming from the choice of the batch indices, SGD generates the iterates $(w_1^{S,\mathcal{A}}, \dots, w_T^{S,\mathcal{A}}) \in \left( \mathds{R}^d \right)^T$. Then $\rho_S$ could be defined as the conditional distribution of the sets $\left\{w_1^{S,\mathcal{A}}, \dots, w_T^{S,\mathcal{A}}\right\}$, given the data $S$.
\end{example}

\begin{example}[Stochastic Differential Equations]
    \label{example:sde}
    Consider the Stochastic Differential Equation (SDE) $dW_t^S = -\nabla \er(W_t^S) dt + \sigma dX_t$, where $\er$ is the empirical risk and $(X_t)_{t \geq 0}$ is a well-behaved stochastic process. For instance, $X$ could be a Brownian motion in the case of Langevin dynamics \citep{mou2018generalization}, or a Lévy process, in the case of heavy-tailed dynamics \citep{simsekli2019tail, gurbuzbalaban2021heavy,dupuis2024generalization}. In such a setting, given a fixed time horizon $T>0$, the posterior $\rho_S$ describes the distributions of the sets $\{ W_t^S,~0\leq t \leq T \} \subset \mathds{R}^d$. We will cover such a setting in more detail in \Cref{sec:langevin_girsanov}.
\end{example}

We can now extend the PAC-Bayesian bounds of \Cref{theorem:existing-pac-bayes-germain,theorem:existing-pac-bayes-rivasplata} in a straightforward manner. To do so, we only need to consider a function $\Phi: E \times \zcal^n \longrightarrow \mathds{R}$ measurable with respect to $\mathfrak{E} \otimes \mathcal{F}^{\otimes n}$, and apply \Cref{theorem:existing-pac-bayes-germain,theorem:existing-pac-bayes-rivasplata}. The following example illustrates a typical such function $\Phi$; other pertinent choices will be discussed in \Cref{sec:main-result-data-dependent}.

\begin{example}[Supremum function]
    \label{example:sup_function}
    The supremum of the generalization error,
    \begin{equation}
        \label{eq:phi_sup}
        \Phi_{\sup} (\wcal, S) := \sup_{w \in \wcal} \left( \risk(w) - \er(w) \right),
    \end{equation}
    \looseness=-1
    can be used under mild assumptions over $E$ and $\ell$. For instance, it is the case for almost surely finite hypothesis sets or for random closed sets, as it will be discussed in \Cref{sec:random_closed_sets_construction}.
\end{example}

The above example will be further refined in \Cref{sec:main-result-data-dependent}, which illustrates the capacity of our approach to prove worst-case generalization bounds over data-dependent hypothesis sets. The following theorem is a direct consequence of \Cref{theorem:existing-pac-bayes-germain,theorem:existing-pac-bayes-rivasplata}.

\begin{restatable}[PAC-Bayesian bounds for random sets]{theorem}{theorempacbayesset}
\label{theorem:pac-bayes-set}
Let $(E,\mathfrak{E})$ be defined as before and $\Phi: E \times \zcal^n \longrightarrow \mathds{R}$ be a measurable function. We also consider a prior $\pi$ and posteriors $(\rho_S)_{S \in \mathcal{Z}^n}$, as in \Cref{def:priors_posteriors}. Then we have for any $\zeta \in (0,1)$:
\begin{equation}
\label{eq:set_kl_bound}
\mathds{P}_S \left( \mathds{E}_{\wcal\sim\rho_S}\Phi(\wcal, S) \leq \klb{\rho_S}{\pi} + \log(1/\zeta) + \log \mathds{E}_S\mathds{E}_{\wcal \sim \pi} \left[ e^{\Phi(\wcal, S)} \right] \right) \geq 1 - \zeta,
\end{equation}
as well as the disintegrated bound
\begin{equation}
\label{eq:set_disintegrated_bound}
\mathds{P}_{S,\wcal\sim\rho_S} \left(\Phi(\wcal,S)\leq  \log \left( \frac{d \rho_S}{d\pi}(\wcal) \right) + \log(1/\zeta) + \log  \mathds{E}_S\mathds{E}_{\wcal \sim \pi} \left[ e^{\Phi(\wcal,S)} \right] \right) \geq 1 - \zeta,
\end{equation}
with both bounds holding as long as all expectations appearing are well-defined.
\end{restatable}
\begin{proof}
It is a consequence of the results presented in \Cref{sec:pac-bayes-bounds}, namely \Cref{theorem:existing-pac-bayes-germain,theorem:existing-pac-bayes-rivasplata}, adapted to the Markov kernel $(\rho_S)_{S \in \zcal^n}$ and the probability space $(E,\mathfrak{E})$.
\end{proof}
To further illustrate that our framework generalizes that of classical PAC-Bayesian theory, we provide the following example.

\begin{example}
    In the singleton distributions setting of \Cref{example:singletons}, \Cref{theorem:pac-bayes-set} is equivalent to classical PAC-Bayesian bounds found in \citep{alquier2024user, rivasplata2020pac, germain2009pac}, if we use the supremum function of \Cref{eq:phi_sup}. Indeed, in that case, if the loss $\ell$ is bounded by $B$, we have, by Hoeffding's lemma and Fubini's theorem:
    $$
     \mathds{E}_S\mathds{E}_{\wcal \sim \pi} \big[ e^{\lambda\Phi_{\sup}(\wcal,S)} \big] = \mathds{E}_{\{ w \}\sim \pi} \mathds{E}_S \big[ e^{\lambda\left(\risk(w) - \er(w)\right)} \big] \leq e^{\frac{\lambda^2B^2}{8n}}.
    $$
\end{example}

\Cref{theorem:pac-bayes-set} shows that, to deduce meaningful generalization bounds over data-dependent hypothesis sets, one must be able to bound both the IT terms, namely $\klb{\rho_S}{\pi}$ and $\frac{d \rho_S}{d\pi}(\wcal)$, on one side, and the log-exp terms, $ \log\mathds{E}_S\mathds{E}_{\wcal \sim \pi} \big[ e^{\Phi(\wcal,S)} \big] $, on the other side.
In the following sections, we will show, through several examples, how to design well-suited functions $\Phi$ to obtain such bounds.
We will also describe techniques that help analyze the IT terms in several cases.
In the general case, these IT terms measure the deviation of the posterior distribution $\rho_S$ from the prior distribution $\pi$.
In particular, the Radon-Nikodym derivative term measures the ratio of posterior and prior probability on each random hypothesis set and may be seen as a ``disintegrated relative entropy''. 

Let us make a short remark to highlight the generality of our framework.

\begin{remark}
    \label{rq:uniform-over-posteriors}
    \Cref{theorem:pac-bayes-set} is stated using a Markov kernel formulation, \ie, $(\rho_S)_{S\in\zcal^n}$ is a Markov kernel on $E\times \zcal^n$, according to \Cref{def:priors_posteriors}. However, PAC-Bayesian bounds are often stated uniformly over the choice of posterior distribution \citep[see \eg,][]{alquier2024user}. As \Cref{theorem:pac-bayes-set} is a direct extension of the existing PAC-Bayesian bounds, it is also possible to state \Cref{eq:set_kl_bound} in this fashion. More precisely, if $\pcal_0(\Rd)$ denotes the family of probability distributions on $\Rd$, \Cref{eq:set_kl_bound} becomes:
    \begin{align*}
        \mathds{P}_S \left( \forall \rho \in \pcal_0(\Rd),~ \Eof[\rho]{\Phi(\wcal, S)}  \leq \klb{\rho}{\pi} + \log(1/\zeta) + \log \mathds{E}_S\mathds{E}_{\pi} \left[ e^{\Phi(\wcal, S)} \right] \right) \geq 1 - \zeta.
    \end{align*}
    In our application, we will use the Markov Kernel formulation, but our KL-based bounds in \Cref{thm:rad_data_dep_dim,theorem:data_dep_general_form,theorem:data_dep_lower_bound} can be written uniformly over the choice of posterior $\rho$. Note however that, to the best of our knowledge, this uniform formulation is not possible in the disintegrated case, \ie, for \Cref{eq:set_disintegrated_bound}.
\end{remark}

In the same way that we extended \Cref{theorem:existing-pac-bayes-germain,theorem:existing-pac-bayes-rivasplata} into our framework, it is clear that other parts of the PAC-Bayesian literature may be treated the same way.
In particular, the disintegrated framework of \citet{viallard2024general} could be used similarly.
It would also be possible to use tighter versions of the proposed bounds.
For better clarity, we will focus on the bounds presented in \Cref{theorem:pac-bayes-set} in the sequel.
Additionally, more general PAC-Bayesian bounds could be formulated within our framework.
For instance, we could state PAC-Bayesian uniform generalization bounds featuring the integral probability metrics (IPM) used by \citet{amit2022integral}, see \Cref{sec:ipm-bounds} for more details. Note that our proofs can also be adapted for other bounds with IPM developed by \citet{viallard2023learning,viallard2024tighter}.

\subsection{More detailed measure-theoretic construction}
\label{sec:random_closed_sets_construction}

The methods described in \Cref{sec:pac_bayes_for_random_sets} are valid as soon as the general PAC-Bayesian theory applies to the measurable spaces $(E,\mathfrak{E})$ and the functions $\Phi$ under consideration. We already presented, through several examples (\Cref{example:singletons,example:SGD,example:sde}), settings in which this will be the case under mild assumptions on the loss $\ell(w,z)$ (\eg, bounded loss or integrability assumptions). 
We now provide a general measure-theoretic construction that allows us to apply our framework more widely. Note that this section may be skipped without harming the general understanding of the paper.

Our goal is to define a measurable space $(E,\mathfrak{E})$ of random sets with enough structure so that the measurability conditions of \Cref{theorem:pac-bayes-set} are satisfied. Inspired by \citet{molchanov2017theory}, we restrict the analysis of this section to the theory of \emph{random closed sets}. Note that, as soon as the loss function $\ell(w,z)$ is continuous in $w$, there is no loss of generality in considering only closed sets. The reader may find in \citep{molchanov2017theory} an extensive overview of the theory of random sets. A similar formalization of learning algorithms through random closed sets has been considered in \citep{hodgkinson2022generalization} and \citep{dupuis2023generalization}.

Let us denote by $\closed$ the set of closed sets in $\Rd$ and give the definition of a suitable $\sigma$-algebra on $\closed$, called the Effrös $\sigma$-algebra.

 \begin{definition}[Effrös $\sigma$-algebra]
     \label{def:random_closed_set}
     Let $\mathcal{O}(\Rd)$ be the set of open sets of $\Rd$. The Effrös $\sigma$-algebra on $\Rd$, denoted $\mathfrak{E}(\Rd)$, is the $\sigma$-algebra on $\closed$, generated by:
     $$
     \set{\mathcal{F}_U,~ U \in \mathcal{O}(\Rd)}, \quad \text{with } \mathcal{F}_U := \set{C \in \closed,~ C \cap U \neq \emptyset}.
     $$
     Given $(\Omega, \mathcal{T})$ a measurable space, as explained in \citep[Section $1.1.1$]{molchanov2017theory}, we may define a random closed set as a measurable mapping
     \begin{align*}
         \mathcal{W} : (\Omega, \mathcal{T}) \longrightarrow (\closed, \mathfrak{E}(\Rd)).
     \end{align*}
 \end{definition}

 The following lemma ensures that we now have enough structure to apply the PAC-Bayesian theory in its general form.

 \begin{restatable}{lemma}{measurabilityClosedSets}
     \label{lemma:measurability_for_closed_sets}
     Let $(\Omega, \mathcal{T})$ be a measurable space and $\zeta: \Rd \times \Omega \longrightarrow \R$ a stochastic process, which is continuous in $w\in\Rd$. Define, for $\wcal \in \closed$ and $\omega \in \Omega$, the map $\Phi(\wcal, \omega) := \sup_{w \in \wcal} \zeta(w,\omega)$. Then, $\Phi$ is measurable with respect to $\mathfrak{E}(\Rd) \otimes \mathcal{T}$.
 \end{restatable}

In particular, this implies that the supremum of the generalization error, \ie, $G_S(\wcal)$ defined by \Cref{eq:worst_gen_error_intro}, is measurable with respect to $\mathfrak{E}(\Rd) \otimes \fcal^{\otimes n}$. It also implies a similar measurability result on the Rademacher complexity terms that will appear in \Cref{sec:data_dep_rademacher_bound,sec:lower_bound}. The proofs of this section are discussed in \Cref{sec:proofs-random-closed-sets-construction}.

\section{Uniform Generalization Bounds with Data-dependent Hypothesis Sets}
\label{sec:main-result-data-dependent}

In this section, we present some of our main results, which consist of applying the framework described in  \Cref{sec:pac_bayes_for_random_sets} to informed choices of the function $\Phi$ appearing in \Cref{theorem:pac-bayes-set}. Let us first define our generic assumptions. We will explicitly mention the assumptions they require for each statement in the sequel.

\begin{assumption}[Bounded measurable loss]
    \label{ass:bounded_loss}
      The loss function $\ell: \mathds{R}^d \times \zcal$ is measurable and bounded in $[0,B]$, for some constant $B>0$.
\end{assumption}

Moreover, we denote, as in \Cref{sec:pac_bayes_for_random_sets}, the probability space of hypothesis sets as $(E,\mathfrak{E})$. We fix a prior $\pi$ and posteriors $(\rho_S)_{S \in \zcal^n}$, defined on $E$, according to \Cref{def:priors_posteriors}. We make the following technical assumption, which will ensure that all quantities appearing in the rest of this section are well-defined and measurable. Note that \Cref{sec:random_closed_sets_construction} justifies that this assumption holds in numerous settings.

\begin{assumption}[Supremum measurability]
    \label{ass:supremum_measurability}
    Both $\ell$ and $(E,\mathfrak{E})$ have enough regularity so that, for any coefficients $b,a_1,\dots,a_n \in \mathds{R}$, the following is $\mathfrak{E}\otimes \mathcal{F}^{\otimes n}$-measurable:
    $$
    (\wcal,S) \longmapsto \sup_{w \in \wcal} \sum_{i=1}^n \left( a_i \ell(w,z_i) - b \risk(w) \right).
    $$
\end{assumption}

We will prove three generalization bounds. First, in \Cref{sec:mgf_bound}, we will build up on the supremum function given in \Cref{example:sup_function} to derive our first bounds. While interesting, this approach appears to be inefficient in some cases, which will be made clear later. To solve this issue, we show, in \Cref{sec:data_dep_rademacher_bound}, how a slight change in the function $\Phi$, can lead to a generalization bound in terms of \emph{data-dependent Rademacher complexity}. Finally, in  \Cref{sec:lower_bound}, the same methods are applied to derive a data-dependent uniform lower bound. An advantage of our framework is that all these bounds come with the same interpretable IT terms and apply to a wide variety of settings.

\subsection{Warm-up: a first bound with the moment generating Rademacher function}
\label{sec:mgf_bound}

We first apply \Cref{theorem:pac-bayes-set}, to the function
\begin{equation}
    \label{eq:sup_lambda_function}
    ~\Phi_\lambda (\wcal, S) := \lambda\sup_{w\in\wcal} \left( \risk(w) - \er(w) \right)  = \lambda \Phi_{\sup}(S,\wcal), \quad \wcal \in E, ~S \in \zcal^n,
\end{equation}
for $\lambda >0$. 
The introduction of the parameter $\lambda$, in the above equation, is a classical trick in PAC-Bayesian theory, as this parameter can be further optimized in particular applications to obtain generalization bounds in a more compact form, see for instance \Cref{rq:lambda-optimization-london-technique} below.

Before writing our generalization bound, we introduce the moment generating function (MGF) of the Rademacher complexity, defined as:
\begin{align}
    \label{eq:mgf_rademacher}
    \forall \lambda > 0,\ \ \Psi_{S,\wcal}(\lambda) = \mathds{E}_\epsilon \left[ \exp \left\{ \frac{\lambda}{n} \sup_{w \in \wcal}\sum_{i=1}^n \epsilon_i \ell(w,z_i) \right\} \right],
\end{align}
where $\epsilon_1,\dots,\epsilon_n$ are \iid Rademacher random variables and $\wcal \subseteq \mathds{R}^d$ is a set.

The following theorem is a PAC-Bayesian type bound for the worst-case generalization error in terms of the MGF of the Rademacher complexity.

\begin{restatable}[PAC-Bayesian bounds with Rademacher MGF]{theorem}{theoremMGF}
    \label{theorem:mgf_rademacher}
    Under \Cref{ass:supremum_measurability}, we have, for any $\lambda>0$, the following bounds, as soon as the expectations are well defined:
    \begin{align*}
    &\mathds{P}_S \left( \mathds{E}_{\wcal \sim \rho_S} \big[ \lambda \sup_{w \in \wcal} \big( \mathcal{R}(w) - \er(w) \big) \big] \leq \klb{\rho_S}{\pi} + \log(1/\zeta) + M(\lambda) \right) \geq 1 - \zeta,\\
    &\mathds{P}_{S, \wcal \sim \rho_S} \left( \lambda \sup_{w \in \wcal} \big( \mathcal{R}(w) - \er(w) \big) \leq \log\!\LP\frac{d\rho_S}{d\pi}(\wcal)\RP +\log(1/\zeta) + M(\lambda) \right) \geq 1 - \zeta,
    \end{align*}
    with $ M(\lambda) := \log \mathds{E}_{\wcal \sim \pi}\mathds{E}_S\Psi_{S,\wcal}(2\lambda)$.
\end{restatable}

The main interest in this theorem is that it does not require any boundedness assumption on the loss. The proof of \Cref{theorem:mgf_rademacher}, deferred to \Cref{sec:proof-rademacher-mgf}, is based on an ``exponential symmetrization lemma'', similar to the usual symmetrization inequality for Rademacher complexity, \ie, \Cref{lemma:symmetrization}. We present in the following example a simple case where the term $M(\lambda)$ can be simplified and the parameter $\lambda$ accordingly optimized.

\begin{example}[Almost surely finite random sets]
    \label{example:almost_surely_finite_set}
    Let us assume that \Cref{ass:bounded_loss} holds and that $\wcal$ is $\pi$-almost surely finite (note that its cardinality is still random). Then, by Fubini's theorem and Hoeffding's lemma, we have (we mimic the proof of Massart's lemma):
    \begin{align*}
        \Psi_{S,\wcal}(2\lambda) \leq \Eof[\epsilon]{\sum_{w\in\wcal} \exp \left\{ \frac{2\lambda}{n}\sum_{i=1}^n \epsilon_i \ell(w,z_i) \right\} } \leq \sum_{w\in\wcal}e^{\frac{2\lambda^2 B^2}{n}} = |\wcal|  e^{\frac{2\lambda^2 B^2}{n}}.
    \end{align*}
    This gives that, with probability at least $1 - \zeta$ over $S\sim \datadist$ and $\wcal \sim \rho_S$:
    $$
     \lambda \sup_{w \in \wcal} \big( \mathcal{R}(w) - \er(w) \big) \leq \log\!\LP\frac{d\rho_S}{d\pi}(\wcal)\RP +\log(1/\zeta) + \log \left(\mathds{E}_\pi \left[|\wcal|  \right]\right) + \frac{2\lambda^2 B^2}{n}.
    $$
    We refer to \Cref{rq:lambda-optimization-london-technique} for a discussion on the role and optimization of the parameter $\lambda$.
\end{example}

In general, the quantity $\Psi_{S,\wcal}(2 \lambda)$ could be bounded using covering arguments, similarly to \citep{simsekli2020hausdorff} and \citep{dupuis2023generalization}. However, because of the expectation over the prior appearing in the definition of $M(\lambda)$, such techniques would only lead to non-data-dependent quantities. This can be seen in \Cref{example:almost_surely_finite_set}, where the left-hand side of the bound features the expected (over the prior $\pi$) cardinality of $\wcal$ instead of a data-dependent term. As our goal is to provide data-dependent generalization bounds, this shows that the particular choice of function $\Phi$ used in this subsection has to be improved. Nevertheless, \Cref{theorem:mgf_rademacher} will be used to derive our uniform SGLD bounds in \Cref{sec:bounds_for_sgld}.

In the next subsection, we present an alternative approach towards data-dependent generalization bounds over random sets.

\subsection{Generalization bounds with data-dependent Rademacher complexity}
\label{sec:data_dep_rademacher_bound}

In this section, we use our framework to prove uniform generalization bounds in terms of a data-dependent Rademacher complexity, which is a term of the form $\rad_S(\wcal_S)$, where the hypothesis set $\wcal_S$ depends on the data $S$. We remind the reader that $G_S(\wcal)$ was defined in \Cref{eq:worst_gen_error_intro}. In this section, we apply \Cref{theorem:pac-bayes-set} to the following function:
\begin{align}
    \label{eq:phi_function_for_rademacher}
    ~\Phi_\lambda(\wcal, S) = \lambda G_S(\wcal) - 2\lambda \rad_S(\wcal), \qquad \lambda>0.
\end{align}

This leads to the following theorem, which is a PAC-Bayesian data-dependent uniform generalization bound involving a data-dependent\footnote{We use the expression ``data-dependent Rademacher complexity'' to highlight the fact that the hypothesis set $\wcal$ on which it is applied is data-dependent.} Rademacher complexity term.

\begin{restatable}[Data-dependent Rademacher complexity bound]{theorem}{theoremRADdatadep}
   \label{theorem:data_dep_rademacher_set}
   Suppose that \Cref{ass:bounded_loss,ass:supremum_measurability} hold. Then, for any $\lambda > 0$ we have
    \begin{align*}
    &\mathds{P}_S \left( \mathds{E}_{\wcal \sim \rho_S} \big[ G_S(\wcal) \big] \leq \mathds{E}_{\wcal \sim \rho_S} \big[  2\rad_S(\wcal) \big] + \frac{\klb{\rho_S}{\pi} + \log(1/\zeta)}{\lambda} + \lambda\frac{9B^2}{8n} \right) \geq 1-\zeta,\\
    &\mathds{P}_{S,\wcal\sim\rho_S} \left( G_S(\wcal) \leq   2\rad_S(\wcal)+ \frac{\log \frac{d \rho_S}{d\pi}(\wcal) + \log(1/\zeta)}{\lambda} + \lambda \frac{9B^2}{8n} \right) \geq 1-\zeta.
    \end{align*}
\end{restatable}

The IT terms appearing in \Cref{theorem:data_dep_rademacher_set} are exactly the same as those in \Cref{theorem:pac-bayes-set} and \Cref{theorem:mgf_rademacher}, which illustrates the universality of our approach. 
Therefore, comparing the bounds obtained within our framework can be achieved regardless of these IT terms.

The data dependence of the Rademacher complexity terms $\rad_S(\wcal)$, appearing in \Cref{theorem:data_dep_rademacher_set}, might not be obvious at first sight. This data dependence comes from the fact that $\wcal$ is here drawn from the posterior $\rho_S$. If we look back at \Cref{example:SGD}, such a set $\wcal$ would be a data-dependent trajectory of SGD obtained while training on the dataset $S$. 

\citet{foster2019hypothesis} also considered data-dependent Rademacher complexity terms. However, their results rely on so-called hypothesis set stability assumptions. We point out the fact that no such assumption is needed to derive \Cref{theorem:data_dep_rademacher_set}. Our result can also be compared with \citep{sachs2023generalization}, in which the authors derived generalization bounds in terms of ``algorithmic-dependent Rademacher complexity'', \ie, a notion of complexity depending on the algorithm, but not directly on the dataset $S$ used in training. This is to be opposed to the term $\rad_S(\wcal)$ in \Cref{theorem:data_dep_rademacher_set}, which depends explicitly on the dataset $S\in\zcal^n$.

Building on classical arguments in learning theory, this result opens the door to introducing other data-dependent terms in the bounds by further bounding the Rademacher complexity term. In particular, we may introduce VC dimension terms and/or perform covering arguments and chaining techniques. In \Cref{sec:fractal_bounds}, we will focus on fractal-dimension-based generalization bounds, where the proposed methods are particularly useful.

Let us make a remark about the role of the parameter $\lambda > 0$ appearing in the above theorem.

\begin{remark}
    \label{rq:lambda-optimization-london-technique}
    \Cref{theorem:data_dep_rademacher_set}, as well as our other main results (\Cref{theorem:mgf_rademacher,theorem:data_dep_lower_bound}), along with their applications in the sequel, are valid for any $\lambda > 0$. In many contexts, this parameter $\lambda$ can be optimized to simplify the expression of our generalization bounds. In particular, using a proof technique similar to \citep[Lemma $9$]{london2017pac}, one can see that \Cref{theorem:rademacher-bound} implies that with probability at least $1 - \zeta$ over $S\sim\datadist$, we have:
     \begin{align}
            \label{eq:informal-optimized-london}
             &\mathds{E}_{\wcal \sim \rho_S} \big[ G_S(\wcal) \big] \leq \mathds{E}_{\wcal \sim \rho_S} \big[  2\rad_S(\wcal) \big] + 6B \sqrt{\frac{\klb{\rho_S}{\pi} + \log(2/\zeta)}{2n}}.
        \end{align}
    A similar result could be shown for the disintegrated bound. Any bound (\eg, \Cref{sec:fractal_bounds,sec:langevin_girsanov,sec:bounds_for_sgld}) with a parameter $\lambda$ playing a similar role can be further optimized in this way.
    
    However, one can notice that in the case of a data-independent hypothesis set $\wcal$, where we have $\klb{\rho_S}{\pi} = 0$, \Cref{eq:informal-optimized-london} recovers worse absolute constants than the previously known generalization bounds, as given by \Cref{theorem:rademacher-bound}. For this reason, we chose to present all our results in the more general form, with a free parameter $\lambda>0$. Depending on the application, the reader may use any of the two formulations.
\end{remark}

A classical question in PAC-Bayesian analysis concerns the minimization of the bound with respect to the posterior distribution \citep[Section $2.1.2$]{alquier2024user} and leads to the consideration of the so-called Gibbs posterior.
Such an analysis extends to our framework.
Indeed, by \Cref{rq:uniform-over-posteriors} and \Cref{theorem:data_dep_rademacher_set}, we have that, with probability at least $1 - \zeta$ over $S\sim\datadist$, for every posterior distribution $\rho$, we have:
\begin{align*}
    \Eof[\wcal\sim\rho]{\sup_{w\in\wcal} \risk(w)} \leq \Eof[\wcal\sim\rho]{\sup_{w\in\wcal} \er(w) + 2\rad_S(\wcal)} + \frac{\klb{\rho}{\pi} + \log(1/\zeta)}{\lambda} + \lambda\frac{9B^2}{8n}.
\end{align*}
By a classical application of Donsker-Varadhan's variational formula, we see that the probability distribution $\rho$ minimizing the right-hand-side of this expression is the following ``Gibbs-Rademacher posterior'', defined for any $\lambda >0$ by:
\begin{align}
    \label{eq:gibbs-rademacher-posterior}
    d \rho_S^{(\lambda)}(\wcal) \propto \exp \left\{ -\lambda\sup_{w\in\wcal} \er(w) - 2\lambda\rad_S(\wcal)  \right\} d \pi(\wcal),
\end{align}
where the symbol $\propto$ indicates that the appropriate normalization factor has been omitted.
\Cref{eq:gibbs-rademacher-posterior} provides a family of posteriors that are optimal in the sense that they minimize the worst population risk $\risk(w)$ over $\wcal$, given a prior distribution $\pi$. In the case of singleton random sets (\Cref{example:singletons}), they generalize the Gibbs distributions classically encountered in PAC-Bayesian analysis \citep{alquier2024user}, which are of the form $d \pcal_S^{(\lambda)} \propto e^{-\lambda \er (w)} d \qcal(w)$. The form of $\rho_S^{(\lambda)}$ suggests that the best learning algorithm (given a prior $\pi$) samples data-dependent hypothesis sets with the lowest Rademacher complexities.

The proof of \Cref{theorem:data_dep_rademacher_set}, presented in \Cref{sec:proofs-main-result-data-dependent}, highlights that \Cref{theorem:data_dep_rademacher_set} is a consequence of the symmetrization lemma, \ie, \Cref{lemma:symmetrization}. 
This suggests the following general form, which could be used to derive a wide variety of generalization bounds, involving other types of functionals than the ones we use here but still using the same IT terms.

\begin{restatable}[A general form of set-dependent bounds]{theorem}{theoremGeneral}
   \label{theorem:data_dep_general_form}
   Let $\Phi: E\times \zcal^n \longrightarrow \mathds{R}$ be a measurable function. We assume that for every $\wcal \in E$, we have $\mathds{E}_S[\Phi(\wcal, S)] \leq 0$ and $\Phi(\wcal,S) - \Eof[S]{\Phi(\wcal,S)}$ is $\sigma^2$-subgaussian.\footnote{$X$ is said to be subgaussian if, for every $\lambda > 0$, we have $\Eof{e^{\lambda X}} \leq e^{\frac{\lambda^2 \sigma^2}{2}}$, see \citep{vershynin2018high}.}
   Then, there exists an absolute constant $C$ such that, for any $\lambda > 0$:
    \begin{align*}
    &\mathds{P}_S \left( \mathds{E}_{\wcal \sim \rho_S} \big[ \Phi(\wcal, S) \big] \leq  \frac{\klb{\rho_S}{\pi} + \log(1/\zeta)}{\lambda} + \lambda\frac{C\sigma}{n} \right) \geq 1-\zeta,\\
    &\mathds{P}_{S,\wcal\sim\rho_S} \left(  \Phi(\wcal, S) \leq  \frac{\log \frac{d \rho_S}{d\pi}(\wcal) + \log(1/\zeta)}{\lambda} + \lambda \frac{C\sigma}{n} \right) \geq 1-\zeta.
    \end{align*}
\end{restatable}

\begin{remark}
    \label{rq:bernstein-form}
    In the proof of \Cref{theorem:data_dep_rademacher_set}, the subgaussian condition is, in this case, a consequence of McDiarmid's inequality \citep{mcdiarmid1998concentration}.
    This opens the door for further generalizations of our framework by considering Bernstein forms of this inequality \citep[Theorem $3.8$]{mcdiarmid1998concentration}; see also the note of \citet{ying2004mcdiarmid}.
    In our case, this would allow us to remove the bounded loss assumption, \Cref{ass:bounded_loss}, at the cost of introducing intricate variance terms in the bound. We leave this as a direction for future work.
\end{remark}

\subsection{Data-dependent generalization lower bounds}
\label{sec:lower_bound}

\Cref{theorem:data_dep_general_form} provides us with a general recipe to prove PAC-Bayesian bounds. To illustrate this ability, we now derive a \emph{data-dependent generalization lower bound}, based on Rademacher complexity. To the best of our knowledge, such data-dependent uniform lower bounds have not been derived before.

In the same way that \Cref{theorem:data_dep_rademacher_set} was deduced from the symmetrization lemma (which fulfills the first condition of \Cref{theorem:data_dep_general_form}), the main result of this subsection is based on the so-called desymmetrization inequality (recalled in \Cref{sec:technical_lemmas}).

This leads us to the consideration of the following set function:
\begin{equation}
    \label{eq:lower_bound_set_function}
    \Phi_{\text{lower bound}}(\wcal, S) := \frac{1}{2} \rad_S(\wcal)  - \frac{B}{2\sqrt{n}} - \sup_{w \in \wcal} |\er(w) - \mathcal{R}(w)| ,
\end{equation}

from which we deduce a lower bound on $\gabs := \sup_{w \in \wcal} |\er(w) - \mathcal{R}(w)|$, which is presented in the next theorem.

\begin{restatable}[Data-dependent lower bound]{theorem}{lowerBound}
    \label{theorem:data_dep_lower_bound}
   Suppose that \Cref{ass:bounded_loss,ass:supremum_measurability} hold. There exists an absolute constant $C>0$, such that, for any $\lambda > 0$ we have, with probability at least $1 - \zeta$ over $S\sim \datadist$:
    \begin{align*}
     \mathds{E}_{\wcal \sim \rho_S} \big[\gabs \big] \geq\frac{1}{2}\mathds{E}_{\wcal \sim \rho_S} \big[ \rad_S(\wcal)\big] -  \frac{B}{2\sqrt{n}} - \frac{ \klb{\rho_S}{\pi} + \log(1/\zeta)}{\lambda}  -  \frac{CB^2 \lambda}{n} .
     \end{align*}
     Moreover, we have the disintegrated lower bound:
     \begin{align*}
    \mathds{P}_S \mathds{P}_{\wcal \sim \rho_S} \left( \gabs\geq\frac{1}{2} \rad_S(\wcal) -  \frac{B}{2\sqrt{n}} - \frac{\frac{d \rho_S}{d \pi}(\wcal) + \log(1/\zeta)}{\lambda}  -  \frac{CB^2 \lambda}{n} \right)\geq 1 - \zeta.
    \end{align*}
\end{restatable}

\begin{proof}
    This is a direct consequence of \Cref{theorem:data_dep_general_form}, where $\Phi$ is chosen to be $\Phi_{\text{lower bound}}(\wcal, S)$, from \Cref{eq:lower_bound_set_function}. The negative expectation follows from \Cref{prop:desymmetrization_inequality} and the subgaussian property from McDiarmid's inequality, as in the proof of \Cref{theorem:data_dep_rademacher_set}.
\end{proof}
\section{Fractal-Dimension-Based Data-Dependent Generalization Bounds}
\label{sec:fractal_bounds}

In this section, we detail the application of our framework to fractal-dimension-based generalization bounds, as mentioned in the introduction.

We start by defining covering numbers and fractal dimensions, in \Cref{sec:covering_bounds}, before proving generalization bounds involving the two main kinds of fractal dimensions found in the literature, namely the ``data-dependent'' dimension \citep{dupuis2023generalization} and the ``Euclidean-based'' dimension \citep{simsekli2020hausdorff,hodgkinson2022generalization}. Our bounds are deduced from standard covering arguments applied to the data-dependent Rademacher complexity term appearing in \Cref{theorem:data_dep_rademacher_set}. These data-dependent covering bounds, which extend classical covering arguments for data-independent hypothesis sets \citep{shalevschwartz2014understanding}, are presented in \Cref{sec:covering-bounds-proofs}.
Finally, \Cref{sec:IT_terms_comparison} is devoted to the comparison of our information-theoretic terms and the ones existing in the literature.

\subsection{Covering numbers and fractal dimensions}
\label{sec:covering_bounds}

To be able to encapsulate all the existing fractal-dimension-based generalization bounds, we define coverings in full generality in pseudometric spaces. We say that $(X,\vartheta)$ is a pseudometric space if $\vartheta : X \times X \longrightarrow \R_+$ is symmetric, satisfies the triangle inequality, and vanishes on the diagonal (\ie $\vartheta(x,x) = 0$). 
Given a compact pseudometric space $(X,\vartheta)$, we define, for any $\delta > 0$, $N_\delta^\vartheta(X)$ to be the set of centers of a minimal covering of $X$ by closed $\delta$-balls. The \emph{covering number} is the cardinality of this set, denoted $\vert N_\delta^\vartheta(X) \vert$. In the case of the Euclidean distance in $\Rd$, the metric will be omitted in the notations.

It has been shown in \citep{dupuis2023generalization} that, under mild assumptions on the measurability of the learning algorithm, we can construct measurable coverings, which will be assumed to be the case in all the following. This is formalized by the following assumption discussed in greater detail in \Cref{sec:covering_numbers_measurability}.

\begin{assumption}[Measurable covering numbers]
    \label{ass:measurable_covering_numbers}
    The covering numbers appearing in this section are all measurable with respect to $\mathcal{F}^{\otimes n} \otimes \mathfrak{E}$.
\end{assumption}

Our goal is to relate the generalization error to the \emph{upper box-counting dimension} of the random hypothesis set $\wcal \sim \rho_S$.
The upper box-counting dimension, or upper Minkowski dimension \citep{falconer2014fractal}, is defined for a compact pseudometric space $(X,\vartheta)$, by
\begin{align}
    \label{eq:upper_box_counting}
     \upperbox^\vartheta (X) := \limsup_{\delta \to 0}\frac{\log(\vert N_\delta^\vartheta(X) \vert)}{\log(1/\delta)}.
\end{align}

The upper box-counting dimension is a central tool in fractal geometry. Intuitively, it may be seen as a measure of the complexity of a set, which extends the usual notion of dimension for vector spaces or Riemannian manifolds, see \citep{falconer2014fractal, mattila1999geometry}.

Other studies have used other notions of fractal dimension in learning theory, in particular the \emph{Hausdorff dimension} \citep{simsekli2020hausdorff}. This is made possible by leveraging technical geometrical assumptions, such as Ahlfors regularity \citep{mackay2010conformal}, which guarantee the Minkowski and Hausdorff dimensions are equal.

Thanks to our PAC-Bayesian framework for random sets, we leverage the same proof technique to obtain generalization bounds with two variations of the upper box-counting dimension of the space $\wcal$. By varying the pseudometric $\vartheta$ in \Cref{eq:upper_box_counting}, we define the following fractal dimensions:
\begin{itemize}
    \item The \textbf{data-dependent fractal dimension} is based on the data-dependent pseudometric considered by \citet{dupuis2023generalization}, and defined by:
        \begin{equation}
        \label{eq:data_dep_pseudo_metric}
        \vartheta_S (w,w') := \frac{1}{n} \sum_{i=1}^n |\ell(w,z_i) - \ell(w',z_i)|,
        \end{equation}
        where $S\in\zcal^n$ is a dataset.
    This dimension will be denoted $\upperbox^{\vartheta_S}(\wcal)$, accordingly with \Cref{eq:upper_box_counting}. Note that \Cref{eq:data_dep_pseudo_metric} defines a random pseudometric (\ie, not a metric) because the map $w \longmapsto \ell(w,\cdot)$ might not be injective.
    \item The \textbf{Euclidean-based fractal dimension} is based on the Euclidean distance in $\Rd$, it will be simply denoted $\upperbox(\wcal)$. This is the intrinsic dimension studied in \citep{simsekli2020hausdorff,birdal2021intrinsic,hodgkinson2022generalization}.
\end{itemize}

Let $\vartheta$ denote either $\vartheta_S$, for some $S\in\zcal^n$, or the Euclidean distance. It follows from classical arguments in learning theory \citep{shalevschwartz2014understanding} that
\begin{align}
    \label{eq:convering_bounds_informal}
    \rad_S(\wcal) \lesssim \inf_{\delta>0} \left\{\delta +  \sqrt{\frac{ \log(\vert N_\delta^{\vartheta}(\wcal) \vert)}{n}} \right\},
\end{align}
where absolute constants have been omitted. By combining \Cref{eq:convering_bounds_informal} with \Cref{theorem:data_dep_rademacher_set}, we can obtain data-dependent covering bounds (see \Cref{sec:covering-bounds-proofs}). These results lay the foundation for our fractal-based generalization bounds, which we will now present.
Note that classical covering bounds deduced from Rademacher complexity naturally have a data-dependent flavor, due to the presence of the pseudometric $\vartheta_S$ \citep{shalevschwartz2014understanding,dupuis2023generalization}. The novelty of our data-dependent covering bounds is to apply to data-dependent hypothesis sets, which leads to the introduction of additional information-theoretic terms.

\subsection{Bounds with data-dependent fractal dimensions}
\label{sec:data_dep_fractal_bounds}

In this section, we present our generalization bound based on the data-dependent fractal dimension, as defined in \Cref{sec:covering_bounds}.
Inspired by \citet[Theorem $3.4$]{dupuis2023generalization}, our approach for obtaining these bounds is through using covering arguments and the link between covering numbers and the fractal dimension, \ie, \Cref{eq:upper_box_counting}. The drawback of this approach is that it relies on some terms whose dependence on $n$ is unknown a priori. This issue is identified but not resolved in \citep{dupuis2023generalization}, leading to the introduction of constants with unknown dependence on $n$.

Despite these difficulties, our framework allows us to design a natural assumption to prove better generalization bounds. Indeed, the previous issue is due to the possible lack of uniformity in $n$ of the limit in \Cref{eq:upper_box_counting} defining the upper box-counting dimension. In order to prove our data-dependent fractal dimension bound, we find that it is enough to assume that the convergence in probability (under $S\sim\datadist$ and $\wcal \sim \rho_S$) of \Cref{eq:upper_box_counting} is uniform in $n$. To avoid harming the readability of this section, this assumption is presented in detail in \Cref{sec:data-dep-fractal-bound-proof}.

Based on \Cref{ass:uniform_cv_data_dep}, we can now state our generalization bounds in terms of the data-dependent fractal dimension induced by the data-dependent pseudometric of \Cref{eq:data_dep_pseudo_metric}.

\begin{restatable}{theorem}{CorollaryDatadepDim}
    \label{thm:rad_data_dep_dim}
    Suppose that Assumptions \ref{ass:bounded_loss}, \ref{ass:supremum_measurability}, \ref{ass:measurable_covering_numbers} and \ref{ass:uniform_cv_data_dep} hold. Then there exists a constant $C > 0$ and for any $\lambda,\epsilon,\gamma  > 0$, there exists $n_{\gamma,\epsilon} \in \mathds{N}^\star$, such that, for $n \geq n_{\gamma, \epsilon}$, we have, with probability at least $1 - \zeta - \gamma$ under $S \sim \datadist$ and $\wcal \sim \rho_S$, for any $\lambda > 0$:
    \begin{align*}
      G_S(\wcal) \leq  \frac{2}{n} + 2 B\sqrt{\frac{2 \left( \upperbox^{\vartheta_S}(\wcal) + \epsilon \right)  \log(n)}{n}}+ \frac{\log \frac{d \rho_S}{d\pi}(\wcal) + 
    \log(1/\zeta)}{\lambda} + C\lambda \frac{B^2}{n}.
    \end{align*}
\end{restatable}
\looseness=-1
The proof of this theorem is referred to \Cref{sec:data-dep-fractal-bound-proof}, it is significantly simpler than the proofs of \citet{dupuis2023generalization}, hence underlining the effectiveness of our framework.
It extends the results of \citet{dupuis2023generalization} in our PAC-Bayesian setting. Thanks to our framework and the additional \Cref{ass:uniform_cv_data_dep}, our bound has a more explicit dependence on $n$ and $\upperbox^{\vartheta_S}(\wcal)$, compared to \citep{dupuis2023generalization}. The only non-explicit dependence is encapsulated in the information-theoretic term (the Radon-Nikodym derivative), as it is often the case for PAC-Bayesian bounds. 
As we will show in \Cref{sec:IT_terms_comparison}, if the prior distribution $\pi$ is well chosen, the IT term $\log (d\rho_S / d\pi)$, that appears in \Cref{thm:rad_data_dep_dim}, is smaller than the total mutual information term appearing in \citep[Theorem $3.8$]{dupuis2023generalization}.
Moreover, \Cref{ass:uniform_cv_data_dep} is much simpler than the so-called ``geometric stability'' assumption used in \citep[Definition $3.6$]{dupuis2023generalization}.
More importantly, \Cref{thm:rad_data_dep_dim} has a better rate of convergence in $n$ than \citep[Theorem $3.8$]{dupuis2023generalization}, while featuring the same fractal dimension and IT term. More precisely, our bound vanishes as $n^{-1/2}$, while it is of order $n^{-2\alpha/3}$ in \citep[Theorem $3.8$]{dupuis2023generalization}, with $\alpha < 3/2$ a parameter appearing in an intricate geometric stability assumption \citep[Definition $3.6$]{dupuis2023generalization}. 

\subsection{Bounds with Euclidean-based fractal dimensions}
\label{sec:euclidean_fractal_bounds}

In this subsection, we adapt the results of the previous subsection to the Euclidean-based fractal dimension, according to the terminology of \Cref{sec:covering_bounds}. As usual in this literature \citep{simsekli2020hausdorff,hodgkinson2022generalization,birdal2021intrinsic,camuto2021fractal}, we assume, in this subsection, that the loss $\ell(w,z)$ is $L$-Lipschitz continuous in $w$.

While this is a strong additional assumption compared to the setting of \Cref{thm:rad_data_dep_dim}, it comes with the benefit of allowing for assumptions weaker than \Cref{ass:uniform_cv_data_dep}. Note, however, that it would still be possible to adapt \Cref{ass:uniform_cv_data_dep} to the present setting.

Such weaker assumptions are made possible by the fact that the covering numbers $|N_\delta(\wcal)|$, based on the Euclidean distance, have a weaker dependency on the number $n$ of data points, compared to their ``data-dependent'' counterparts $|N^{\vartheta_S}_\delta(\wcal)|$. More precisely, $|N_\delta(\wcal)|$ depends on $n$ only through the posterior distribution $\rho_S$. Thus, inspired by the law of large numbers, the issue created by the dependence in $n$ can be addressed by assuming convergence of the posterior distribution to a data-independent distribution, in some sense. This is made precise by \Cref{ass:total_vaiation_cv}, which assumes convergence in total variation of $\rho_S$ to a data-independent distribution when $n\to\infty$, and is formally described in \Cref{sec:euclidean_fractal-bound-proof}.

The next theorem is a consequence of \Cref{theorem:data_dep_rademacher_set} and \Cref{cor:euclidean_covering_numbers}.
\begin{restatable}{theorem}{thmDimBoundTotalVar}
    \label{thm:euclidean_dim_with_tv_assumption}
    Suppose that Assumptions \ref{ass:bounded_loss}, \ref{ass:supremum_measurability}, \ref{ass:measurable_covering_numbers} and \ref{ass:total_vaiation_cv} hold. We further assume that the loss $\ell(w,z)$ is $L$-Lipschitz continuous in $w$ and that $\wcal$ is $\pi$-almost surely bounded.
    Then, for any $\epsilon,\gamma > 0$, there exists a constant $C > 0$ and $n_{\gamma,\epsilon} \in \mathds{N}^\star$ such that, for any $\lambda > 0$ and $n \geq n_{\gamma,\epsilon}$, with probability at least $1 - \zeta - 3\gamma$ over $S \sim \datadist$ and $\wcal \sim \rho_S$, we have:
    \begin{align*}
      G_S(\wcal) \leq     \frac{2 L }{n} + 4 B \sqrt{\frac{ \left( \upperbox(\wcal) + \epsilon \right) \log(n) }{2n}} + \frac{\log \frac{d \rho_S}{d\pi}(\wcal) + \log(1/\zeta)}{\lambda} + C\lambda \frac{B^2}{n}. 
    \end{align*}
\end{restatable}

This theorem, whose proof is presented in \Cref{sec:euclidean_fractal-bound-proof}, is the extension of \citep[Theorem $2$]{simsekli2020hausdorff} in our setting. Compared to \citep[Theorem $2$]{simsekli2020hausdorff}, our result does not require any complicated geometric information on the hypothesis set $\wcal$ and has a simpler IT term. 

Moreover, it should be noted that \Cref{thm:rad_data_dep_dim,thm:euclidean_dim_with_tv_assumption} have been derived by the same proof technique, as consequences of our generalization bound with data-dependent Rademacher complexity, \ie, \Cref{theorem:data_dep_rademacher_set}. This is an important improvement over the prior art in fractal-based generalization, where all the existing results \citep{simsekli2020hausdorff,dupuis2023generalization} rely on notably different arguments.

Unlike the results of \Cref{sec:data_dep_fractal_bounds}, the above theorem relies on a Lipschitz continuity assumption on the loss, which may be unrealistic for modern deep neural networks. Therefore, \Cref{thm:rad_data_dep_dim} may look tighter than \Cref{thm:euclidean_dim_with_tv_assumption}. However, it was obtained under stronger assumptions, \eg, uniformity in $n$ of the limit of \Cref{eq:upper_box_counting}. Hence, the fact that the bound may hold under weaker assumptions is the main advantage of using the Euclidean-based fractal dimension in \Cref{thm:euclidean_dim_with_tv_assumption}.

\subsection{Comparison of the information-theoretic terms and prior optimization}
\label{sec:IT_terms_comparison}

This subsection is devoted to the comparison between the IT terms appearing in our work and those appearing in other studies considering fractal dimensions \citep{simsekli2020hausdorff, hodgkinson2022generalization, dupuis2023generalization}.

In most works proving data-dependent uniform generalization bounds, especially in the fractal-based generalization literature, the data-dependent hypothesis set is represented by a random set $\wcal_S$, depending on the data $S$ and also on some external randomness, induced by the learning algorithm, here omitted. This is the setting in which the bounds represented by \Cref{eq:fractal_informal} were proven. As already discussed, these bounds typically introduce some kind of mutual information  (MI) between $\wcal_S$ and $S$, to deal with the data-dependence of $\wcal_S$ \citep{simsekli2020hausdorff, hodgkinson2022generalization, camuto2021fractal, dupuis2023generalization}. We focus in particular on the so-called total mutual information term\footnote{Different studies use different total mutual information terms, but as noted by \citet{dupuis2023generalization} and \citet{hodgkinson2022generalization}, $I_\infty(\wcal_S,S)$ is the simplest and most intuitive one that has been introduced, it is therefore pertinent for our comparison.} $I_\infty(\wcal_S,S)$, for which we give the following definition.

\begin{definition}
    \label{def:total-mi}
    Let $X$ and $Y$ be two random variables and denote their distributions by $\prob_X$ and $\prob_Y$ respectively and their joint distribution by $\prob_{X,Y}$, then we define the total mutual information by:
    \begin{align*}
    \mutualinfty(X,Y) := \log \left(\sup_{A \in \mathcal{T}} \frac{\prob_{(X,Y)}(A)}{\prob_X \otimes \prob_Y(A)}\right).
    \end{align*}
\end{definition}
This quantity has already been used in learning theory, see in particular \citep{dwork2015generalization,hodgkinson2022generalization}.

In this subsection, we translate the settings mentioned above in our framework by letting the posterior $\rho_S$ be the conditional distribution of $\wcal_S$, given $S$.

The aforementioned works did not consider a PAC-Bayesian setting. Therefore, to provide a meaningful comparison with our framework, we may choose a prior that is, in a sense, optimal, and compute the corresponding IT terms. Classically, following \citet{catoni2007pac} and \citet{alquier2024user} we optimize the prior with respect to the family of posterior distributions, by using the following ``optimized prior'', which corresponds to the marginal distribution of $\wcal_S\sim\rho_S$ under $S\sim\datadist$.
\begin{align}
    \label{eq:optimized_prior}
    \forall A \in \mathfrak{E}, ~\pi(A) := \mathds{E}_{S \sim \datadist}[\rho_S(A)].
\end{align}

Under mild assumptions, we have $\rho_S \ll \pi$, $\datadist$-almost surely. For instance, it is the case if $\datadist$ is a strictly positive Borel probability measure and the maps $S\mapsto \rho_S(A)$ are continuous. In the rest of this section, we will assume that this absolute continuity condition is fulfilled.

Moreover, it is known \citep[Section $6.5.2$]{alquier2024user} that $\mathds{E}_S \left[ \klb{\rho_S}{\pi} \right] = I_1(\wcal_S,S)$, where $I_1$ is the mutual information defined in \Cref{sec:appendix_technical_background}.

The following lemma, proven in \Cref{sec:smalest-it}, shows that the generalization bounds implied by this optimized prior yield tighter IT terms than existing bounds.

\begin{restatable}{lemma}{smallestITTerm}
    \label{lemma:smallest-it}
    With the same notations, we have, for $\datadist$-almost all $S$ and $\rho_S$-almost all $\wcal$:
    $$
    \log  \frac{d \rho_S}{d\pi}(\wcal)  \leq I_\infty(\wcal_S,S).
    $$
\end{restatable}

This shows that our framework provides tighter generalization guarantees. In addition, it simplifies the previous fractal bounds and derives them from one proof technique.

\section{Application to Langevin Dynamics}
\label{sec:langevin_girsanov}

In this section, we present the application of our PAC-Bayesian framework to the derivation of uniform generalization bounds for continuous Langevin dynamics (CLD). More precisely, let us consider a measurable space $(\Omega, \mathcal{T})$ and the following stochastic differential equation (SDE), which we call the \emph{empirical} dynamics (restatement of \Cref{eq:langevin_equation}):
\begin{align}
    \label{eq:langevin_empiric}
    dW_t = -\nabla \er(W_t) dt + \sigma dB_t, \quad W_0 = w_0, \text{ with } \sigma := \sqrt{2\beta^{-1}}.
\end{align}
where $(B_t)_{t\geq 0}$ is a standard Brownian motion in $\mathds{R}^d$ and $w_0 \in \Rd$ is the initialization\footnote{Note that we could also initialize the dynamics randomly and independently from the other random variables, without changing our results.}. 
The main feature of our method is to analytically express the KL divergence appearing in \Cref{theorem:pac-bayes-set}. This is done by exploiting Girsanov's theorem \citep[Section $8.6$]{oksendal2003stochastic}, as well as semi-martingale properties of \Cref{eq:langevin_equation}, similar to \citep{dalalyan2017theoretical,raginsky2017non}. In \Cref{sec:langevin_setting_girsanov}, we will specify our setting and express the IT terms in a general form. Then, in \Cref{sec:gen_bounds_cld}, we will derive the corresponding uniform bounds for CLD.

\subsection{The setting}
\label{sec:langevin_setting_girsanov}

For \Cref{eq:langevin_empiric} to have a unique continuous and square-integrable strong solution, we make the following classical assumption. 

\begin{assumption}
    \label{ass:smoothness}
    The loss $\ell$ is differentiable and $M$-smooth in $w$, which means:
    $$
    \forall w,w' \in \R^d, ~\forall z \in \zcal,~\normof{\nabla\ell(w,z) - \nabla\ell(w',z)} \leq M \normof{w - w'}.
     $$  
\end{assumption}

Let us consider a fixed time horizon $T>0$. We introduce the \emph{random trajectory}, \ie, the set of points encountered by the process, defined by
\begin{align}
\label{eq:cld-set-definition}
\wcal(\omega) :=\{ W_t(\omega),~0 \leq t \leq T \}.
\end{align}

In \Cref{sec:pac_bayes_for_random_sets}, we defined both the prior and posteriors directly on the set $E$, containing \emph{subsets} of $\mathds{R}^d$. While this is effective in many applications, in the case of SDE trajectories, it is beneficial to adapt our formulation to take into account the underlying probability space $\Omega$. 
More precisely, we define the prior $\pi$ and the posteriors $\rho_S$ directly on the underlying space $\Omega$, satisfying the same Markov kernel properties, as previously defined. We additionally require that all these distributions induce a complete probability space structure on $\Omega$ and that the measures are equivalent, \ie, $\rho_S \ll \pi$ and $\pi \ll \rho_S$. $W$ is seen as a stochastic process defined on $\Omega$. This provides us with a rigorous measure-theoretic setup, where all relevant quantities (\eg, $G_S(\wcal)$, $\rad_S(\wcal)$) are measurable.

Thanks to these notations, we can directly restate the PAC-Bayesian bounds of \Cref{theorem:pac-bayes-set}; for instance, we can write:
\begin{align*}
    \mathds{P}_{S}\left(\Eof[\omega \sim\rho_S]{\Phi(\wcal(\omega),S)}\leq  \klb{\rho_S}{\pi} + \log(1/\zeta) + \log  \mathds{E}_S\mathds{E}_{\omega\sim \pi} \big[ e^{\Phi(\wcal(\omega),S)} \big] \right) \geq 1 - \zeta.
\end{align*}
To ease the notation, we will omit $\omega$ and denote $\wcal\sim\rho_S$ for $\wcal \sim \wcal_\#\rho_S$.

Note that considering, as prior and posterior, the pushforward measures, $\wcal_\# \pi$, and $\wcal_\# \rho_s$, respectively, would take us back to the exact setting of \Cref{sec:pac_bayes_for_random_sets}. 
The data processing inequality, \ie, $\klb{\wcal_\# \rho_s}{\wcal_\# \pi} \leq \klb{\rho_S}{\pi}$, ensures that both setups are linked.

For technical reasons, we make an additional Lipschitz continuity assumption on $\ell$, which is similar to prior work \citep{aristoff2012estimating, mou2018generalization, li2020generalization, farghly2021time}. This ensures that a technical condition, Novikov's condition, holds, see \Cref{sec:proofs-langevin_girsanov} for more details. 

\begin{assumption}
    \label{ass:lipschitz_loss}
    The loss $\ell$ is $L$-Lipschitz continuous in $w$, uniformly with respect to $z$.
\end{assumption}

We will proceed in two steps: first, in \Cref{sec:gen_bounds_cld}, we will provide two expressions of the KL divergence term, depending on the choice of the prior distribution. This highlights the fact that the IT terms appearing in our main theorems can be expressed in particular cases. In the second step, we conclude the derivation of uniform generalization bounds by deriving a bound on the Rademacher complexity of Langevin dynamics, in \Cref{sec:rademacher_langevin}. The main result of this section then follows from \Cref{theorem:data_dep_rademacher_set}.

\subsection{Expression of the KL divergence}
\label{sec:gen_bounds_cld}

To get an expression of the KL divergence term, appearing in our main theorems, we must make a suitable choice of prior distribution $\pi$. To leverage classical tools from stochastic calculus, namely Girsanov's theorem (see \Cref{sec:proofs-langevin_girsanov}), we define $\pi$ as the path measure of the following \emph{data-independent} SDE:
\begin{align*}
        dW_t = -\nabla F(W_t) dt + \sigma dB_t, \quad W_0 = w_0.
\end{align*}
We consider two types of prior, given the choice of function $F$:
\begin{enumerate}
    \item The \textbf{Brownian prior} corresponds to $F=0$.
    \item The \textbf{expected dynamics} prior corresponds to $F = \risk$ (\ie, the population risk). It might be used to tighten the bounds under certain conditions.
\end{enumerate}

We provide, in \Cref{thm:girsanov_expression}, an expression of the KL divergence $\klb{\rho_S}{\pi}$ that is induced by a general function $F$.
We now detail the results that we obtained for both choices of prior distributions, which are a direct consequence of \Cref{thm:girsanov_expression}.

\subsubsection{Brownian prior}
\label{sec:brownian_prior}

In this subsection, we set $F=0$, so that, under the prior distribution $\pi$, we have $W_t = w_0 + \sigma B_t$, \ie $W$ is a (scaled and translated) Brownian motion. As a consequence of \Cref{thm:girsanov_expression}, we have the following expression of the KL divergence:
 \begin{equation}       \label{eq:girsanov_brownian_measure}
        \klb{\rho_S}{\pi} = \frac{1}{2\sigma^2} \int_0^T \mathds{E}_{\rho_S} \big[\Vert \nabla \er(W_t) \Vert^2 \big] dt,
\end{equation}
from which we deduce the following corollary.
\begin{corollary}
    \label{cor:langevin_brownian_bound}
    Under Assumptions \ref{ass:bounded_loss}, \ref{ass:smoothness} and \ref{ass:lipschitz_loss}, with probability at least $1-\zeta$ over $S \sim \datadist$, we have, for all $\lambda>0$:
    \begin{align*}
        \Eof[\rho_S]{G_S(\wcal)} \leq 2\Eof[\rho_S]{\rad_S(\wcal)} + \frac{1}{2\lambda\sigma^2} \int_0^T \mathds{E}_{\rho_S} \big[\Vert \nabla \er(W_t) \Vert^2 \big] dt + \frac{\log(1/\zeta)}{\lambda} + \lambda \frac{9 B^2}{8n}.
    \end{align*}
\end{corollary}

Note that our bound does not require any $\ell^2$-regularization to hold, as in \citep{mou2018generalization, li2020generalization, farghly2021time}. The Rademacher complexity term will be bounded in \Cref{sec:rademacher_langevin}. We will further discuss the implications of this result after having presented a bound on the Rademacher complexity term, in \Cref{sec:rademacher_langevin}.

\subsubsection{Expected dynamics prior}
\label{sec:expected_dynamics_prior}

We now turn to the case of the expected dynamics prior where, under $\pi$, $W$ follows the following SDE:
\begin{align}
    \label{eq:expected_langevin_dynamics}
    dW_t = -\nabla \risk(W_t) dt + \sigma dB_t, \quad W_0 = w_0.
\end{align}

The consideration of such expected dynamics to prove generalization bounds has already been studied in other works \citep{amir2022thinking, dupuis2023from}, although in different settings and leveraging distinct proof techniques. According to \Cref{thm:girsanov_expression}, the KL divergence can now be expressed as:
\begin{equation}
        \label{eq:girsanov_limiting_measure}
        \klb{\rho_S}{\pi} = \frac{1}{2\sigma^2} \int_0^T \mathds{E}_{\rho_S} \big[\Vert \nabla \er(W_t) - \nabla \risk (W_t) \Vert^2 \big] dt.
\end{equation}

Interestingly, the term appearing under the integral in \Cref{eq:girsanov_limiting_measure} has the form of a generalization term; it can be expected that this term decreases to $0$ as $n\to\infty$, hence allowing to gain an order of convergence in the bound.
The following proposition is a bound on this KL term, proven by exploiting this idea.

\begin{restatable}{proposition}{propBoundKLLangevinExpectedDyn}
    \label{prop:kl_expected_measures_bounds}
    Suppose that \Cref{ass:lipschitz_loss} holds. With probability at least $1 - \zeta$ over $S \sim \datadist$:
    $$
    \klb{\rho_S}{\pi} \leq\log(1/\zeta) + \frac{ L^2 \beta T}{ n} +  \frac{2 \beta^2 T^2 L^4}{n} .
    $$
\end{restatable}

By optimizing the value of $\lambda$ in the corresponding bound, our result becomes:

$$
 \Eof[\rho_S]{G_S(\wcal)} = \mathcal{O} \left( \Eof[\rho_S]{\rad_S(\wcal)} + \sqrt{\frac{\log(1/\zeta)}{n}} + \frac{\beta T L^2}{n} \right).
$$

\subsection{Rademacher complexity of Langevin dynamics}
\label{sec:rademacher_langevin}

In this section, we prove a bound on the expected Rademacher complexity of Langevin dynamics. Combined with the results of \Cref{sec:brownian_prior,sec:expected_dynamics_prior}, this provides fully computable uniform generalization bounds for Langevin dynamics. 
To perform a covering-like argument, we restrict ourselves to the case of Lipschitz losses. Hence, all the quantities appearing in \Cref{theorem:data_dep_rademacher_set} will be then analytically bounded. Note that, without any Lipschitz assumption, it would also be possible to leverage the results of \Cref{sec:data_dep_fractal_bounds} and introduce a data-dependent fractal dimension in the generalization bound. For the sake of simplicity, and because it may be of independent interest, we focus here on the Lipschitz case.

\begin{restatable}{theorem}{thmRademacherLangevin}
\label{thm:rademacher_complexity_langevin}
Suppose that Assumptions \ref{ass:bounded_loss}, \ref{ass:smoothness} and \ref{ass:lipschitz_loss} hold. We consider \Cref{eq:langevin_empiric}, followed by $(W_t)_{0 \leq t \leq T}$ under $\rho_S$, and denote $\mathcal{W} = \set{W_t, ~0 \leq t \leq T}$, as before. Then there exists a universal constant $C>0$ such that:
    \begin{align*}
        \Eof[\rho_S]{\rad_S(\wcal)} \leq \frac{1}{\sqrt{n}} + \max\set{1,B} \sqrt{\frac{2 \log\left(2T n L^2 (1 + C^2d^2\sigma^2) \right)}{n}}.
    \end{align*}
\end{restatable}

Thus, we see that the terms dominating in our uniform generalization bounds for CLD are the IT terms rather than the Rademacher complexity term. The overall rate of the bounds can be summarized by that of these IT terms.

Therefore, by optimizing the parameter $\lambda$ in \Cref{cor:langevin_brownian_bound}, its results can be informally summarized as:
\begin{align*}
    \Eof[\rho_S]{G_S(\wcal)} = \landau{B \left\{ \frac{1}{2n\sigma^2} \int_0^T  \mathds{E}_{\rho_S} \big[\Vert \nabla \er(W_t) \Vert^2 \big] dt  \right\}^{\frac{1}{2}} },
\end{align*}

which is similar to the results presented in \citep{mou2018generalization}, except that our bound does not feature exponential time decay. However, this is expected, as \citet{mou2018generalization} only bound the generalization gap at time $T$, while we consider the worst case gap over the time interval $[0,T]$. Moreover, note that our bound does not require any convexity, dissipativity, or regularization to hold, while such assumptions are in general necessary to obtain time-uniform bounds, as highlighted by \Cref{tab:cld}. An important aspect of our proof technique is that \Cref{eq:girsanov_brownian_measure,eq:girsanov_limiting_measure} provide exact expressions for the KL divergence, depending on the choice of prior $\pi$. Therefore, the time dependence coming from the integral term can only be improved by significantly adapting our framework and relying on stronger assumptions. We leave these important questions for future work.

Moreover, using the additional Lipschitz assumption, as it is the case in \citep{mou2018generalization, li2020generalization, haghifam2020sharpened, farghly2021time}, we can subsequently optimize the value of the parameter $\lambda$ and get that, with probability at least $1 - \zeta$:
\begin{align*}
    \Eof[\rho_S]{G_S(\wcal)}  = \landau{B \sqrt{\frac{TL^2\beta + \log(1/\zeta)}{n}}}.
\end{align*}

The order of magnitude of these bounds is coherent with existing literature, as one can see in \Cref{tab:cld}, in terms of the relative influence of the quantities $(\beta,T,L,n,B)$.

Finally, it is worth noticing that the application of our methods to stochastic processes differs from the martingale techniques derived by \citet{chugg2023unified}, in addition to being deduced from a different proof technique.
If we denote by $\acal$ the randomness of the algorithm, the bounds in \citep[Theorem $3.1$]{chugg2023unified} would apply on the quantity
$$\sup_{0\leq t \leq T}\Eof[\acal]{ \risk(W_t) - \er(W_t) },$$
It can be understood that the above quantity may be much smaller than the left-hand side of \Cref{cor:langevin_brownian_bound}, by noticing that, in this case, the integration over the posterior $\rho_S$ is equivalent to an expectation over the randomness of the algorithm.

\section{Uniform Generalization Bounds for SGLD}
\label{sec:bounds_for_sgld}

In this section, we consider the case of SGLD, as described by the following recursion, which is a restatement of \Cref{eq:sgld}:
\begin{align}
    \label{eq:sgld_not_intro}
    \forall k \in \mathds{N}, ~ W_{k+1} = W_k - \eta_{k+1} \hat{g}_{k+1} + \sigma_{k+1} \epsilon_{k+1},
\end{align}
where $\sigma_{k+1} := \sqrt{2 \eta_{k+1}\beta^{-1}}$, $\eta_{k+1}$ is the learning rate at iteration $k+1$, $\hat{g}_{k+1}$ is an unbiased estimate of $\nabla \er(W_k)$ and $(\epsilon_k)_{k\geq 1}$ are \iid $\mathcal{N}(0,I_d)$ random variables. The results of this section follow from arguments that are similar to \Cref{sec:langevin_girsanov}, which can be extended to the discrete case through classical arguments (see \Cref{sec:proofs-bounds_for_sgld}). More precisely, we derive, in \Cref{sec:proofs-bounds_for_sgld}, an expression of the KL divergence term appearing in \Cref{theorem:mgf_rademacher}, in the case of SGLD. This leads to the next theorem, which is a uniform generalization bound for SGLD, following from our Rademacher MGF bound, \ie, \Cref{theorem:mgf_rademacher}.  

\begin{restatable}{theorem}{SGLDMainBound}
\label{thm:bound_sgld_normal_walk_prior}
Suppose that \Cref{ass:bounded_loss,ass:integrability_of_z_sgld} hold.
Then, with probability at least $1 - \zeta$ over $S \sim \datadist$, for all $\lambda > 0$
    \begin{align*}
        \Eof[\rho_S]{\max_{w \in \wcal} \left( \risk(w) - \er(w) \right)} \leq  \frac{1}{\lambda} \left( \log(T/\zeta) + \frac{\beta}{4} \sum_{k=1}^T \eta_k \Eof[U,\epsilon]{\normof{\ghat_k}^2} \right)
        + \lambda  \frac{2B^2}{n},
    \end{align*}
    where $(U,\epsilon)$ denotes the randomness of the algorithm, \ie, the randomness coming from the unbiased estimation of $\nabla \er(W_k)$ and the Gaussian noise, repsectively. The dependence of $\hat{g}$ on $S$ has been omitted to ease the notations. The expectation over $\rho_S$, on the left-hand side, may be seen as an expectation over $(U,\epsilon)$ as well.
\end{restatable}

\Cref{ass:integrability_of_z_sgld}, which will be formally introduced in \Cref{sec:proofs-bounds_for_sgld}, is a technical integrability assumption that is necessary for our proofs to hold.
It is satisfied, for instance, if the gradients are uniformly bounded, \ie, the loss is Lipschitz continuous.

To compare it with other works, let us analyze the above bound in the case where the gradients are bounded, \ie, $\Eof[U,\epsilon]{\normof{\ghat_k}^2} \leq L^2$. When $\ghat$ is computed as the average gradient over a batch of data points, this corresponds to assuming that the loss $\ell$ is $L$-Lipschitz continuous. 
Based on this assumption, we can optimize the parameter $\lambda$ in the above theorem to get a bound of the following form:
\begin{align}
    \label{eq:bound_sgld_lipschitz_case}
    \Eof[\rho_S]{\max_{w \in \wcal} \left( \risk(w) - \er(w) \right)} \leq 2B \sqrt{ \frac{4 \log(T/\zeta) + \beta L^2 \sum_{k=1}^T \eta_k}{2n}}.
\end{align}

While, to our knowledge, no uniform bound for SGLD has been proposed, this still allows for a meaningful comparison with the results of \citet{mou2018generalization}, see \Cref{tab:sgld}.
In the aforementioned work, the authors prove high probability bounds with respect to $S$, and in expectation over the noise (\cf our bound in \Cref{thm:bound_sgld_normal_walk_prior}), but with an additional exponential decay in the sum on the right-hand side.
Our result does not feature any exponential decay, this lack of time-uniformity is expected as our bound is uniform over the whole trajectory.
In the Lipschitz case, the order of magnitude of our bound is also comparable to the results of \citep{negrea2019information, neu2021information}. However, note that most works use a subgaussian assumption on the loss $\ell$, while our method requires bounded loss, which is stronger. A comparison of our result with existing bounds is given in \Cref{tab:cld}.

As already mentioned in the introduction, the expectation $\mathds{E}_{\rho_S}$ over the posterior is taken \emph{outside} of the maximum in \Cref{thm:bound_sgld_normal_walk_prior}. To the best of our knowledge, there would be no trivial way to extend existing generalization bounds for SGLD to obtain \Cref{thm:bound_sgld_normal_walk_prior}.

\section*{Conclusion}

In this paper, we introduced a PAC-Bayesian framework to prove data-dependent uniform generalization bounds. We provided a rigorous mathematical formulation of our methods and proved two upper bounds in terms of the moment-generating function of the Rademacher complexity and the data-dependent Rademacher complexity. We additionally demonstrated the ability of our methods to prove data-dependent uniform generalization lower bounds.

We successfully applied the introduced techniques in two particular contexts. First, we used our data-dependent Rademacher complexity term to derive uniform bounds in terms of the fractal dimension of the hypothesis set. Compared to prior art, our method yields tighter bounds and uses the same information-theoretic term for all kinds of fractal dimensions. Moreover, our approach greatly simplifies and unifies the proof techniques of the existing literature. Second, we established that in the context of Langevin dynamics and SGLD, the information-theoretic terms appearing in our PAC-Bayesian bounds can be further upper-bounded by closed-form quantities. This allows us to prove the first uniform generalization bounds over the trajectory of these algorithms. \\

\textbf{Future work.} Some directions remain to be studied regarding our work. First, the generality of the proposed framework opens the door to several refinements of the methods. For instance, one could apply chaining techniques \citep{vershynin2018high}, other PAC-Bayesian or information-theoretic bounds, such as conditional mutual information bounds \citep{steinke2020reasoning}, or try to extend the ``Rademacher viewpoint'' of \citet{kakade2008complexity} and \citet{yang2019fast} into our framework. 
As we mentioned above, our methods could be combined with concentration inequalities in the Bernstein form \citep{mcdiarmid1998concentration} in order to weaken the assumptions. 
Beyond the use of fractal dimensions, our work may help to further bridge the gap between generalization and topological data analysis \citep{andreeva2024topological}.
Regarding Langevin dynamics, it would be beneficial to investigate under which assumption the time dependence of our bounds can be improved.
Finally, the optimization of our PAC-Bayesian bounds with respect to the random set posterior might lead to non-vacuous bounds, extending the study of \citet{dziugaite2017computing} to random sets.

\acks{U.\c{S}. is partially supported by the French government under the management of
Agence Nationale de la Recherche as part of the ``Investissements d'avenir'' program, reference
ANR-19-P3IA-0001 (PRAIRIE 3IA Institute). B.D. and U.\c{S}. are partially supported by the European Research Council Starting Grant
DYNASTY – 101039676.
}

\newpage

\appendix

The appendix is organized as follows:
\begin{itemize}
    \item In \Cref{sec:appendix_technical_background}, we remind the reader of some notation and provide a technical background.
    \item \Cref{sec:posponed_proofs} is dedicated to the omitted proofs from the main part of the paper. 
\end{itemize}

\section{Additional Technical Background}
\label{sec:appendix_technical_background}

In this section, we remind the reader of some probabilistic technical background as well as a few technical lemmas.

\subsection{Probability theory background}
\label{sec:proba_IT}

The goal of this subsection is to introduce notation and definitions. Let us fix some measurable space $(\Omega, \mathcal{T})$.
Given two probability measures $\mu$ and $\nu$ on $(\Omega, \mathcal{T})$, the absolute continuity of $\mu$ with respect to $\nu$ will be denoted $\mu \ll \nu$. If $\mu \ll \nu$, the Kullback-Leibler (KL) divergence between $\mu$ and $\nu$ is defined by:
\begin{align}
    \label{eq:KL_divergence}
    \klb{\mu}{\nu} := \int \log \left( \frac{d\mu}{d\nu} \right) d\mu ,
\end{align}
where $\frac{d\mu}{d\nu}$ denotes the Radon-Nikodym derivative of $\mu$ with respect to $\nu$. If $\mu$ is not absolutely continuous with respect to $\nu$, we set $\klb{\mu}{\nu} = +\infty$, by convention.

A probability space $(\Omega, \mathcal{T}, \prob)$ is said to be \emph{complete} if, for all $A \in \mathcal{T}$ such that $\prob(A) = 0$, we have $\forall B \subseteq A, ~B \in \mathcal{T}$.

Given a random variable $X$ and a probability measure $\prob$ on $(\Omega, \mathcal{T})$, we denote by $\prob_X$, or $X_\# \prob$, the law of $X$, \ie, the image measure of $\prob$ under $X$. Given two random variables $X$ and $Y$, the mutual information between $X$ and $Y$ is defined by: 
\begin{align*}
    \text{\normalfont I}_1 (X,Y) := \klb{\prob_{(X,Y)}}{\prob_X \otimes \prob_Y}.
\end{align*}
This is the most common notion of mutual information, which appears, for instance, in the generalization bounds of \citet{xu2017information}. The total mutual information has been defined in \Cref{def:total-mi}. It satisfies $I_1 \leq I_\infty$.

\subsection{A few technical lemmas}
\label{sec:technical_lemmas}

\looseness=-1
The next lemma is just a way of writing McDiarmid's inequality~\citep{mcdiarmid1998concentration} in an exponential form.
A proof can be found in \citep[Theorem $6.2$]{boucheron2013concentration} for instance.

\begin{lemma}
    \label{lemma:exp_mc_diarmid}
    Consider any function $f: \xcal^n\to \R$, where $\xcal$ is any measurable space. We assume that $f$ satisfies the bounded difference inequality, \ie, for all $i \in \{1,\dots, n\}$ and all $(x_1,\dots,x_n) \in \xcal^n$, one has:
\begin{align*}
\sup_{ x' \in \xcal} \vert f(x_1,\dots,x_n) - f(x_1,\dots, x_{i-1}, x',x_{i+1}, \dots,x_n) \vert \leq c_i,
\end{align*}
Then, given $(X_1,\dots,X_n)$ some \iid random variables on $\xcal$, the random variable $Z := f(X_1,\dots,X_n)$ satisfies:
\begin{align*}
    \mathds{E}\left[ e^{\lambda \left( f(Z) - \mathds{E}f(Z) \right)} \right] \leq e^{\frac{\lambda^2}{8} \sum_{i=1}^{n}c_i^2}
\end{align*}
\end{lemma}

We recall below the symmetrization lemma, which is one of the key ingredients of Rademacher complexity-based bounds, presented in \Cref{sec:uniform-convergence-bounds}.
A proof can be found, for instance, in \citep{shalevschwartz2014understanding}.

\begin{restatable}[Symmetrization]{lemma}{lemmasymmetrization}
\label{lemma:symmetrization}
Let $\mathcal{W}$ be a data-independent (\eg fixed) set. We have:
\begin{align*}
    \mathds{E}_{S} \left[G_S(\wcal) \right] \leq 2\rad(\wcal).
\end{align*}
\end{restatable}

The symmetrization technique can also be used to obtain a lower bound, often called desymmetrization inequality. A proof can be found in \citep{dupuis2023generalization}, the only difference is that we write here the inequality with a slightly better absolute constant than in \citep{dupuis2023generalization}, which can be obtained by using Khintchine's inequality instead of Massart's lemma in the last step of the proof. 

\begin{restatable}[Desymmetrization inequality]{proposition}{desymmetrization}
    \label{prop:desymmetrization_inequality}
    Assume that $\wcal$ is a fixed set and that the loss $\ell$ satisfies \Cref{ass:bounded_loss}. Then we have:

    $$
    \mathds{E}_S \left[ \sup_{w \in \wcal} |\er(w) - \mathcal{R}(w)| \right] \geq \frac{1}{2} \rad(\wcal)  - \frac{B}{2\sqrt{n}}.
    $$
\end{restatable}

The next theorem is Egoroff's theorem; it is a classical result from measure theory \citep{bogachev2007measure}. It has been used in the context of fractal-based generalization bounds \citep{simsekli2020hausdorff, camuto2021fractal, hodgkinson2022generalization, dupuis2023generalization}. 

\begin{theorem}[Egoroff's theorem]
    \label{theorem:egoroff}
    Let $(\Omega, \mathcal{T}, \mu)$ be a finite measure space.  
    Let $f, (f_n)_n: \Omega \longrightarrow X$ be measurable functions, with $X$ some arbitrary separable metric space. Assume that $\mu$-almost everywhere, we have $f_n(x) \longrightarrow f(x)$ as $n \to \infty$. Then, for all $\gamma > 0$, there exists $\Omega_\gamma \in \mathcal{T}$ such that $\mu(\Omega\backslash\Omega_\gamma) \leq \gamma$ and the convergence of $(f_n)$ to $f$ is uniform on $\Omega_\gamma$.
\end{theorem}

\section{Omitted Proofs and Additional Results}
\label{sec:posponed_proofs}

In this section, we present the omitted proofs of all our main results.

\subsection{Omitted proofs of \Cref{sec:main-result-data-dependent}}
\label{sec:proofs-main-result-data-dependent}

\subsubsection{Omitted proofs of \Cref{sec:mgf_bound} - Rademacher MGF}
\label{sec:proof-rademacher-mgf}

Before proving \Cref{theorem:mgf_rademacher}, we first prove the following exponential symmetrization lemma, which is an exponential equivalent of the usual symmetrization that is used in the Rademacher complexity literature \citep[see \eg,][]{shalevschwartz2014understanding}.

\begin{lemma}[Exponential symmetrization lemma]\label{lemma:mgf-rad}
For any set $\wcal$, for any set $\zcal$, for any measurable function $\ell: \wcal\times \zcal \to \R$, for any $\lambda>0$, we have,
\begin{align*}
\mathds{E}_S\LB \exp\LC \lambda \sup_{w \in \wcal} \LP \mathcal{R}(w) - \er(w) \RP \RC \RB \le \mathds{E}_S\mathds{E}_\epsilon\exp\LC \frac{2\lambda}{n} \sup_{w \in \wcal}\sum_{i=1}^n \epsilon_i\ell(w,z_i)\RC,
\end{align*}
where the previous inequality holds as soon as the measurability of the quantities inside the expectations is ensured.
\end{lemma}

\begin{proof}
Let $S' := (z_1',\dots,z_n') \sim \datadist$ be independent of $S \sim \datadist$, by Jensen's inequality:
\begin{align*}
\mathds{E}_S\LB \exp\LC \lambda \sup_{w \in \wcal} \LP \mathcal{R}(w) - \er(w) \RP \RC \RB 
&= \mathds{E}_S \LB \exp\LC \frac{\lambda}{n} 
 \sup_{w \in \wcal}\sum_{i=1}^{n} \mathds{E}_{S'}[\ell(w,z_i') - \ell(w,z_i)]\RC \RB\\
&\le \mathds{E}_S \LB \exp\LC \mathds{E}_{S'} \frac{\lambda}{n}\sup_{w \in \wcal} \sum_{i=1}^n \ell(w,z_i') - \ell(w,z_i)\RC \RB \\
&\le \mathds{E}_{S,S'}  \exp\LC \frac{\lambda}{n} \sup_{w \in \wcal}\sum_{i=1}^n( \ell(w,z_i') - \ell(w,z_i))\RC .
\end{align*}
By the usual symmetrization trick, see \citep[Lemma $26.2$]{shalevschwartz2014understanding}, we can introduce $(\epsilon_1, \dots, \epsilon_n)$ some Rademacher random variables and write:
\begin{align*}
\mathds{E}_{S,S'}  \exp&\LC \frac{\lambda}{n} \sup_{w \in \wcal}\sum_{i=1}^n \LP \ell(w,z_i') - \ell(w,z_i) \RP \RC  = \mathds{E}_{S,S',\epsilon} \exp\LC \frac{\lambda}{n} \sup_{w \in \wcal}\sum_{i=1}^n \epsilon_i\LB \ell(w,z_i') - \ell(w,z_i)\RB \RC \\
&\le \mathds{E}_S\mathds{E}_{S'}\mathds{E}_\epsilon \LB \exp\LC \frac{\lambda}{n} \sup_{w \in \wcal}\sum_{i=1}^n \epsilon_i\ell(w,z_i') + \frac{\lambda}{n} \sup_{w \in \wcal}\sum_{i=1}^n (-\epsilon_i)\ell(w,z_i)  \RC \RB\\
&= \mathds{E}_S\mathds{E}_{S'}\mathds{E}_\epsilon \LB \exp\LC \frac{\lambda}{n} \sup_{w \in \wcal}\sum_{i=1}^n \epsilon_i\ell(w,z_i')\RC  \exp\LC \frac{\lambda}{n} \sup_{w \in \wcal}\sum_{i=1}^n (-\epsilon_i)\ell(w,z_i)\RC \RB.
\end{align*}
From Cauchy-Schwarz's inequality, we finally obtain
\begin{align*}
&\mathds{E}_S\mathds{E}_{S'}\mathds{E}_\epsilon \LB \exp\LC \frac{\lambda}{n} \sup_{w \in \wcal}\sum_{i=1}^n \epsilon_i\ell(w,z_i')\RC \exp\LC \frac{\lambda}{n} \sup_{w \in \wcal}\sum_{i=1}^n (-\epsilon_i)\ell(w,z_i)\RC \RB\\
&\le \LB\mathds{E}_{S'}\mathds{E}_{\epsilon}  \exp\LC \frac{2\lambda}{n} \sup_{w \in \wcal}\sum_{i=1}^n \epsilon_i\ell(w,z_i')\RC\RB^\frac{1}{2} \LB\mathds{E}_S\mathds{E}_\epsilon  \exp\LC \frac{2\lambda}{n} \sup_{w \in \wcal}\sum_{i=1}^n (-\epsilon_i)\ell(w,z_i)\RC \RB^\frac{1}{2}\\
&= \LB \mathds{E}_S\mathds{E}_\epsilon \exp\LC \frac{2\lambda}{n} \sup_{w \in \wcal}\sum_{i=1}^n \epsilon_i\ell(w,z_i)\RC\RB^\frac{1}{2} \LB  \mathds{E}_S\mathds{E}_\epsilon\exp\LC \frac{2\lambda}{n} \sup_{w \in \wcal}\sum_{i=1}^n \epsilon_i\ell(w,z_i)\RC \RB^\frac{1}{2}\\
&= \mathds{E}_S\mathds{E}_\epsilon\exp\LC \frac{2\lambda}{n} \sup_{w \in \wcal}\sum_{i=1}^n \epsilon_i\ell(w,z_i)\RC.
\end{align*}

\end{proof}

We can now prove \Cref{theorem:mgf_rademacher}.

\begin{proof}
Let us fix some $\lambda > 0$, we apply \Cref{theorem:pac-bayes-set} with $\Phi_\lambda(\wcal,S) = \lambda \sup_{w \in \wcal} \big( \mathcal{R}(w) - \er(w) \big)$. Therefore, our task boils down to a bound on the log-exp term, which we achieve by applying Fubini's theorem and \Cref{lemma:mgf-rad}. This gives:
\begin{align*}
\mathds{E}_S\mathds{E}_{\wcal \sim \pi} e^{\lambda \sup_{w \in \wcal} \big( \mathcal{R}(w) - \er(w) \big)} &= \mathds{E}_{\wcal \sim \pi}\mathds{E}_S e^{\lambda \sup_{w \in \wcal} \big( \mathcal{R}(w) - \er(w) \big)}\\
&\le \mathds{E}_{\wcal \sim \pi}\mathds{E}_S\mathds{E}_\epsilon\exp\LC \frac{2\lambda}{n} \sup_{w \in \wcal}\sum_{i=1}^n \epsilon_i\ell(w,z_i)\RC\\
&= \mathds{E}_{\wcal \sim \pi}\mathds{E}_S\Psi_{S,\wcal}(2\lambda),
\end{align*}
implying the desired results.
\end{proof}

\subsubsection{Omitted proofs of \Cref{sec:data_dep_rademacher_bound}}
\label{sec:proof-rademacher-data-dependent}

We end this section by giving the proof of \Cref{theorem:data_dep_rademacher_set}.
\begin{proof}
    Let us define the function, defined for any $\lambda > 0$:
    \begin{align}
        \label{eq:gen_minus_rad}
        \Phi_\lambda (\wcal, S) := \lambda\sup_{w \in \wcal} \big( \mathcal{R}(w) - \er(w) \big)  - 2\lambda \rad_S(\wcal).
    \end{align}
    Our goal is to apply the results of \Cref{theorem:pac-bayes-set} to function $\Phi_\lambda$. Our assumptions ensure that the above terms are well-defined and measurable.
    Therefore, for both inequalities (KL-based and disintegrated), our task boils down to bounding the following quantity:
    \begin{align*}
    \mathcal{L} = \log \mathds{E}_S \mathds{E}_{\wcal \sim \pi} \left[ \exp \left\{ \lambda \sup_{w \in \wcal} \big( \mathcal{R}(w) - \er(w) \big)  - 2\lambda \rad_S(\wcal) \right\} \right].
    \end{align*}
    The key of our reasoning is that in the above expectation, the variables $S \sim \datadist$ and $\wcal \sim \pi$ are now independent. This justifies the following considerations.

    Let us denote $(\epsilon_1,\dots,\epsilon_n)$ some \iid Rademacher random variables, independent of $S$ and $\wcal$.
    In order to bound $\mathcal{L}$ we remark that we have (with $S^i$ being $S = (z_1,\dots,z_n)$ with $i$-th element replaced by another one, denoted $z_i'$):
    \begin{align*}
    |\Phi_\lambda (\wcal, S) - \Phi_\lambda (\wcal, S^i) \vert &= \lambda \bigg\vert \sup_{w \in \wcal} \big(\mathcal{R}(w){-}\er(w) \big){-}2\rad_S(\wcal) \\
    &\quad\quad\quad- \sup_{w \in \wcal} \big( \mathcal{R}(w){-}\widehat{\mathcal{R}}_{S^i}(w) \big){-}2\rad_{S^i}(\wcal) \bigg\vert\\
    &\le  \frac{\lambda}{n}\sup_{w \in \wcal}\Big\vert\ell(w, z_i')-\ell(w, z_i) \Big\vert + 2\lambda \left\vert \rad_{S^i}(\wcal) {-} \rad_S(\wcal) \right\vert\\
    &\le \frac{B\lambda}{n} + \frac{2\lambda}{n}\mathds{E}_{\epsilon} \left[\sup_{w\in\wcal}\left\vert\epsilon_i\big( \ell(w, z_i)-\ell(w, z'_i)\big) \right\vert \right]\\
    &\le \frac{3\lambda B}{n},
    \end{align*}
    where we used the fact that $\ell$ is bounded in $[0,B]$. From \Cref{lemma:exp_mc_diarmid}, we deduce that:
    \begin{align*}
    \mathds{E}_S\left[  e^{\Phi_\lambda(\wcal,S) - \mathds{E}_S[\Phi_\lambda(\wcal,S) ]} \right] \le \exp\left\{ \frac{9\lambda^2B^2}{8n}\right\}.
    \end{align*}
    Moreover, by the classical symmetrization inequality, \Cref{lemma:symmetrization}, we know that:
    \begin{equation}
    \label{eq:use_of_symmetrization_in_proof}
    \forall \wcal \in E, ~\mathds{E}_S\left[\Phi_\lambda(\wcal,S) \right] \leq 0.
    \end{equation}
    Therefore, by Fubini's theorem, we have:
    \begin{align*}
        \mathcal{L} &= \log \mathds{E}_S\mathds{E}_{\wcal \sim \pi} \left[  e^{\Phi_\lambda(\wcal,S) } \right]\\
        &= \log\mathds{E}_S\mathds{E}_{\wcal \sim \pi} \left[  e^{\Phi_\lambda(\wcal,S)  - \mathds{E}_S[\Phi_\lambda(\wcal,S) ] } e^{ \mathds{E}_S[\Phi_\lambda(\wcal,S) ] } \right]\\
        &\leq \log\mathds{E}_{\wcal \sim \pi} \mathds{E}_S\left[  e^{\Phi_\lambda(\wcal,S)  - \mathds{E}_S[\Phi_\lambda(\wcal,S) ] }\right]\\
        &\leq  \frac{9\lambda^2B^2}{8n}.
    \end{align*}
    The desired bounds immediately follow.
\end{proof}
The proof of \Cref{theorem:data_dep_general_form} follows the same lines, with just a change in the function $\Phi$, the bounded difference condition being obtained through the inverted triangular inequality.

\subsection{Data-dependent uniform generalization bounds with IPMs}
\label{sec:ipm-bounds}

The goal of this subsection is to further underline the generality of our framework by briefly extending our PAC-Bayesian framework on random sets with the Integral Probability Metrics (IPM) used by \citet{amit2022integral} and \citet{viallard2023learning,viallard2024tighter} to derive general PAC-Bayesian bounds.
As for the main results of \Cref{sec:pac_bayes_for_random_sets,sec:main-result-data-dependent}, this extension is straightforward given our measurability assumptions.

With the notations of \Cref{sec:pac_bayes_for_random_sets}, we have the following definition of IPMs on $E$, see \citep[Definition $3$]{amit2022integral}.  

\begin{definition}
    \label{def:ipm}
    Let $\fcal$ be a family of functions $E\to\R$ and $\mu$ and $\nu$ be two probability distributions on $(E,\mathfrak{E})$. The IPM between $\mu$ and $\nu$ associated with $\fcal$ is defined as:
    \begin{align*}
        \gamma_{\fcal}(\mu,\nu) := \sup_{\phi \in \fcal} \left| \Eof[\nu]{\phi(\wcal)} - \Eof[\mu]{\phi(\wcal)}  \right|.
    \end{align*}
\end{definition}

\looseness=-1
Using \Cref{def:ipm}, we can easily state the following bound based on IPMs over random sets.  

\begin{theorem}
    \label{thm:ipm-random-set}
    Let $(E,\mathfrak{E})$ be defined as in \Cref{sec:pac_bayes_for_random_sets}, $\pi$ be a fixed prior distribution on $\Rd$ and $\Phi: E \times \zcal^n \longrightarrow \mathds{R}$ be a measurable function. For any $n \in \mathds{N}^*$ and $S\in \zcal^n$, we consider a family $\fcal_S$ of bounded measurable functions $E\longrightarrow\R$. We assume that for every $n$ and every $S\in\zcal$, we have $\Phi(\cdot, S)\in\fcal_S$. Then, we have:
    \begin{align*}
        \mathds{P}_S \left(\forall \rho \in \pcal_0(\Rd),~  \mathds{E}_{\rho}\Phi(\wcal, S) \leq \gamma_{\fcal_S} (\rho,\pi) + \log(1/\zeta) + \log \mathds{E}_S\mathds{E}_{\pi} \left[ e^{\Phi(\wcal, S)} \right] \right) \geq 1 - \zeta,
    \end{align*}
    where $\pcal_0(\Rd)$ denotes the set of probability distributions over $\Rd$.
\end{theorem}

\begin{proof}
    The proof mimics the one of Proposition~4 in \citep{amit2022integral}, with only the difference of using a general function $\Phi: E \times \zcal^n \longrightarrow \mathds{R}$. By definition of IPM and Jensen's inequality, we have, for any $\rho \in \pcal_0(\Rd)$:
    \begin{align*}
        e^{\Eof[\wcal\sim\rho]{\Phi(\wcal,S)} - \gamma_{\fcal_S} (\rho,\pi) } \leq e^{\Eof[\wcal\sim\pi]{\Phi(\wcal,S)}} \leq \Eof[\wcal\sim\pi]{e^{\Phi(\wcal,S)}}.
    \end{align*}
    Therefore, by Markov's inequality, with probability at least $1-\zeta$ over $S\sim \datadist$ we have:
    \begin{align*}
        \sup_{\rho \in \pcal_0(\Rd)} e^{\Eof[\wcal\sim\rho]{\Phi(\wcal,S)} - \gamma_{\fcal_S} (\rho,\pi) } \leq \frac{1}{\zeta}  \mathds{E}_S\mathds{E}_{\wcal\sim\pi} \left[ e^{\Phi(\wcal, S)} \right].
    \end{align*}
    Applying the logarithm on both sides gives the result.
\end{proof}

For instance, let us extend \Cref{thm:rad_data_dep_dim} with IPMs. The corresponding bound is given by the following corollary.

\begin{corollary}
    \label{cor:ipm-rad-bound}
    Suppose that \Cref{ass:bounded_loss,ass:supremum_measurability} hold.
    For any $n \in \mathds{N}^*$ and $S\in \zcal^n$, we consider a family $\fcal_S$ of bounded measurable function $E\to\R$. Let us consider some $\lambda > 0$ such that for every $n$ and every $S\in\zcal$, we have $\Phi_\lambda(\cdot, S)\in\fcal_S$, where $\Phi_\lambda$ is given by \Cref{eq:phi_function_for_rademacher}. Then, with probability at least $1 - \zeta$ over $S\sim\datadist$, we have
    \begin{align*}
        \forall \rho\in\pcal_0(\Rd),\quad \mathds{E}_{\wcal\sim\rho}\Phi(\wcal, S) \leq 2 \Eof[\wcal\sim\rho]{\rad_S(\wcal)} + \frac{\gamma_{\fcal_S} (\rho,\pi) + \log(1/\zeta)}{\lambda} + \lambda \frac{9 B^2}{8n}.
    \end{align*}
\end{corollary}

\begin{proof}
    We use \Cref{thm:ipm-random-set} and upper bound the term $\E_S\E_\pi \left[ e^{\Phi_\lambda(\wcal, S)} \right]$ exactly as in the proof of \Cref{thm:rad_data_dep_dim}.
\end{proof}

\subsection{Fractal based generalization bounds - Omitted proofs of \Cref{sec:fractal_bounds}}
\label{sec:proofs-fractal_bounds}

In this section, we present the omitted proofs of \Cref{sec:fractal_bounds}.
As in other fractal-based works on generalization bounds \citep{simsekli2020hausdorff,dupuis2023generalization}, the bounds involving fractal dimensions are deduced from bounds involving covering numbers. 
Therefore, we first give generalization bounds with data-dependent covering numbers.

\subsubsection{Data-dependent covering numbers}
\label{sec:covering-bounds-proofs}

We deduce two covering bounds from the Rademacher complexity bound of \Cref{theorem:data_dep_rademacher_set}. The first one, presented in the next corollary, uses covering numbers defined through the data-dependent pseudometric introduced in \Cref{eq:data_dep_pseudo_metric}.
It is a direct consequence of \Cref{theorem:data_dep_rademacher_set} and Massart's Lemma, as it is classically done in the analysis of Rademacher complexity \citep{shalevschwartz2014understanding}.
As these arguments are very classical, we omit the proofs to avoid harming the readability of the paper.

\begin{restatable}{corollary}{thmCoveringDataDep}
    \label{cor:rad_covering_numbers}
    Under Assumptions \ref{ass:bounded_loss}, \ref{ass:supremum_measurability} and \ref{ass:measurable_covering_numbers}, there exists an absolute constant $C > 0$, such that, for any $\lambda, \delta > 0$, with probability at least $1 - \zeta$ under $S \sim \datadist$, we have, with probability at least $1 - \zeta$ over $S\sim\datadist$ and $\wcal\sim\rho_S$:
    \begin{align*}
     G_S(\wcal) \leq   2 \delta + 2B \sqrt{\frac{2 \log(\vert N_\delta^{\vartheta_S}(\wcal) \vert)}{n}}+ \frac{\log \frac{d \rho_S}{d\pi}(\wcal) + \log(1/\zeta)}{\lambda} + C\lambda \frac{B^2}{n}.
    \end{align*}
\end{restatable}

In some cases, one may be interested in introducing covering numbers with respect to the Euclidean distance on $\Rd$.
This is, for instance, the setting considered by \citet{simsekli2020hausdorff} and \citet{hodgkinson2022generalization}. As highlighted by these authors, this requires a Lipschitz continuity on the loss $\ell$. This leads to the next corollary.

\begin{restatable}{corollary}{CorollaryEuclideanCoverBound}
    \label{cor:euclidean_covering_numbers}
    Suppose that Assumptions \ref{ass:bounded_loss}, \ref{ass:supremum_measurability} and \ref{ass:measurable_covering_numbers} hold. We assume that $\ell(w,z)$ is $L$-Lipschitz in $w$ and that $\wcal$ is $\pi$-almost surely bounded.
      Then there exists a constant $C > 0$, such that, for any $\lambda, \delta > 0$, with probability at least $1 - \zeta$ under $S\sim \datadist$ and $\wcal \sim \rho_S$:
      \begin{align*}
           G_S(\wcal) \leq  2L\delta + 2B \sqrt{\frac{2 \log(\vert N_\delta(\wcal) \vert)}{n}} + \frac{\log \frac{d \rho_S}{d\pi}(\wcal) + \log(1/\zeta)}{\lambda} + C\lambda \frac{B^2}{n}
      \end{align*}
\end{restatable}

Thus, we are able to deduce two types of covering bounds from our data-dependent Rademacher complexity bounds. In the next two subsections, we will deduce fractal-based generalization bounds built on these results.

\subsubsection{Proof of \Cref{thm:rad_data_dep_dim}}
\label{sec:data-dep-fractal-bound-proof}

\textbf{Presentation of \Cref{ass:uniform_cv_data_dep}.} As mentioned in \Cref{sec:data_dep_fractal_bounds}, as an additional contribution, our framework is suitable for the creation of natural assumptions to handle the uniformity in $n$ of the limit in \Cref{eq:upper_box_counting}, defining the data-dependent fractal dimension.

For $S \in\zcal^n$ and $\wcal \sim \rho_S$, the covering number $|N^{\vartheta_S}(\wcal)|$ has a dependence in $n$ through its dependence in the dataset $S$.
To overcome this technical difficulty, we observe that
by definition of the upper box-counting dimension, we have, for all $S \in \zcal^n$:
\begin{align*}
\upperbox^{\vartheta_S} (\wcal) := \lim_{\delta \to 0} 
\sup_{0<r<\delta} \frac{\log(\vert N_r^{\vartheta_S}(\wcal) \vert)}{\log(1/r)}.
\end{align*}
We know that (almost) sure convergence implies convergence in probability. For this reason, it makes sense to assume the uniformity in $n$ of this convergence in probability, which is formalized by the following assumption.
\begin{assumption}
    \label{ass:uniform_cv_data_dep}
    We assume that, for all $\epsilon > 0$, one has:
    \begin{align*}
    \sup_{n \in \mathds{N}^\star}\int_{\zcal^n} \rho_S \left( \sup_{0<r<\delta} \frac{\log(\vert N_r^{\vartheta_S}(\wcal) \vert)}{\log(1/r)} - \upperbox^{\vartheta_S} (\wcal) \geq \epsilon \right) d\datadist(S) \underset{\delta\to 0}{\longrightarrow} 0.
    \end{align*}
\end{assumption}

We can now present the proof of \Cref{thm:rad_data_dep_dim}.

\begin{proof}
    Let us fix $\epsilon, \gamma > 0$. From \Cref{ass:uniform_cv_data_dep}, we know that there exists $\delta_{\gamma, \epsilon} > 0$ such that, for all $n \in \mathds{N}^\star$ and $\delta < \delta_{\gamma, \epsilon}$, with probability at least $1 - \gamma$ under $\mathds{P}_{S\sim\datadist,\wcal\sim\rho_S}$:

$$
 \log(\vert N_\delta^{\vartheta_S}(\wcal) \vert) \leq \log(1/\delta) (\epsilon + \upperbox^{\vartheta_S} (\wcal) ).
$$

Now, we define the sequence $\delta_n := 1/n$, for $n\geq 1$. For $n > \lceil 1/\delta_{\gamma, \epsilon} \rceil$, we therefore have that:

$$
\mathds{P}_{S\sim\datadist,\wcal\sim\rho_S} \big(  \log(\vert N_{1/n}^{\vartheta_S}(\wcal) \vert) \leq \log(n) (\epsilon + \upperbox^{\vartheta_S} (\wcal) ) \big) \geq 1 - \gamma.
$$

The result follows from a union bound and \Cref{cor:rad_covering_numbers}.
\end{proof}

\subsubsection{Proof of \Cref{thm:euclidean_dim_with_tv_assumption}}
\label{sec:euclidean_fractal-bound-proof}

\textbf{Presentation of \Cref{ass:total_vaiation_cv}.} Let us introduce the probability space $\zcal^{\infty}$, endowed with the cylindrical $\sigma$-algebra (denoted $\mathcal{F}^{\otimes \infty}$) and the product measure $\mu_z^{\otimes \infty}$. For any $S = (z_i)_{i\geq 1} \in \zcal^\infty$, we denote the canonical projection $\zcal^\infty \longrightarrow \zcal^n$ by $S_n := (z_1, \dots, z_n)$.

The following assumption consists of the convergence of the posterior distributions to a limit distribution when $n \to \infty$, in the sense of the total variation distance.

\begin{assumption}
    \label{ass:total_vaiation_cv}
    There exists a probability measure $\mathds{Q}$ on the space of hypothesis sets $(E, \mathfrak{E})$, such that\footnote{We define the total variation between two measures $\mu$ and $\nu$ as $2\sup_A |\mu(A) - \nu(A)|$, some authors remove the $2$ from this definition. This wouldn't affect any of the results.}, for $\datainfty$-almost all $S \in \zcal^{\infty}$:
    \begin{align*}
    \tv (\rho_{S_n}, \mathds{Q}) := 2\sup_{A \in \mathfrak{E}} \big| \rho_{S_n}(A) - \mathds{Q}(A) \big| \underset{n \to \infty}{\longrightarrow} 0.
    \end{align*}
    Note that we do not impose the distribution $\q$ to be equal to the prior distribution $\pi$, but it may be the case in particular applications.
\end{assumption}
By Pinsker's inequality, this is weaker than assuming a convergence in the KL divergence.

Based on the above assumption, we can present the proof of \Cref{thm:euclidean_dim_with_tv_assumption}.

\begin{proof}
    Let us consider a decreasing sequence $(\delta_k)$ such that $\forall k,~\delta_k > 0$ and $\delta_k \to 0$.

    For any bounded set $\wcal \subseteq \mathds{R}^d$ and $\delta > 0$, we introduce the notation:
    \begin{equation}
        \label{eq:f_delta_euclidean}
    f_\delta(\wcal) := \sup_{0<r<\delta} \frac{\log(\vert N_r(\wcal) \vert)}{\log(1/r)} - \upperbox (\wcal).
    \end{equation}
    Note that $f_\delta$ is measurable, because the supremum may be taken over rational numbers in the interval $(0,\delta)$.
    Let us fix $\epsilon,\gamma > 0$. For $\mu_z^{\otimes \infty}$-almost all $S \in \zcal^\infty$, we have, because of the total variation convergence assumption:
    $$
    \sup_{k \in \mathds{N}} \big| \rho_{S_n} (f_{\delta_k} (\wcal) \geq \epsilon) - \mathds{Q} (f_{\delta_k} (\wcal) \geq \epsilon)  \big| \underset{n \to \infty}{\longrightarrow} 0.
    $$

    Thanks to the Markov kernel assumption on $S \longmapsto \rho_S(\cdot)$ and by construction of the cylindrical $\sigma$-algebra $\mathcal{F}^{\otimes \infty}$, it can be seen that the mappings:
    $$
    \zcal^\infty \ni S \longmapsto h_n(S) := \sup_{k \in \mathds{N}} \big| \rho_{S_n} (f_{\delta_k} (\wcal) \geq \epsilon) - \mathds{Q} (f_{\delta_k} (\wcal) \geq \epsilon)  \big|,
    $$
    are $\mathcal{F}^{\otimes \infty}$-measurable for any $n \in \mathds{N}^\star$, as a countable supremum of measurable functions. 
    Therefore, we can apply Egoroff's theorem (\Cref{theorem:egoroff}) to this sequence of function\footnote{Note that Egoroff's theorem only requires almost everywhere convergence, which is the case here.} to find a set $\Omega_\gamma \in \mathcal{F}^{\otimes \infty}$, such that $\mu_z^{\otimes \infty} (\Omega_\gamma) \geq 1 - \gamma$, and on which the above convergence is uniform, with respect to $S \in \Omega_\gamma$.
    Therefore, we can find $n_{\gamma,\epsilon}^1$, such that, for every $n \geq n_{\gamma,\epsilon}^1$:
    $$
    \forall S \in \Omega_\gamma,~\forall k \in \mathds{N}^\star,~  \rho_{S_n} (f_{\delta_k} (\wcal) \geq \epsilon) \leq \gamma +  \mathds{Q} (f_{\delta_k} (\wcal) \geq \epsilon). 
    $$
    By definition of the upper Minkowski dimension, we know that $f_{\delta_k} (\wcal) \to 0$, pointwise, when $k \to \infty$. Therefore, we also have the convergence in probability $\mathds{Q} (f_{\delta_k} (\wcal) \geq \epsilon) \underset{k \to 0}{\longrightarrow} 0$.
    Applying this to the sequence $\delta_n = 1/n$, we find that there exists $n_{\gamma,\epsilon}^2 \in \mathds{N}^\star$, such that:
    $$
    \forall n \geq n_{\gamma,\epsilon}^2,~\mathds{Q} (f_{\delta_n} (\wcal) \geq \epsilon) \leq \gamma.
    $$
    Setting $n_{\gamma,\epsilon} = \max (n_{\gamma,\epsilon}^1, n_{\gamma,\epsilon}^2)$, we have that, for $n \geq n_{\gamma,\epsilon}$:
    $$
    \begin{aligned}
        \int_{\zcal^n} \rho_S (f_{\delta_n}(\wcal) \geq \epsilon) d\datadist(S) &= \int_{\zcal^{\infty}} \rho_{S_n} (f_{\delta_n}(\wcal) \geq \epsilon) d\mu_z^{\otimes \infty} (S) \\
        &\leq \gamma + \int_{\Omega_\gamma} \rho_{S_n} (f_{\delta_n}(\wcal) \geq \epsilon) d\mu_z^{\otimes \infty} (S) \\
        &\leq \gamma + 2\gamma\mu_z^{\otimes \infty} (\Omega_\gamma)  \\
        &\leq 3\gamma.
    \end{aligned}
    $$
    Hence, there exists $n_{\gamma,\epsilon}$ (potentially slightly different), such that, with probability at least $1 - 3\gamma$ under the joint law of $S \sim \datadist$ and $\wcal\sim\rho_S$, we have, for $n \geq n_{\gamma,\epsilon}$:
    $$
    \log(|N_{\delta_n}(\wcal)|) \leq \log(n) \big( \epsilon + \upperbox(\wcal) \big).
    $$
    The result then immediately follows from \Cref{cor:euclidean_covering_numbers} and a union bound.
\end{proof}

\subsubsection{Proof of \Cref{lemma:smallest-it}}
\label{sec:smalest-it}

\begin{proof}
    Let us denote $\mathds{P}_{S, \rho_S}$ the joint distribution of $S$ and $\rho_S$. Note that there is a slight abuse of notation here, as $S$ is a random variable and $\rho_S$ a distribution, we use it to ease further notations. We have, for any $A \in  \mathfrak{E}\otimes\mathcal{F}^{\otimes n}$:
    \begin{align*}
        \Eof[(S,\wcal) \sim \mathds{P}_{S, \rho_S}]{\mathds{1}_A(\wcal,S)} &= \int_{\zcal^n} \int \mathds{1}_A(\wcal,S) d\rho_S(\wcal) d\datadist(S) \\
       &= \int_{\zcal^n} \int \mathds{1}_A(\wcal,S)\frac{d \rho_S}{d\pi}(\wcal) d\pi(\wcal) d\datadist(S).
    \end{align*}
    Therefore $\frac{d \mathds{P}_{S, \rho_S}}{d (\datadist \otimes \pi)}(\wcal,S) = \frac{d \rho_S}{d\pi}(\wcal)$.
    For the purpose of this proof, let us introduce the ``infinite'' Rényi-divergence. For some measurable space $(\Omega, \mathcal{T})$ and two probability measures $\mu$ and $\nu$ such that $\mu \ll \nu$, we define:
    \begin{align*}
        D_\infty(\mu||\nu) := \log \left( \sup_{A \in \mathcal{T}} \frac{\mu(A)}{\nu(A)} \right).
    \end{align*}
    From \citet[Theorem $6$]{vanerven2014renyi}, we get:
    \begin{align*}
    I_\infty(\wcal_S,S) &= D_\infty \left( \mathds{P}_{S, \rho_S} \| \datadist \otimes \pi \right) \\&= \log \left(\esssup \left( \frac{d \mathds{P}_{S, \rho_S}}{d (\datadist \otimes \pi)}(\wcal,S) \right) \right) \\&= \log \left( \esssup \left(\frac{d \rho_S}{d\pi}(\wcal) \right) \right),
    \end{align*}
    where the essential supremum is over $\mathds{P}_{S, \rho_S}$. The result follows.
\end{proof}

\subsection{Additional details and omitted proofs of \Cref{sec:langevin_girsanov}}
\label{sec:proofs-langevin_girsanov}

\subsubsection{Expression of the KL divergence - omitted proofs of \Cref{sec:langevin_girsanov}}

The next result gives the expression of the IT terms appearing in \Cref{theorem:pac-bayes-set}. 
\begin{restatable}{theorem}{thmLangevinITterms}
    \label{thm:girsanov_expression}
    Let $F: \mathds{R}^d \longrightarrow \mathds{R}$ be an arbitrary function that is bounded, Lipschitz, and smooth (\ie, \Cref{ass:smoothness}). Let $S \in \zcal^n$, we assume that $\ell$ satisfies \Cref{ass:bounded_loss,ass:smoothness,ass:lipschitz_loss}. 
    We consider a probability measure $\pi$ on $\Omega$, under which we have the following SDE:
    \begin{align}
        dW_t = -\nabla F(W_t) dt + \sigma dB_t, \quad W_0 = w_0.
    \end{align}
    Then there exists a probability measure $\rho_S$ on $\Omega$ such that, under $\rho_S$, $W$ satisfies:
    \begin{align*}
        dW_t = -\nabla \er(W_t) dt + \sigma dB_t^S, \quad W_0 = w_0,
    \end{align*}
    with $(B_t^S)_{t\geq 0}$ a $\rho_S$-Brownian motion. Moreover, we have $\rho_S \sim \pi$ and the following relation:
    \begin{align*}
        \klb{\rho_S}{\pi} = \frac{1}{2\sigma^2} \int_0^T \mathds{E}_{\rho_S} \big[\Vert \nabla \er(W_t) - \nabla F(W_t) \Vert^2 \big] dt,
    \end{align*}
 \end{restatable}

The proof of this result follows from classical arguments and is an adaptation of \citep{aristoff2012estimating}.
The reader may find technical background related to Girsanov's theorem and Noikov's condition in \cite{oksendal2003stochastic}.

\begin{proof}
    Let us fix $S$ and denote $U := F - \er$. We also denote by $\mathcal{F}_{t \geq 0}$ a right-continuous filtration of $(\Omega, \mathcal{T})$ such that the Brownian motion $(B_t)_t$ is adapted to $(\mathcal{F}_t)_t$. Without loss of generality, we assume $\mathcal{F}_T = \mathcal{T}$, where $\mathcal{T}$ is the $\sigma$-algebra on $\Omega$.
    
    We define the probability measure $\rho_S$, on the filtration $(\mathcal{F}_t)$ by:
    \begin{align}
    \label{eq:rnd_girsanov_cld}
        \frac{d \rho_S |_{\mathcal{F}_t}}{d \pi |_{\mathcal{F}_t}}:= \exp \bigg\{\frac{1}{\sigma} \int_0^t \nabla U(W_s) \cdot d B_s -  \frac{1}{2\sigma^2} \int_0^t \Vert \nabla \er(W_s) - \nabla \mathcal{F}(W_s) \Vert^2 ds\bigg\},
    \end{align}
    
    where $\pi$ and $\rho_S$ denote the restrictions to $\mathcal{F}_T$.
    
    It is known, thanks to the Novikov condition and our assumptions on the SDEs, that this defines a continuous $\pi$-martingale. Its stochastic logarithm, $\frac{1}{\sigma}\int_0^t \nabla U(W_s) \cdot d B_s$,
    is also a martingale.
    As $B_t$ is also a $\pi$-martingale, by Girsanov's theorem, we know that the following is a continuous local $\rho_S$-martingale:
    $$
    Y_t := B_t - \frac{1}{\sigma}  \int_0^t \nabla U(W_s)ds.
    $$
    Moreover, $[Y,Y]_t = t(\delta_{ij})_{1 \leq i ,j \leq d}$ (the quadratic variation is the same when defining it with two equivalent probability measures), so by Lévy's theorem, it is actually a $\rho_S$-Brownian  motion, which we will denote $B_t^S := Y_t$.

    Then we have, almost surely (under either $\pi$ or $\rho_S$):
    $$
    \sigma B_t^S = \sigma B_t - \int_0^t \nabla F(W_s)ds + \int_0^t \nabla \er(W_s)ds = W_t - W_0  + \int_0^t \nabla \er(W_s)ds,
    $$
    which is the desired dynamics. By a similar calculation based on Itô's lemma, we have:
    \begin{align*}
        \frac{1}{\sigma} \int_0^t \nabla U (W_s) \cdot dB_s = \frac{1}{\sigma}  \int_0^t \nabla U (W_s) \cdot dB_s^S + \frac{1}{\sigma^2} \int_0^t \normof{\nabla U (W_s)}^2 ds.
    \end{align*}

    Hence, for the KL divergence, by the martingale property and Fubini's theorem, using \Cref{eq:rnd_girsanov_cld}, we have:
    $$
    \begin{aligned}
    \klb{\rho_S}{\pi} &= \frac{1}{\sigma} \mathds{E}_{\rho_S} \bigg[ \int_0^T \nabla U(W_t) \cdot d B_t^S \bigg] +  \frac{1}{2\sigma^2}\mathds{E}_{\rho_S} \bigg[ \int_0^T \Vert \nabla \er(W_t) - \nabla F(W_t) \Vert^2 dt \bigg] \\
    &= \frac{1}{2\sigma^2} \int_0^T \mathds{E}_{\rho_S} \big[\Vert \nabla \er(W_t) - \nabla F(W_t) \Vert^2 \big] dt.
    \end{aligned}
    $$
\end{proof}

Our use of Girsanov's theorem allows for more general changes of measure than what is given by \Cref{thm:girsanov_expression}. This is formalized by the following remark.

\begin{remark}[Disintegrated bounds from Girsanov's theorem]
    \label{rq:disintegrated-from-girsanov}
    In addition to providing a closed-form expression of the KL divergence appearing in the KL-based bound of \Cref{theorem:pac-bayes-set}, the proof of \Cref{thm:girsanov_expression} gives a formula for the Radon-Nykodym derivative $d \rho_S / d \pi$, through \Cref{eq:rnd_girsanov_cld}. This formula can therefore be used to perform more general changes of measure and in particular, it naturally provides an explicit form of the Radon-Nykodym derivative appearing in \Cref{eq:set_disintegrated_bound}. This is a straightforward extension of the presented theory. For the sake of simplicity, we do not discuss it in more details here.
\end{remark}

\Cref{eq:girsanov_brownian_measure} and \Cref{cor:langevin_brownian_bound} are immediate consequences of the above theorem and \Cref{theorem:data_dep_rademacher_set}, using $F=0$.
We now present the proof of \Cref{prop:kl_expected_measures_bounds}. Let us consider the setting of \Cref{sec:expected_dynamics_prior}, where $\pi$ represents the expected dynamics prior defined in \Cref{sec:langevin_setting_girsanov}. Then the following results hold. To avoid harming the readability of the paper, we only sketch this proof.

\begin{proof}
    The idea is the following: we fix some $\alpha > 0$ and apply the PAC-Bayesian bound of \Cref{theorem:existing-pac-bayes-germain} to the function $\phi: \Omega, \zcal^n \longrightarrow \R$, given by:
    \begin{align*}
        \phi(\omega, S) := \frac{1}{2\sigma^2}  \int_0^T \mathds{E}_{\rho_S} \big[\Vert \nabla \er(W_t(\omega) - \nabla F(W_t(\omega) \Vert^2 \big] dt.
    \end{align*}
    This gives us that, for any $\zeta \in (0,1)$, we have, with probability at least $1 - \zeta$ over $\datadist$:
    $$
    \alpha \mathds{E}_{\rho_S} \bigg[ \frac{1}{2\sigma^2} \int_0^T \Vert \nabla \er(W_t) - \nabla \mathcal{R}(W_t) \Vert^2 dt  \bigg] \leq \klb{\rho_S}{\pi} + \log(1/\zeta) + \log(E),
    $$
    with (keep in mind that $\sigma = \sqrt{2\beta^{-1}}$):
    $$
    E := \mathds{E}_S \mathds{E}_\pi \bigg[ \exp\bigg\{\frac{\alpha\beta}{4} \int_0^T \Vert \nabla \er(W_t) - \nabla \mathcal{R}(W_t) \Vert^2 dt \bigg\} \bigg].
    $$
    By using the previously obtained expression of the KL divergence, this can be rewritten as:
    $$
    (\alpha - 1) \klb{\rho_S}{\pi} \leq \log(1/\zeta) + \log(E).
    $$
    By Jensen's inequality, we have:
    $$
    E \leq \frac{1}{T} \int_0^T \mathds{E}_S \mathds{E}_\pi \bigg[ \exp\bigg\{ \frac{\alpha\beta}{4} T \Vert \nabla \er(W_t) - \nabla \mathcal{R}(W_t) \Vert^2\bigg\} \bigg] dt.
    $$
    A quick computation shows that:
    \begin{align*}
         \mathds{E}_S [\Vert \nabla \er(W_t) - \nabla \mathcal{R}(W_t) \Vert^2] \leq \frac{2 L^2 }{n}.
    \end{align*}
    Moreover, if $S$ and $S^i$ are two datasets of size $n$ differing by only the $i$-th element, we have, by the inverted triangle inequality:
    \begin{align*}
        \big| \Vert \nabla \er(W_t) - \nabla \mathcal{R}(W_t) \Vert^2 - \Vert \nabla \mathcal{R}_{S^i}(W_t) - \nabla \mathcal{R}(W_t) \Vert^2 \big| \leq \frac{8L^2}{n}.
    \end{align*}
    Therefore, by \Cref{lemma:exp_mc_diarmid}, we have:
    $$
    \mathds{E}_S  \bigg[ \exp\bigg\{ \frac{\alpha \beta T}{4} \big( \Vert \nabla \er(W_t) - \nabla \mathcal{R}(W_t) \Vert^2 
    - \mathds{E}_S[ \Vert \nabla \er(W_t) - \nabla \mathcal{R}(W_t) \Vert^2 ] \big) \bigg\} \bigg] \leq e^{ \alpha^2 \frac{ \beta^2 T^2 L^4}{2n}}.
    $$
    Combining these equations and applying Fubini's theorem, we get that:
    \begin{equation}
    \label{eq:alpha_optim_limiting_measure}
    (\alpha - 1) \klb{\rho_S}{\pi} \leq \log(1/\zeta) + \alpha  \frac{ L^2 \beta T}{ 2n} + \alpha^2 \frac{ \beta^2 T^2 L^4}{2n} .
    \end{equation}
    The result follows by choosing $\alpha = 2$ in the above computation.
\end{proof}

\subsubsection{Rademacher complexity of CLD - omitted proofs of \Cref{sec:rademacher_langevin}}
\label{sec:rademacher-langevin-proofs}

To end this section, we prove our bound for the Rademacher complexity of Langevin dynamics. We use the following lemma, which is taken from \citep[Exercise $2.5.10$]{vershynin2018high}.

\begin{lemma}
    \label{lemma:maximum-of-reflexions}
    On a probability space $(\Omega, \prob)$, we consider almost surely non-negative random variables $(X_1,\dots,X_N)$ (not necessarily \iid) and $\Sigma>0$ such that, for all $i$, we have:
    \begin{align*}
        \forall a \geq 0,~\Pof{X_i \geq a} \leq 2e^{-\frac{a^2}{2\Sigma}},
    \end{align*}
    then there exists an absolute constant $A>0$ such that $\Eof{\max_{1\leq i \leq N} X_i} \leq A \sqrt{\Sigma \log(N)}$.
\end{lemma}

We can now prove \Cref{thm:rademacher_complexity_langevin}.

\begin{proof}
     Let us fix $S \in \zcal^n$ and some integer $K \in\mathds{N}^\star$, let $\delta := T/K$. For $i \in \{0,\dots,K\}$, we define $t_i:= i\delta$ and suppose that $K$ is big enough so that $\delta < 1$. From \Cref{thm:girsanov_expression} and its proof, $(B_t^S)$ is a $\rho_S$ Brownian motion in $\mathds{R}^d$, we denote its coordinates (which are $\rho_S$-independent) by $B_t^S := (B_{1,t}^S, \dots,B_{d,t}^S)$. 
    We also introduce some \iid Rademacher random variables $\epsilon = (\epsilon_1,\dots,\epsilon_n)$.     
    Let\footnote{There is a slight abuse of notation here, we just want to highlight that we consider the posterior distribution $\rho_S$ on $\Omega$.} $\omega \sim \rho_S$ and $\wcal = \wcal(\omega)$, we know that:
    \begin{align*}
        \rad_S(\wcal) = \mathds{E}_\epsilon{\sup_{0\leq t \leq T} \frac{1}{n} \sum_{i=1}^n \epsilon_i \ell(W_t, z_i)} 
    \end{align*}
    Let us take $t \in [0,T]$, then there exists $j$ such that $t_j> t \geq t_{j-1}$. Therefore:
    \begin{align*}
        \frac{1}{n} \sum_{i=1}^n \epsilon_i \ell(W_t, z_i) \leq L \normof{W_t - W_{t_{j-1}}} + \max_{0 \leq j \leq K-1} \frac{1}{n} \sum_{i=1}^n \epsilon_i \ell(W_{t_j}, z_i).
    \end{align*}
    We focus on the first term of this equation, from \Cref{eq:langevin_empiric}, we have, $\rho_S$-almost surely:
    \begin{align*}
        \normof{W_t - W_{t_{j-1}}} &\leq \normof{\int_{t_{j-1}}^t \nabla \er (W_u) du}_2 + \sigma \normof{B^S_t - B^S_{t_{j-1}}}_2 \\
        &\leq L\delta + \sigma \max_{0 \leq j \leq K-1} \sup_{t_j \leq s < t_{j+1}} \normof{B^S_s - B^S_{t_j}}_2 \\
        &\leq L\delta + \sigma \max_{0 \leq j \leq K-1} \sup_{t_j \leq s < t_{j+1}} \sum_{k=1}^d |B^S_{k,s} - B^S_{k, t_j}| \\
        &\leq L\delta + \sigma \sum_{k=1}^d  \max_{0 \leq j \leq K-1} \underbrace{\sup_{t_j \leq s < t_{j+1}}  |B^S_{k,s} - B^S_{k, t_j}|}_{:= Y_{k,j}}.
    \end{align*}
Each of the coordinates of $B_t^S$ are one-dimensional standard Brownian motion. Now we fix $k$; from the strong Markov property, we know that the $Y_{k,j}$ are independent for $0 \leq j \leq K-1$. Moreover, they all have the same distribution as $Y_{k, 1}$. We can also write that:
\begin{align*}
    Y_{k, 1}  = \sup_{0 \leq t \leq \delta} |B_{1,t}^S| \leq   \sup_{0 \leq t \leq \delta} B_{1,t}^S  +   \sup_{0 \leq t \leq \delta} \left(-B_{1,t}^S \right) ,
\end{align*}
where the inequality follows from the fact that $ \sup_{0 \leq t \leq \delta} B_{1,t}^S$ and $\sup_{0 \leq t \leq \delta} (-B_{1,t}^S) $ are almost-surely positive. From the reflection principle, we have that for all $0 \leq j \leq K-1$:
\begin{align*}
     \Pof[\rho_S]{\sup_{0 \leq t \leq \delta} B_{1,t}^S \geq a} \leq 2 \Pof[\rho_S]{ B_{1,\delta}^S \geq a} \leq 2 e^{-\frac{a^2}{2 \delta}},
\end{align*}
where the last inequality follows from \citep[Equation $(2.10)$]{vershynin2018high}. It is clear that $\sup_{0 \leq t \leq \delta} (-B_{1,t}^S) $ satisfies the same inequality.
By \Cref{lemma:maximum-of-reflexions}, we have that there exists an absolute constant $C>0$, such that:
\begin{align*}
    \Eof[\rho_S]{\rad_S(\wcal)} \leq L^2\delta + CLd\sigma\sqrt{\delta} \sqrt{\log(K)} + \Eof[\rho_S]{\rad_S(\{ W_{t_0}, \dots, W_{t_{K-1}} \})}.
\end{align*}
Therefore, by Massart's lemma \citep[Lemma $26.8$]{shalevschwartz2014understanding}:
\begin{align*}
    \Eof[\rho_S]{\rad_S(\wcal)} \leq L^2\delta + \left(CLd\sigma\sqrt{\delta} + B\sqrt{\frac{2}{n}}\right) \sqrt{\log(T/\delta)}.
\end{align*}
We choose $K := \lceil T n L^2 (1 + C^2d^2\sigma^2) \rceil$, and the bound follows.
\end{proof}

\subsection{Omitted proofs of \Cref{sec:bounds_for_sgld}}
\label{sec:proofs-bounds_for_sgld}

\subsubsection{Setting details}

Let us precise our setting for SGLD. We consider a measurable space $(\Omega, \fcal)$, endowed with a filtration $\F := (\fcal_k)_{k\geq 0}$, \ie a sequence of $\sigma$-algebras, $\fcal_0 \subseteq \fcal_1 \subseteq \dots \subseteq \fcal_T$. By convention, we set $\fcal = \fcal_T$. 
We also fix a dataset $S \in \zcal^n$ and a probability distribution $\prob$ on $(\Omega, \fcal)$. We consider $W^S \in \Rd$ satisfying \Cref{eq:sgld}, \ie,
\begin{align}
    \label{eq:sgld_ws}
    \forall k \in \mathds{N}, ~ W^S_{k+1} = W^S_k - \eta_{k+1} \hat{g}_{k+1} + \sigma_{k+1} \epsilon_{k+1}, \quad \epsilon_{k+1} \sim \mathcal{N}(0,I_d).
\end{align}

Note that the dependence of $W^S$ on $S$ comes from the fact that $\hat{g}_{k+1}$ is an unbiased estimate of $\nabla \er(W_k)$. We assume that:
\begin{itemize}[noitemsep,leftmargin=.2in]
    \item $(\epsilon_k)_{k\geq 1}$ are adapted to $\F$, \ie $\epsilon_k$ is measurable with respect to the $\sigma$-algebra $\mathcal{F}_k$, and are \iid with distribution $\mathcal{N}(0, I_d)$.
    \item $(\hat{g}_k)_{k\geq 1}$ are adapted to $\F$.
    \item For $k \geq 1$, $\epsilon_k$ is independent of the following $\sigma$-algebra $ \Tilde{\fcal_k} := \sigma \left( \sigma(\hat{g}_k) \cup \fcal_{k-1} \right)$.
\end{itemize}
Note that this also implies that $(W_k)_{k\geq 1}$ is adapted to $\F$. To simplify our final bounds, we assume that $\ghat_{k+1}$ has the form:
\begin{align}
    \label{eq:sgld_U_def}
    \ghat_{k+1} = G(W_k^S, S, U_{k+1}),
\end{align}
where $U_{k+1}$ denotes a sequence of \iid random variables which adapted to $\F$ and independent from $\epsilon_{k+1}$ and $W_k$. For example, $U_{k+1}$ may denote a random set of indices of $\set{1,\dots,n}$, with size $b \in \mathds{N}^\star$, in which case $ \ghat_{k+1} = \frac{1}{b} \sum_{i \in U_{k+1}} \nabla \ell(W_k, z_i)$. This is similar to settings adopted in other works \citep{mou2018generalization, negrea2019information}.

\subsubsection{Proof of \Cref{thm:bound_sgld_normal_walk_prior}}

To prove \Cref{thm:bound_sgld_normal_walk_prior}, we are going to write lemmas acting as a discrete version of Girsanov's theorem and Lévy's characterization theorem. While those are classical arguments, we provide proof for the sake of completeness.

Let us first make the following remark regarding this proof technique.

\begin{remark}[Generality of the Girsanov approach]
    \label{rq:girsanov-generality-discrete}
    To bound the KL divergence appearing in \Cref{theorem:pac-bayes-set} for SGLD and obtain \Cref{thm:bound_sgld_normal_walk_prior}, it is not necessary to prove a discrete version of Girsanov theorem. However, we describe this proof to highlight the similarity with the use of Girsanov's theorem in \Cref{sec:gen_bounds_cld} and to obtain a more general change of measure. Indeed, our approach does not only provide a bound on the KL divergence, it gives a closed-form expression of the Radon-Nykodym derivative between the posterior and prior distributions. This is similar to what was observed in \Cref{rq:disintegrated-from-girsanov} for the continuous case. Therefore, our framework can be used to prove disintegrated uniform generalization bounds for SGLD. If one is only interested in the KL divergence, it can simply be obtained by decomposing the joint distribution of the path of SGLD into a product of conditional distributions, we leave the details to the reader.
\end{remark}

Let us introduce the following random variable:
\begin{align}
    \label{eq:z_def_discrete_sgld}
     \forall N \geq 1,~Z_N := \prod_{k=1}^N e^{E_N}, \quad \text{with: } E_N := \frac{\eta_k}{\sigma_k} \scal{\ghat_k}{\epsilon_k} - \frac{\eta_k^2}{2 \sigma_k^2} \normof{\ghat_k}^2.
\end{align}

For our method to work, we need the following assumption, which is, in particular, true if the stochastic gradients $\ghat_k$ are almost surely bounded.

\begin{assumption}
    \label{ass:integrability_of_z_sgld}
    The random variables $W^S_k$ and $Z_k$ are square integrable, \ie, in $L^2(\prob)$.
\end{assumption}

The proof is detailed through several lemmas, which we will now present.
\begin{lemma}
    \label{lemma:z_martingale}
    $(Z_k)_{k\leq 1}$ is a $\prob$-martingale, with respect to $\F$, where, as a reminder, $\mathds{F}$ denotes the filtration $\mathcal{F}_0 \subseteq \dots \subseteq \mathcal{F}_T$.
\end{lemma}

\begin{proof}
    We fix $N\geq 1$. As $\epsilon_N$ is independent of $\Tilde{\mathcal{F}}_N$ (defined above), we have:
    \begin{align*}
        \condexp{Z_N}{\fcal_{N-1}} = Z_{N-1}\condexp{e^{E_N}}{\fcal_{N-1}} = Z_{N-1}\condexp{e^{-\frac{\eta_N^2}{2\sigma_N^2}\normof{\ghat_N}^2}\condexp{e^{\frac{\eta_N}{\sigma_N} \scal{\ghat_N}{\epsilon_N}} }{\Tilde{\mathcal{F}}_N}}{\fcal_{N-1}}.
    \end{align*}
    From the formula for the Moment Generating Function (MGF) of multivariate Gaussian distributions, we deduce that $\condexp{Z_N}{\fcal_{N-1}} = Z_{N-1}$.
\end{proof}
\begin{lemma}
    \label{lemma:wz_martingale}
    $(W^S_k Z_k)_{k\leq 1}$ is a $\prob$-martingale, with respect to $\F$.
\end{lemma}
\begin{proof}
     We fix $N\geq 1$. With similar arguments to the previous proof and using the definition of $W^S$, we have:
     \begin{align*}
         \condexp{W^S_N Z_N}{\fcal_{N-1}} &= W^S_{N-1}  \condexp{ Z_N}{\fcal_{N-1}} - \eta_N \condexp{Z_N \ghat_N}{\fcal_{N-1}} + \sigma_N \condexp{Z_N \epsilon_N}{\fcal_{N-1}}\\
         &= W^S_{N-1} Z_{N-1} - \eta_N Z_{N-1}\condexp{e^{E_N} \ghat_N}{\fcal_{N-1}} + \sigma_N Z_{N-1}\condexp{e^{E_N} \epsilon_N}{\fcal_{N-1}}.
     \end{align*}
     But we also compute separately:
     \begin{align*}
         \condexp{e^{E_N} \ghat_N}{\fcal_{N-1}} &= \condexp{\ghat_N e^{-\frac{\eta_N^2}{2\sigma_N^2}\normof{\ghat_N}^2}\condexp{e^{\frac{\eta_N}{\sigma_N} \scal{\ghat_N}{\epsilon_N}} }{\Tilde{\mathcal{F}}_N}}{\fcal_{N-1}} = \condexp{ \ghat_N}{\fcal_{N-1}},\\
        \condexp{e^{E_N} \epsilon_N}{\fcal_{N-1}} &=  \condexp{e^{-\frac{\eta_N^2}{2\sigma_N^2}\normof{\ghat_N}^2}\condexp{e^{\frac{\eta_N}{\sigma_N} \scal{\ghat_N}{\epsilon_N}} \epsilon_N}{\Tilde{\mathcal{F}}_N}}{\fcal_{N-1}} = \frac{\eta_N}{\sigma_N}\condexp{ \ghat_N}{\fcal_{N-1}},
     \end{align*}
     where the last line follows from $\Eof[X \sim \mathcal{N}(0,I_d)]{e^{\scal{u}{X}}X} = e^{\normof{u}^2/2}u$. The result follows.
\end{proof}

We define a new probability measure on $(\Omega, \fcal)$ by $\frac{d \q}{d \prob} = Z_T$.
Thanks to \Cref{lemma:z_martingale}, we have $ \frac{d \q|_{\fcal_N}}{d \prob|_{\fcal_N}} = Z_N$, for $N\leq T$. By construction, it follows from \Cref{lemma:wz_martingale} that $(W_k^S)_{k\geq1}$ is a $\q$-martingale with respect to $\F$. Now we define:
\begin{align}
    \label{eq:y_discrete_case}
    Y_N := \sum_{k=1}^N \epsilon_k - \sum_{k=1}^N \frac{\eta_k}{\sigma_k} \ghat_k,
\end{align}
and by convention $Y_0 = 0$.
The following lemma is now clear from $Y_N - Y_{N-1} = \frac{W^S_N - W^S_{N-1}}{\sigma_N}$.

\begin{lemma}
    \label{lemma:y_martingale}
    $(Y_k)_k$ is a $\mathds{Q}$-martingale, with respect to $\F$.
\end{lemma}

The most important lemma is the following, it shows that, under $\mathds{Q}$, $W^S$ is a data-independent normal random walk.

\begin{restatable}{lemma}{LemmaNormalrandomWalk}
    \label{lemma:normal_random_walk}
    The variables $(Y_k - Y_{k-1})_{k\geq 1}$ are, under $\q$, independent and identically distributed with distribution $\mathcal{N}(0,I_d)$.
\end{restatable}

\begin{proof}
    Inspired by the proof of Lévy's theorem for the characterization of Brownian motion from its quadratic variation, we prove the lemma by computing Characteristic Functions (CF). To achieve this, let $J \subseteq \set{1,\dots, T}$ be an arbitrary set of indices. Let $M := |J|$ and $j_0 := \max(J)$. We denote $\Delta_k := Y_k - Y_{k-1}$ and compute the following CF, for $u \in (\Rd)^J$:
    \begin{align}
        \label{eq:normal_walk_proof_step_1}
        \Eof[\q]{e^{i \sum_{j \in J} \scal{u_j}{\Delta_j} }} &= \Eof[\q]{e^{i \sum_{j \in J\backslash\set{j_0}}  \scal{u_j}{\Delta_j}} \condexp[\q]{e^{i  \scal{u_{j_0}}{\Delta_{j_0}} }}{\Tilde{\fcal}_{j_0}} } \\
         &= \Eof[\q]{e^{i \sum_{j \in J\backslash\set{j_0}}  \scal{u_j}{\Delta_j}} e^{-i \frac{\eta_{j_0}}{\sigma_{j_0}}\scal{\ghat_{j_0}}{u_{j_0}}}\condexp[\q]{e^{i  \scal{u_{j_0}}{\epsilon_{j_0}} }}{\Tilde{\fcal}_{j_0}} } .
    \end{align}
    For any $N \leq T$, we compute, from the definition of $Z_N$, for any $A \in \Tilde{\fcal}_{N}$:
    \begin{align*}
        \Eof[\q]{e^{i  \scal{u_N}{\epsilon_N } } \mathds{1}_A } &= \Eof[\prob]{ e^{i  \scal{u_N}{\epsilon_N }} e^{E_N} Z_{N-1} \mathds{1}_A } \\
         &= \Eof[\prob]{ Z_{N-1} \mathds{1}_A e^{-\frac{\eta_N^2}{2 \sigma_N^2} \normof{\ghat_N}^2} \condexp[\prob]{e^{ \frac{\eta_N}{\sigma_N} \scal{\ghat_N}{\epsilon_N} + i  \scal{u_N}{\epsilon_N }} }{\Tilde{\fcal}_{N}}} \\
         &= \Eof[\prob]{ Z_{N-1} \mathds{1}_A e^{i \frac{\eta_{N}}{\sigma_{N}}\scal{\ghat_{N}}{u_{N}}} e^{-\frac{\normof{u_N}^2}{2}}}, 
    \end{align*}
    where we used $\Eof[X \sim \mathcal{N}(0,I_d)]{e^{\scal{a + ib}{X}}} = e^{\frac{\normof{a}^2}{2} + i \scal{a}{b} - \frac{\normof{b}^2}{2}}$, for $a,b \in\Rd$. Hence 
    $$
    \condexp[\q]{e^{i  \scal{u_N}{\epsilon_N }}}{\Tilde{\fcal}_{N}} = e^{-\frac{\normof{u_N}^2}{2}} e^{i \frac{\eta_{N}}{\sigma_{N}}\scal{\ghat_{N}}{u_{N}}}.
    $$
    Therefore, we deduce that:
    \begin{align*}
         \Eof[\q]{e^{i \sum_{j \in J} \scal{u_j}{\Delta_j} }} =  e^{-\frac{\normof{u_{j_0}}^2}{2}} \Eof[\q]{e^{i \sum_{j \in J\backslash\set{j_0}}  \scal{u_j}{\Delta_j}} }. 
    \end{align*}
    The result is implied by an immediate recursion (on $M = |J|$) and identifying the CF of multivariate Gaussian distributions. 
\end{proof}

This shows that, under $\q$, $W$ follows the dynamics, $W_{k+1} = W_k + \sigma_{k+1} \mathcal{N}(0,I_d) $, with independent realizations of $\mathcal{N}(0,I_d)$ at each iterations. Let us denote by $F_T(\Rd)$ the set of finite subsets of $\Rd$ with cardinality $T$. With a slight abuse of notation, we see $W^S$ as a map $W^S: \Omega \longrightarrow F_T(\Rd)$.
We define the posterior and prior distributions on $F_T(\Rd)$ by:
\begin{align}
    \label{eq:sgld_prio_posterior}
    \rho_S = W^S_\# \prob, \quad \pi = W^S_\# \q.
\end{align}

\begin{remark}
    $\rho_S$ is data-dependent by definition of the dynamics. While $\q$ may depend on the data $S$, the pushforward $W^S_\# \q$ is not data-dependent, as it corresponds to a data-independent dynamics, \eg, $W_{k+1} = W_k + \sigma_{k+1} \mathcal{N}(0,I_d)$.
\end{remark}

We can now prove the main result of \Cref{sec:bounds_for_sgld}, \ie \Cref{thm:bound_sgld_normal_walk_prior}.

\begin{proof}
    We apply \Cref{theorem:mgf_rademacher} and, as the random sets drawn from $\rho_S$ and $\pi$ are surely of cardinal $T < +\infty$, we apply the reasoning of \Cref{example:almost_surely_finite_set} to get:
    \begin{align*}
        \Eof[\rho_S]{\max_{w \in \wcal} \left( \risk(w) - \er(w) \right)} \leq  \frac{1}{\lambda} \left( \log(T/\zeta) + \klb{\rho_S}{\pi} \right)
        + \lambda \frac{2B^2}{n}.
    \end{align*}
    By the data processing inequality, we know that we have, for a fixed $S \in \zcal^n$, $\klb{\rho_S}{\pi} \leq \klb{\prob}{\q}$, where $\prob$ and $\q$ correspond to the notations above (the dependence on $S$ is implicit here). Using the definition of $Z_T$, we easily compute, from the fact that $\epsilon_k$ is independent of $\ghat_k$:
    \begin{align*}
        \klb{\prob}{\q} = - \int \log(Z_T) d\prob = \frac{1}{2} \sum_{k=1}^T \frac{\eta_k^2}{\sigma_k^2} \Eof[\prob]{\normof{\ghat_k}^2} = \frac{\beta}{4} \sum_{k=1}^T \eta_k \Eof[\prob]{\normof{\ghat_k}^2}.
    \end{align*}
\end{proof}

\subsection{Random closed sets formalization: omitted proofs}
\label{sec:proofs-random-closed-sets-construction}

In this section, we prove the measurability results related to the general measure-theoretic construction that we propose in \Cref{sec:random_closed_sets_construction}.
This is essential to provide strong theoretical foundations for the introduced techniques.

As already mentioned, the reader may refer to \citep{molchanov2017theory} for a more detailed introduction to the theory of random (closed) sets.
The Effrös $\sigma$-algebra was defined in \Cref{def:random_closed_set}, we slightly refine its definition in the following proposition, which summarizes results from Propositions $1.1.1$ and $1.1.1'$ of \citep{molchanov2017theory}.

\begin{proposition}
\label{prop:effros-generation}
    The $\sigma$-algebra $\mathfrak{E}(\Rd)$ is generated by both of the following family of sets:
    \begin{align*}
        \set{\mathcal{F}_K,~ K \subset \Rd \text{ compact}}, \quad \text{and: }  \set{\mathcal{F}_U,~ U \subseteq \Rd \text{ open}},
    \end{align*}
    where $\mathcal{F}_A := \set{C \in \closed, C\cap A \neq \emptyset}$.
\end{proposition}

We now prove \Cref{lemma:measurability_for_closed_sets}.

\newcommand{\Qpos}{\mathds{Q}_{>0}}
\newcommand{\Qd}{\mathds{Q}^d}

\begin{proof}
    It is enough to prove that, for each $\forall t \in \R,~\set{\Phi > t} \in \mathfrak{E}(\Rd) \otimes \mathcal{T}$. We remark that
   $$
   \begin{aligned}
        \Phi(\wcal, \omega) > t &\iff \exists w \in \wcal, ~\exists \eta \in \Qpos,~\zeta(w,\omega) \geq t + \eta \\
        &\iff \exists q \in \Qd, ~\exists \epsilon,\eta\in \Qpos,~\begin{cases} & \forall q' \in \Qd\cap B(q,\epsilon), ~\zeta(q',\omega) \geq t + \eta \\
        &\wcal \cap B(q,\epsilon) \neq \emptyset,\end{cases}
        \end{aligned}
    $$
    hence
    \begin{align*}
        \set{\Phi(\wcal, \omega) > t} = \bigcup_{q \in \Qd, \epsilon,\eta\in \Qpos} \fcal_{B(q,\epsilon)} \cap \bigcap_{q' \in \Qd\cap B(q,\epsilon)} \closed \times \set{x,~ \zeta(q',\omega) \geq t + \eta},
    \end{align*}
    where, implicitely, we see  $\fcal_{B(q,\epsilon)} $ as $\fcal_{B(q,\epsilon)} \times \Omega$.
    The above is in $\mathfrak{E}(\Rd) \otimes \mathcal{T}$ by countable unions, intersections, and using the previous proposition.
\end{proof}

\subsubsection{Covering numbers measurability}
\label{sec:covering_numbers_measurability}

In this subsection, we justify that the random closed sets formalization also implies the measurability of the covering numbers used in \Cref{sec:fractal_bounds} (and therefore of the fractal dimensions).
Without loss of generality, we can only consider ``rational'' covering numbers, which we define as:

\newcommand{\rationalcov}{N_\delta^{\mathds{Q}}}

\begin{definition}[Rational covering numbers]
    Let $X\subset \Rd$ be a $\rho$-bounded set and $\delta > 0$. Then $\rationalcov \subset \Qd$ is a minimal set of points, in $\Qd$, such that $X \subset \bigcup_{w \in \rationalcov} \bar{B}_\delta (w)$.
\end{definition}

It is clear that, with the notations of \Cref{sec:covering_bounds}, we have $|N_{2\delta}^{\mathds{Q}}(X)| \leq |N_\delta(X)| \leq |N_\delta^{\mathds{Q}}(X)|$.

Hence, up to a potential small absolute constant, we do not modify the bounds presented in \Cref{sec:fractal_bounds} by considering rational covering numbers in place of the ones used in \Cref{sec:covering_bounds}. Moreover, the above inequalities imply that both notions of covering yield the same upper box-counting dimension. This leads to the following lemma.
\begin{lemma}
    \label{lemma:measurability_eucl_covering}
    We extend the definition of both covering numbers to be $+\infty$ on unbounded closed sets.
    Then the covering number $|\rationalcov(X)|$ is measurable with respect to $\mathfrak{E}(\Rd)$.
\end{lemma}

\begin{proof}
    Let $N \in \mathds{N}^\star$ and $\delta > 0$, we just remark that:
    \begin{align*}
        \set{\wcal, ~|\rationalcov(\wcal)| > N} = \bigcap_{m=1}^{N}~\bigcap_{w_1,\dots,w_m \in \Qd} \set{\wcal,~ \wcal \backslash \bigcup_{i=1}^m \bar{B}_\delta(w_i) \neq \emptyset},
    \end{align*}
    which implies the desired measurability.
\end{proof}

For the data-dependent covering numbers, induced by pseudometric $\vartheta_S$, see \Cref{eq:data_dep_pseudo_metric}, and used in \Cref{sec:covering_bounds}, we can perform similar reasoning and invoke \Cref{lemma:measurability_for_closed_sets} to conclude that the rational covering numbers, associated to pseudometric $\vartheta_S$, are measurable with respect to $\mathfrak{E}(\Rd) \otimes \fcal^{\otimes n}$, as soon as the loss $\ell(w,z)$ is continuous.

\vskip 0.2in
\bibliography{main}

\begin{thebibliography}{83}
\providecommand{\natexlab}[1]{#1}
\providecommand{\url}[1]{\texttt{#1}}
\expandafter\ifx\csname urlstyle\endcsname\relax
  \providecommand{\doi}[1]{doi: #1}\else
  \providecommand{\doi}{doi: \begingroup \urlstyle{rm}\Url}\fi

\bibitem[Alquier(2024)]{alquier2024user}
Pierre Alquier.
\newblock {User-friendly Introduction to PAC-Bayes Bounds}.
\newblock \emph{Foundations and Trends® in Machine Learning}, 2024.

\bibitem[Ambroladze et~al.(2006)Ambroladze, Parrado{-}Hern{\'{a}}ndez, and Shawe{-}Taylor]{ambroladze2006tighter}
Amiran Ambroladze, Emilio Parrado{-}Hern{\'{a}}ndez, and John Shawe{-}Taylor.
\newblock {Tighter PAC-Bayes Bounds}.
\newblock In \emph{Advances in Neural Information Processing Systems (NIPS)}, 2006.

\bibitem[Amir et~al.(2022)Amir, Livni, and Srebro]{amir2022thinking}
Idan Amir, Roi Livni, and Nati Srebro.
\newblock {Thinking Outside the Ball: Optimal Learning with Gradient Descent for Generalized Linear Stochastic Convex Optimization}.
\newblock In \emph{Advances in Neural Information Processing Systems (NeurIPS)}, 2022.

\bibitem[Amit et~al.(2022)Amit, Epstein, Moran, and Meir]{amit2022integral}
Ron Amit, Baruch Epstein, Shay Moran, and Ron Meir.
\newblock {Integral Probability Metrics PAC-Bayes Bounds}.
\newblock In \emph{Advances in Neural Information Processing Systems (NeurIPS)}, 2022.

\bibitem[Andreeva et~al.(2023)Andreeva, Limbeck, Rieck, and Sarkar]{andreeva2023metric}
Rayna Andreeva, Katharina Limbeck, Bastian Rieck, and Rik Sarkar.
\newblock {Metric Space Magnitude and Generalisation in Neural Networks}.
\newblock In \emph{ICML 2023 Workshop on Topological, Algebraic and Geometric Learning}, 2023.

\bibitem[Andreeva et~al.(2024)Andreeva, Dupuis, Sarkar, Birdal, and {\c S}im{\c s}ekli]{andreeva2024topological}
Rayna Andreeva, Benjamin Dupuis, Rik Sarkar, Tolga Birdal, and Umut {\c S}im{\c s}ekli.
\newblock {Topological Generalization Bounds for Discrete-Time Stochastic Optimization Algorithms}.
\newblock In \emph{Advances in Neural Information Processing Systems (NeurIPS)}, 2024.

\bibitem[Aristoff(2012)]{aristoff2012estimating}
David Aristoff.
\newblock {Estimating Small Probabilities for Langevin Dynamics}.
\newblock \emph{arXiv}, abs/1205.2400, 2012.

\bibitem[Awasthi et~al.(2020)Awasthi, Frank, and Mohri]{awasthi2020rademacher}
Pranjal Awasthi, Natalie Frank, and Mehryar Mohri.
\newblock {On the Rademacher Complexity of Linear Hypothesis Sets}.
\newblock \emph{arXiv}, abs/2007.11045, 2020.

\bibitem[Bartlett and Mendelson(2002)]{bartlett2002rademacher}
Peter Bartlett and Shahar Mendelson.
\newblock {Rademacher and Gaussian Complexities: Risk Bounds and Structural Results}.
\newblock \emph{Journal of Machine Learning Research}, 2002.

\bibitem[Bartlett et~al.(2002)Bartlett, Boucheron, and Lugosi]{bartlett2002model}
Peter Bartlett, St{\'e}phane Boucheron, and G{\'a}bor Lugosi.
\newblock {Model Selection and Error Estimation}.
\newblock \emph{Machine Learning}, 2002.

\bibitem[Bartlett et~al.(2017)Bartlett, Foster, and Telgarsky]{bartlett2017spectrally}
Peter Bartlett, Dylan Foster, and Matus Telgarsky.
\newblock {Spectrally-normalized margin bounds for neural networks}.
\newblock In \emph{Advances in Neural Information Processing Systems (NIPS)}, 2017.

\bibitem[Birdal et~al.(2021)Birdal, Lou, Guibas, and {\c S}im{\c s}ekli]{birdal2021intrinsic}
Tolga Birdal, Aaron Lou, Leonidas Guibas, and Umut {\c S}im{\c s}ekli.
\newblock {Intrinsic Dimension, Persistent Homology and Generalization in Neural Networks}.
\newblock \emph{Advances in Neural Information Processing Systems (NeurIPS)}, 2021.

\bibitem[Blanchard and Fleuret(2007)]{blanchard2007occam}
Gilles Blanchard and Fran{\c{c}}ois Fleuret.
\newblock {Occam's Hammer}.
\newblock In \emph{Conference on Learning Theory (COLT)}, 2007.

\bibitem[Bogachev(2007)]{bogachev2007measure}
Vladimir Bogachev.
\newblock \emph{{Measure theory}}.
\newblock Springer, 2007.

\bibitem[Boissonnat et~al.(2018)Boissonnat, Chazal, and Yvinec]{boissonnat2018geometric}
Jean-Daniel Boissonnat, Fr{\'e}d{\'e}ric Chazal, and Mariette Yvinec.
\newblock \emph{{Geometric and Topological Inference}}.
\newblock Cambridge University Press, 2018.

\bibitem[Boucheron et~al.(2013)Boucheron, Lugosi, and Massart]{boucheron2013concentration}
St{\'e}phane Boucheron, Gabor Lugosi, and Pascal Massart.
\newblock \emph{{Concentration Inequalities - A Non-asymptotic Theory of Independence}}.
\newblock Oxford University Press, 2013.

\bibitem[Camuto et~al.(2021)Camuto, Deligiannidis, Erdogdu, G{\"{u}}rb{\"{u}}zbalaban, {\c S}im{\c s}ekli, and Zhu]{camuto2021fractal}
Alexander Camuto, George Deligiannidis, Murat Erdogdu, Mert G{\"{u}}rb{\"{u}}zbalaban, Umut {\c S}im{\c s}ekli, and Lingjiong Zhu.
\newblock {Fractal Structure and Generalization Properties of Stochastic Optimization Algorithms}.
\newblock In \emph{Advances in Neural Information Processing Systems (NeurIPS)}, 2021.

\bibitem[Catoni(2007)]{catoni2007pac}
Olivier Catoni.
\newblock \emph{{Pac-Bayesian supervised classification: the thermodynamics of statistical learning}}.
\newblock Institute of Mathematical Statistics, 2007.

\bibitem[Chugg et~al.(2023)Chugg, Wang, and Ramdas]{chugg2023unified}
Ben Chugg, Hongjian Wang, and Aaditya Ramdas.
\newblock {A Unified Recipe for Deriving (Time-Uniform) PAC-Bayes Bounds}.
\newblock \emph{Journal of Machine Learning Research}, 2023.

\bibitem[Dalalyan(2017)]{dalalyan2017theoretical}
Arnak Dalalyan.
\newblock {Theoretical guarantees for approximate sampling from smooth and log-concave densities}.
\newblock \emph{Journal of the Royal Statistical Society Series B: Statistical Methodology}, 2017.

\bibitem[Dupuis and {\c S}im{\c s}ekli(2024)]{dupuis2024generalization}
Benjamin Dupuis and Umut {\c S}im{\c s}ekli.
\newblock {Generalization Bounds for Heavy-Tailed SDEs through the Fractional Fokker-Planck Equation}.
\newblock In \emph{International Conference on Machine Learning (ICML)}, 2024.

\bibitem[Dupuis and Viallard(2023)]{dupuis2023from}
Benjamin Dupuis and Paul Viallard.
\newblock {From Mutual Information to Expected Dynamics: New Generalization Bounds for Heavy-Tailed SGD}.
\newblock In \emph{NeurIPS 2023 Workshop Heavy Tails in Machine Learning}, 2023.

\bibitem[Dupuis et~al.(2023)Dupuis, Deligiannidis, and Şimşekli]{dupuis2023generalization}
Benjamin Dupuis, George Deligiannidis, and Umut Şimşekli.
\newblock {Generalization Bounds with Data-dependent Fractal Dimensions}.
\newblock In \emph{International Conference on Machine Learning (ICML)}, 2023.

\bibitem[Dwork et~al.(2015)Dwork, Feldman, Hardt, Pitassi, Reingold, and Roth]{dwork2015generalization}
Cynthia Dwork, Vitaly Feldman, Moritz Hardt, Toniann Pitassi, Omer Reingold, and Aaron Roth.
\newblock {Generalization in Adaptive Data Analysis and Holdout Reuse}.
\newblock In \emph{Advances in Neural Information Processing Systems (NIPS)}, 2015.

\bibitem[Dziugaite and Roy(2017)]{dziugaite2017computing}
Gintare~Karolina Dziugaite and Daniel Roy.
\newblock {Computing Nonvacuous Generalization Bounds for Deep (Stochastic) Neural Networks with Many More Parameters than Training Data}.
\newblock In \emph{Conference on Uncertainty in Artificial Intelligence (UAI)}, 2017.

\bibitem[Falconer(2014)]{falconer2014fractal}
Kenneth Falconer.
\newblock \emph{{Fractal Geometry - Mathematical Foundations and Applications}}.
\newblock Wiley, 2014.

\bibitem[Farghly and Rebeschini(2021)]{farghly2021time}
Tyler Farghly and Patrick Rebeschini.
\newblock {Time-independent Generalization Bounds for SGLD in Non-convex Settings}.
\newblock In \emph{Advances in Neural Information Processing Systems (NeurIPS)}, 2021.

\bibitem[Foster et~al.(2019)Foster, Greenberg, Kale, Luo, Mohri, and Sridharan]{foster2019hypothesis}
Dylan Foster, Spencer Greenberg, Satyen Kale, Haipeng Luo, Mehryar Mohri, and Karthik Sridharan.
\newblock {Hypothesis Set Stability and Generalization}.
\newblock In \emph{Advances in Neural Information Processing Systems (NeurIPS)}, 2019.

\bibitem[Futami and Fujisawa(2023)]{futami2023time}
Futoshi Futami and Masahiro Fujisawa.
\newblock {Time-Independent Information-Theoretic Generalization Bounds for SGLD}.
\newblock In \emph{Advances in Neural Information Processing Systems (NeurIPS)}, 2023.

\bibitem[Germain et~al.(2009)Germain, Lacasse, Laviolette, and Marchand]{germain2009pac}
Pascal Germain, Alexandre Lacasse, Fran{\c{c}}ois Laviolette, and Mario Marchand.
\newblock {PAC-Bayesian Learning of Linear Classifiers}.
\newblock In \emph{International Conference on Machine Learning (ICML)}, 2009.

\bibitem[Germain et~al.(2015)Germain, Lacasse, Laviolette, Marchand, and Roy]{germain2015risk}
Pascal Germain, Alexandre Lacasse, Fran{\c{c}}ois Laviolette, Mario Marchand, and Jean{-}Francis Roy.
\newblock {Risk Bounds for the Majority Vote: From a PAC-Bayesian Analysis to a Learning Algorithm}.
\newblock \emph{Journal of Machine Learning Research}, 2015.

\bibitem[Gr{\"{u}}nwald and Mehta(2019)]{grunwald2019tight}
Peter~D. Gr{\"{u}}nwald and Nishant~A. Mehta.
\newblock {A Tight Excess Risk Bound via a Unified PAC-Bayesian–Rademacher–Shtarkov–MDL Complexity}.
\newblock In \emph{Algorithmic Learning Theory (ALT)}, 2019.

\bibitem[G{\"{u}}rb{\"{u}}zbalaban et~al.(2021)G{\"{u}}rb{\"{u}}zbalaban, {\c S}im{\c s}ekli, and Zhu]{gurbuzbalaban2021heavy}
Mert G{\"{u}}rb{\"{u}}zbalaban, Umut {\c S}im{\c s}ekli, and Lingjiong Zhu.
\newblock {The Heavy-Tail Phenomenon in SGD}.
\newblock In \emph{International Conference on Machine Learning (ICML)}, 2021.

\bibitem[Haghifam et~al.(2020)Haghifam, Negrea, Khisti, Roy, and Dziugaite]{haghifam2020sharpened}
Mahdi Haghifam, Jeffrey Negrea, Ashish Khisti, Daniel Roy, and Gintare~Karolina Dziugaite.
\newblock {Sharpened Generalization Bounds Based on Conditional Mutual Information and an Application to Noisy, Iterative Algorithms}.
\newblock In \emph{Advances in Neural Information Processing Systems (NeurIPS)}, 2020.

\bibitem[Herbrich and Graepel(2000)]{herbrich2000pac}
Ralf Herbrich and Thore Graepel.
\newblock {A PAC-Bayesian Margin Bound for Linear Classifiers: Why SVMs work}.
\newblock In \emph{Advances in Neural Information Processing Systems (NIPS)}, 2000.

\bibitem[Hodgkinson et~al.(2022)Hodgkinson, {\c S}im{\c s}ekli, Khanna, and Mahoney]{hodgkinson2022generalization}
Liam Hodgkinson, Umut {\c S}im{\c s}ekli, Rajiv Khanna, and Michael Mahoney.
\newblock {Generalization Bounds Using Lower Tail Exponents in Stochastic Optimizers}.
\newblock In \emph{International Conference on Machine Learning (ICML)}, 2022.

\bibitem[Kakade et~al.(2008)Kakade, Sridharan, and Tewari]{kakade2008complexity}
Sham Kakade, Karthik Sridharan, and Ambuj Tewari.
\newblock {On the Complexity of Linear Prediction: Risk Bounds, Margin Bounds, and Regularization}.
\newblock In \emph{Advances in Neural Information Processing Systems (NIPS)}, 2008.

\bibitem[Koltchinskii(2001)]{koltchinskii2001rademacher}
Vladimir Koltchinskii.
\newblock {Rademacher penalties and structural risk minimization}.
\newblock \emph{IEEE Transactions on Information Theory}, 2001.

\bibitem[Koltchinskii and Panchenko(2004)]{koltchinskii2004rademacher}
Vladimir Koltchinskii and Dmitriy Panchenko.
\newblock {Rademacher processes and bounding the risk of function learning}.
\newblock \emph{arXiv}, math/0405338, 2004.

\bibitem[Koltchinskii and Panchenko(2002)]{koltchinskii2002empirical}
Vladimir Koltchinskii and Dmitry Panchenko.
\newblock {Empirical margin distributions and bounding the generalization error of combined classifiers}.
\newblock \emph{The Annals of Statistics}, 2002.

\bibitem[Lacasse et~al.(2006)Lacasse, Laviolette, Marchand, Germain, and Usunier]{lacasse2006pac}
Alexandre Lacasse, Fran{\c{c}}ois Laviolette, Mario Marchand, Pascal Germain, and Nicolas Usunier.
\newblock {PAC-Bayes Bounds for the Risk of the Majority Vote and the Variance of the Gibbs Classifier}.
\newblock In \emph{Advances in Neural Information Processing Systems (NIPS)}, 2006.

\bibitem[Langford(2005)]{langford2005tutorial}
John Langford.
\newblock {Tutorial on Practical Prediction Theory for Classification}.
\newblock \emph{Journal of Machine Learning Research}, 2005.

\bibitem[Langford and Caruana(2001)]{langford2001not}
John Langford and Rich Caruana.
\newblock {(Not) Bounding the True Error}.
\newblock In \emph{Advances in Neural Information Processing Systems (NIPS)}, 2001.

\bibitem[Langford and Shawe{-}Taylor(2002)]{langford2002pac}
John Langford and John Shawe{-}Taylor.
\newblock {PAC-Bayes \& Margins}.
\newblock In \emph{Advances in Neural Information Processing Systems (NIPS)}, 2002.

\bibitem[Li et~al.(2020)Li, Luo, and Qiao]{li2020generalization}
Jian Li, Xuanyuan Luo, and Mingda Qiao.
\newblock {On Generalization Error Bounds of Noisy Gradient Methods for Non-Convex Learning}.
\newblock In \emph{International Conference on Learning Representations (ICLR)}, 2020.

\bibitem[London(2017)]{london2017pac}
Ben London.
\newblock {A PAC-Bayesian Analysis of Randomized Learning with Application to Stochastic Gradient Descent}.
\newblock In \emph{Advances in Neural Information Processing Systems (NIPS)}, 2017.

\bibitem[Mackay and Tyson(2010)]{mackay2010conformal}
John Mackay and Jeremy Tyson.
\newblock \emph{{Conformal Dimension: Theory and Application}}.
\newblock American Mathematical Society, 2010.

\bibitem[Mattila(1999)]{mattila1999geometry}
Pertti Mattila.
\newblock \emph{{Geometry of Sets and Measures in Euclidean Spaces}}.
\newblock Cambridge University Press, 1999.

\bibitem[Maurer(2004)]{maurer2004note}
Andreas Maurer.
\newblock {A Note on the PAC Bayesian Theorem}.
\newblock \emph{arXiv}, cs.LG/0411099, 2004.

\bibitem[McAllester(1998)]{mcallester1998some}
David McAllester.
\newblock {Some PAC-Bayesian Theorems}.
\newblock In \emph{Conference on Computational Learning Theory (COLT)}, 1998.

\bibitem[McAllester(2003)]{mcallester2003pac}
David McAllester.
\newblock {PAC-Bayesian Stochastic Model Selection}.
\newblock \emph{Machine Learning}, 2003.

\bibitem[McDiarmid(1998)]{mcdiarmid1998concentration}
Colin McDiarmid.
\newblock {Concentration}.
\newblock In \emph{Probabilistic methods for algorithmic discrete mathematics}. Springer, 1998.

\bibitem[Mohri et~al.(2018)Mohri, Rostamizadeh, and Talwalkar]{mohri2018foundations}
Mehryar Mohri, Afshin Rostamizadeh, and Ameet Talwalkar.
\newblock \emph{{Foundations of machine learning}}.
\newblock MIT press, 2018.

\bibitem[Molchanov(2017)]{molchanov2017theory}
Ilya Molchanov.
\newblock \emph{{Theory of Random Sets}}.
\newblock Springer, 2017.

\bibitem[Mou et~al.(2018)Mou, Wang, Zhai, and Zheng]{mou2018generalization}
Wenlong Mou, Liwei Wang, Xiyu Zhai, and Kai Zheng.
\newblock {Generalization Bounds of SGLD for Non-convex Learning: Two Theoretical Viewpoints}.
\newblock In \emph{Conference On Learning Theory (COLT)}, 2018.

\bibitem[Nagarajan and Kolter(2019)]{nagarajan2019uniform}
Vaishnavh Nagarajan and J.~Zico Kolter.
\newblock {Uniform convergence may be unable to explain generalization in deep learning}.
\newblock In \emph{Advances in Neural Information Processing Systems (NeurIPS)}, 2019.

\bibitem[Negrea et~al.(2019)Negrea, Haghifam, Dziugaite, Khisti, and Roy]{negrea2019information}
Jeffrey Negrea, Mahdi Haghifam, Gintare~Karolina Dziugaite, Ashish Khisti, and Daniel Roy.
\newblock {Information-Theoretic Generalization Bounds for SGLD via Data-Dependent Estimates}.
\newblock In \emph{Advances in Neural Information Processing Systems (NeurIPS)}, 2019.

\bibitem[Neu et~al.(2021)Neu, Dziugaite, Haghifam, and Roy]{neu2021information}
Gergely Neu, Gintare~Karolina Dziugaite, Mahdi Haghifam, and Daniel Roy.
\newblock {Information-Theoretic Generalization Bounds for Stochastic Gradient Descent}.
\newblock In \emph{Conference on Learning Theory (COLT)}, 2021.

\bibitem[Neyshabur et~al.(2015)Neyshabur, Tomioka, and Srebro]{neyshabur2015norm}
Behnam Neyshabur, Ryota Tomioka, and Nathan Srebro.
\newblock {Norm-Based Capacity Control in Neural Networks}.
\newblock In \emph{Conference on Learning Theory (COLT)}, 2015.

\bibitem[{\O}ksendal(2003)]{oksendal2003stochastic}
Bernt {\O}ksendal.
\newblock \emph{{Stochastic Differential Equations}}.
\newblock Springer, 2003.

\bibitem[Parrado{-}Hern{\'{a}}ndez et~al.(2012)Parrado{-}Hern{\'{a}}ndez, Ambroladze, Shawe{-}Taylor, and Sun]{parradohernandez2012pac}
Emilio Parrado{-}Hern{\'{a}}ndez, Amiran Ambroladze, John Shawe{-}Taylor, and Shiliang Sun.
\newblock {PAC-Bayes Bounds with Data Dependent Priors}.
\newblock \emph{Journal of Machine Learning Research}, 2012.

\bibitem[Pensia et~al.(2018)Pensia, Jog, and Loh]{pensia2018generalization}
Ankit Pensia, Varun Jog, and Po-Ling Loh.
\newblock {Generalization Error Bounds for Noisy, Iterative Algorithms}.
\newblock \emph{IEEE International Symposium on Information Theory (ISIT)}, 2018.

\bibitem[Raginsky et~al.(2017)Raginsky, Rakhlin, and Telgarsky]{raginsky2017non}
Maxim Raginsky, Alexander Rakhlin, and Matus Telgarsky.
\newblock {Non-Convex Learning via Stochastic Gradient Langevin Dynamics: A Nonasymptotic Analysis}.
\newblock In \emph{Conference on Learning Theory (COLT)}, 2017.

\bibitem[Rivasplata et~al.(2020)Rivasplata, Kuzborskij, Szepesv{\'{a}}ri, and Shawe{-}Taylor]{rivasplata2020pac}
Omar Rivasplata, Ilja Kuzborskij, Csaba Szepesv{\'{a}}ri, and John Shawe{-}Taylor.
\newblock {PAC-Bayes Analysis Beyond the Usual Bounds}.
\newblock In \emph{Advances in Neural Information Processing Systems (NeurIPS)}, 2020.

\bibitem[Sachs et~al.(2023)Sachs, van Erven, Hodgkinson, Khanna, and {\c S}im{\c s}ekli]{sachs2023generalization}
Sarah Sachs, Tim van Erven, Liam Hodgkinson, Rajiv Khanna, and Umut {\c S}im{\c s}ekli.
\newblock {Generalization Guarantees via Algorithm-dependent Rademacher Complexity}.
\newblock In \emph{Conference on Learning Theory (COLT)}, 2023.

\bibitem[Shalev{-}Schwartz and Ben{-}David(2014)]{shalevschwartz2014understanding}
Shai Shalev{-}Schwartz and Shai Ben{-}David.
\newblock \emph{{Understanding Machine Learning - From Theory to Algorithms}}.
\newblock Cambridge University Press, 2014.

\bibitem[Shawe{-}Taylor and Williamson(1997)]{shawetaylor1997pac}
John Shawe{-}Taylor and Robert Williamson.
\newblock {A PAC Analysis of a Bayesian Estimator}.
\newblock In \emph{Conference on Computational Learning Theory (COLT)}, 1997.

\bibitem[{\c S}im{\c s}ekli et~al.(2019){\c S}im{\c s}ekli, Sagun, and G{\"{u}}rb{\"{u}}zbalaban]{simsekli2019tail}
Umut {\c S}im{\c s}ekli, Levent Sagun, and Mert G{\"{u}}rb{\"{u}}zbalaban.
\newblock {A Tail-Index Analysis of Stochastic Gradient Noise in Deep Neural Networks}.
\newblock In \emph{International Conference on Machine Learning (ICML)}, 2019.

\bibitem[{\c S}im{\c s}ekli et~al.(2020){\c S}im{\c s}ekli, Sener, Deligiannidis, and Erdogdu]{simsekli2020hausdorff}
Umut {\c S}im{\c s}ekli, Ozan Sener, George Deligiannidis, and Murat Erdogdu.
\newblock {Hausdorff Dimension, Heavy Tails, and Generalization in Neural Networks}.
\newblock In \emph{Advances in Neural Information Processing Systems (NeurIPS)}, 2020.

\bibitem[Steinke and Zakynthinou(2020)]{steinke2020reasoning}
Thomas Steinke and Lydia Zakynthinou.
\newblock {Reasoning About Generalization via Conditional Mutual Information}.
\newblock In \emph{Conference on Learning Theory (COLT)}, 2020.

\bibitem[{van Erven} and Harremo{\"e}s(2014)]{vanerven2014renyi}
Tim {van Erven} and Peter Harremo{\"e}s.
\newblock {R{\'e}nyi Divergence and Kullback-Leibler Divergence}.
\newblock \emph{IEEE Transactions on Information Theory}, 2014.

\bibitem[Vapnik and Chervonenkis(1968)]{vapnik1968uniform}
Vladimir Vapnik and Alexey Chervonenkis.
\newblock {On the Uniform Convergence of Relative Frequencies of Events to Their Probabilities}.
\newblock In \emph{Doklady Akademii Nauk USSR}, 1968.

\bibitem[Vapnik and Chervonenkis(1971)]{vapnik1971uniform}
Vladimir Vapnik and Alexey Chervonenkis.
\newblock {On the Uniform Convergence of Relative Frequencies of Events to Their Probabilities}.
\newblock \emph{Theory of Probability and its Applications}, 1971.

\bibitem[Vershynin(2018)]{vershynin2018high}
Roman Vershynin.
\newblock \emph{{High-Dimensional Probability: An Introduction with Applications in Data Science}}.
\newblock Cambridge University Press, 2018.

\bibitem[Viallard et~al.(2023)Viallard, Haddouche, {\c S}im{\c s}ekli, and Guedj]{viallard2023learning}
Paul Viallard, Maxime Haddouche, Umut {\c S}im{\c s}ekli, and Benjamin Guedj.
\newblock {Learning via Wasserstein-Based High Probability Generalisation Bounds}.
\newblock In \emph{Advances in Neural Information Processing Systems (NeurIPS)}, 2023.

\bibitem[Viallard et~al.(2024{\natexlab{a}})Viallard, Germain, Habrard, and Morvant]{viallard2024general}
Paul Viallard, Pascal Germain, Amaury Habrard, and Emilie Morvant.
\newblock {A general framework for the practical disintegration of PAC-Bayesian bounds}.
\newblock \emph{Machine Learning}, 2024{\natexlab{a}}.

\bibitem[Viallard et~al.(2024{\natexlab{b}})Viallard, Haddouche, {\c S}im{\c s}ekli, and Guedj]{viallard2024tighter}
Paul Viallard, Maxime Haddouche, Umut {\c S}im{\c s}ekli, and Benjamin Guedj.
\newblock {Tighter Generalisation Bounds via Interpolation}.
\newblock \emph{arXiv}, abs/2402.05101, 2024{\natexlab{b}}.

\bibitem[Xu and Raginsky(2017)]{xu2017information}
Aolin Xu and Maxim Raginsky.
\newblock {Information-Theoretic Analysis of Generalization Capability of Learning Algorithms}.
\newblock \emph{Advances in Neural Information Processing Systems (NIPS 2017)}, 2017.

\bibitem[Yang et~al.(2019)Yang, Sun, and Roy]{yang2019fast}
Jun Yang, Shengyang Sun, and Daniel~M. Roy.
\newblock {Fast-rate PAC-Bayes Generalization Bounds via Shifted Rademacher Processes}.
\newblock In \emph{Advances in Neural Information Processing Systems (NeurIPS)}, 2019.

\bibitem[Ying(2004)]{ying2004mcdiarmid}
Yiming Ying.
\newblock {McDiarmid’s inequalities of Bernstein and Bennett forms}, 2004.

\bibitem[Zantedeschi et~al.(2021)Zantedeschi, Viallard, Morvant, Emonet, Habrard, Germain, and Guedj]{zantedeschi2021learning}
Valentina Zantedeschi, Paul Viallard, Emilie Morvant, R{\'{e}}mi Emonet, Amaury Habrard, Pascal Germain, and Benjamin Guedj.
\newblock {Learning Stochastic Majority Votes by Minimizing a PAC-Bayes Generalization Bound}.
\newblock In \emph{Advances in Neural Information Processing Systems (NeurIPS)}, 2021.

\bibitem[Zhang et~al.(2017)Zhang, Bengio, Hardt, Recht, and Vinyals]{zhang2017understanding}
Chiyuan Zhang, Samy Bengio, Moritz Hardt, Benjamin Recht, and Oriol Vinyals.
\newblock {Understanding deep learning requires rethinking generalization}.
\newblock In \emph{International Conference on Learning Representations (ICLR)}, 2017.

\bibitem[Zhang et~al.(2021)Zhang, Bengio, Hardt, Recht, and Vinyals]{zhang2021understanding}
Chiyuan Zhang, Samy Bengio, Moritz Hardt, Benjamin Recht, and Oriol Vinyals.
\newblock {Understanding deep learning (still) requires rethinking generalization}.
\newblock \emph{Communications of the ACM}, 2021.

\end{thebibliography}

\end{document}